\documentclass[preprint,12pt]{elsarticle}
\usepackage{hyperref}

\usepackage{amssymb,amsmath,amsthm}
\usepackage{subfigure}
\usepackage{stmaryrd}
\usepackage{url}
\usepackage{rotating}
\usepackage{graphicx}

\usepackage[all]{xy}

\newtheorem{lemma}{Lemma}
\newtheorem{proposition}{Proposition}
\newtheorem{corollary}{Corollary}
\newtheorem{example}{Example}
\newtheorem{remark}{Remark}
\newtheorem{assumption}{Assumption}

\usepackage{xcolor}

\newcommand{\bp}{{\bf P}}
\newcommand{\bpi}{{\bf P}^{-1}}
\newcommand{\bc}{{\bf C}}
\newcommand{\bdc}{{\bf DC}}
\newcommand{\bdr}{{\bf DR}}
\newcommand{\bec}{{\bf EC}}
\newcommand{\beq}{{\bf EQ}}
\newcommand{\bo}{{\bf O}}
\newcommand{\bpo}{{\bf PO}}
\newcommand{\bpp}{{\bf PP}}
\newcommand{\bppi}{{\bf PP}^{-1}}
\newcommand{\btpp}{{\bf TPP}}
\newcommand{\btppi}{{\bf TPP}^{-1}}
\newcommand{\bntpp}{{\bf NTPP}}
\newcommand{\bntppi}{{\bf NTPP}^{-1}}

\newcommand{\tpp}[2]{#1\, \btpp\, #2}
\newcommand{\tppi}[2]{#1\, \btppi\, #2}
\newcommand{\p}[2]{#1\, \bp\, #2}
\newcommand{\ov}[2]{#1\, \bo\, #2}
\newcommand{\pp}[2]{#1\, \bpp\, #2}
\newcommand{\ppi}[2]{#1\, \bppi\, #2}
\newcommand{\ntpp}[2]{#1\, \bntpp\, #2}
\newcommand{\ntppi}[2]{#1\, \bntppi\, #2}
\newcommand{\po}[2]{#1\, \bpo\, #2}
\newcommand{\ec}[2]{#1\, \bec\, #2}
\newcommand{\dc}[2]{#1\, \bdc\, #2}
\newcommand{\dr}[2]{#1\, \bdr\, #2}
\newcommand{\eq}[2]{#1\, \beq\, #2}
\newcommand{\ch}[0]{\textit{CH}}
\newcommand{\cvx}{\ch}

\newcommand{\new}[1]{#1}

\makeatletter
\def\old@comma{,}
\catcode`\,=13
\def,{%
  \ifmmode%
    \old@comma\discretionary{}{}{}%
  \else%
    \old@comma%
  \fi%
}
\makeatother

\begin{document}

\begin{frontmatter}

\title{Realizing RCC8 networks using convex regions}
\author[cardiff]{Steven Schockaert}
\ead{s.schockaert@cs.cardiff.ac.be}
\author[uts]{Sanjiang Li}
\ead{sanjiang.li@uts.edu.au}

\address[cardiff]{Cardiff University, School of Computer Science \& Informatics, 5 The Parade, Cardiff CF24 3AA, UK}
\address[uts]{University of Technology Sydney, AMSS-UTS Joint Research Lab, Centre for Quantum Computation \& Intelligent Systems, Broadway NSW 2007, Australia}

\begin{abstract}
RCC8 is a popular fragment of the region connection calculus, in which qualitative spatial relations between regions, such as adjacency, overlap and parthood, can be expressed. While RCC8 is essentially dimensionless, most current applications are confined to reasoning about two-dimensional or three-dimensional physical space. In this paper, however, we are mainly interested in conceptual spaces, which typically are high-dimensional Euclidean spaces in which the meaning of natural language concepts can be represented using convex regions. The aim of this paper is to analyze how the restriction to convex regions constrains the realizability of networks of RCC8 relations. First, we identify all ways in which the set of RCC8 base relations can be restricted to guarantee that  consistent networks can be convexly realized in respectively 1D, 2D, 3D, and 4D. Most surprisingly, we find that if the relation `partially overlaps' is disallowed, all consistent atomic RCC8 networks can be convexly realized in 4D. If instead refinements of the relation `part of' are disallowed, all consistent atomic RCC8 relations can be convexly realized in 3D. We furthermore show, among others, that any consistent RCC8 network with $2n+1$ variables can be realized using convex regions in the $n$-dimensional Euclidean space.
\end{abstract}

\end{frontmatter}

\section{Introduction}
RCC8 is a well-known constraint language for modelling and reasoning about the qualitative spatial relationships that hold between regions \cite{Randell:1992}.  It is based on eight relations which are jointly exhaustive and pairwise disjoint: equals ($\beq$), partially overlaps ($\bpo$), externally connected ($\bec$), disconnected ($\bdc$), tangential proper part ($\btpp$) and its inverse ($\btppi$), and non-tangential proper part ($\bntpp$) and its inverse ($\bntppi$). The intended meaning of these and a number of related relations is summarized in Table \ref{tabDefRCC8}. RCC5 is a variant of RCC8 which is instead based on the following five coarser basic relations: equals ($\beq$), partially overlaps ($\bpo$), disjoint ($\bdr$), proper part ($\bpp$) and its inverse ($\bppi$). In general, constraint networks over RCC8 (or RCC5) relations could be realized using regions from arbitrary topological spaces that satisfy the axioms of the RCC \cite{Randell:1992}. 

In practice, most applications use RCC8 to reason about  two-dimensional or three-dimensional Euclidean space. For example, \cite{sridhar2011video} uses RCC8 to encode how different objects in a video are spatially related to each other, and how these relations change over time. This high-level representation is then used to recognize particular types of events. On the semantic web, RCC8 is used to encode geographic information in OWL ontologies \cite{grutter2008improving,conf/semweb/StockerS08} and linked data \cite{battle2012enabling,koubarakis2011challenges}, and to reason about the integrity of spatial databases \cite{smart2007framework}. It can be shown that focusing on Euclidean spaces of fixed dimension does not affect the notion of consistency which is used in RCC8 \cite{Renz:2002} (see Section \ref{secPreliminaries}).

\renewcommand{\tabcolsep}{4pt}
\begin{table}[t]
\centering
\caption{Qualitative spatial relations between regions $a$ and $b$ with interiors $i(a)$ and $i(b)$. \label{tabDefRCC8}}
\begin{tabular}{lll}
\hline
Name & Symbol & Meaning\\
\hline
Part of & \bp & $a \subseteq b$\\
Proper part of & \bpp & $a \subset b$\\
Equals & \beq & $a = b$\\
Overlaps & \bo & $i(a) \cap i(b) \neq \emptyset$\\
Disjoint from & \bdr & $i(a) \cap i(b) = \emptyset$\\
Disconnected & \bdc & $a\cap b = \emptyset$ \\
Externally Connected & \bec & $a\cap b \neq \emptyset$, $i(a) \cap i(b) = \emptyset$  \\
Partially Overlaps & \bpo & $i(a)\cap i(b) \neq \emptyset$, $a \not\subseteq b$, $b\not \subseteq a$\\
Tangential Proper Part & \btpp & $a\subset b$, $a\not\subset i(b)$\\
 & $\btppi$ & $b\subset a$, $b\not\subset i(a)$\\
Non-Tangential Proper Part & \bntpp & $a\subset i(b)$\\
 & $\bntppi$ & $b\subset i(a)$\\
\hline
\end{tabular}
\end{table}

RCC8 could potentially also play a crucial role in reasoning about higher-dimensional Euclidean spaces. In such applications, however, it is often natural to consider convex regions only. \new{For example, in the theory of conceptual spaces proposed in \cite{gardenfors2001reasoning},} the meaning of natural properties and concepts is represented using convex regions in a suitable metric space, which is most often taken to be Euclidean \cite{gardenfors2001reasoning}. The concepts \emph{car} and \emph{vehicle} would both correspond to convex regions and the constraint $\pp{\textit{car}}{\textit{vehicle}}$ then means that every car is a vehicle, while $\po{\textit{vehicle}}{\textit{wooden-object}}$ means that some but not all vehicles are wooden objects (e.g.\ a dugout canoe) and not all wooden objects are vehicles. \new{While the idea that natural concepts tend to be convex is a conjecture, among other examples, \cite{gardenfors2001reasoning} points to evidence from cognitive studies on color spaces, showing that natural language terms for colors tend to correspond to convex regions in a suitable conceptual space . Requiring convexity is also closely related to prototype theory \cite{rosch1973natural}. Indeed, a common approach is to represent prototypes as points in a feature space and to define the extension of concepts as the cells of the Voronoi diagram induced from these points, which are convex. In machine learning, convex hull based classifiers have been proposed \cite{nalbantov2006nearest}, which are explicitly based on the assumption that categories tend to correspond to convex regions in a feature space.  In particular, in a convex hull based classifier, every category $C_i$ is represented geometrically as the convex hull of the points (i.e.\ entitites) that are known to belong to $C_i$. A test item $x$ is then assigned to the category whose convex hull is closest. While such classifiers are often effective, classification decisions require solving a quadratic program and is thus computationally expensive. As an alternative, in \cite{derrac2014enriching} a classifier is introduced which encodes the intuition that when $x$ is located between two items from class $C_i$ in the feature space, then $x$ is also likely to belong to class $C_i$. Such a betweenness based classifier also implicitly uses the assumption that categories correspond to convex regions, while being more efficient than a convex hull based classifier.}

In computational linguistics, methods are studied to learn representations for the meaning of natural language terms as convex regions from large corpora \cite{Erk:2009:RWR:1596374.1596387}. In the context of conceptual spaces, $\bec$ and $\btpp$ relate to some notion of conceptual neighborhood (cf.\ \cite{Freksa:1991}). For example, $\ec{\textit{orange}}{\textit{red}}$ means that there is a continuous transition from colors which are considered red to colors which are considered orange, without passing through any other colors. This also entails that while orange and red are considered as separate colors, there may exist borderline colors for which it is hard to judge whether they are red or orange. This notion of conceptual neighborhood has proven useful for defining commonsense reasoning approaches to merge conflicting propositional knowledge bases \cite{Schockaert20111815}. Alternatively, the exact boundaries of what it means for an object to be orange may be considered to be vague, but we may still use an RCC8 based representation in such a case \cite{gardenfors2001reasoning}, e.g.\ by relying on the Egg-Yolk calculus \cite{Cohn:1996} or on a fuzzy region connection calculus \cite{Schockaert2009258}. In any case, qualitative spatial reasoning based on (a variant of) RCC8 is well-suited for formalising particular forms of commonsense reasoning about natural language concepts \cite{gardenfors2001reasoning}. This application is likely to inspire further extensions and variants of RCC8 and RCC5; e.g.\ in \cite{schockaert2013combining}, RCC5 is combined with a notion of betweenness to formalize interpolation, a particular form of commonsense reasoning.

As another example of qualitative reasoning about convex regions in higher-dimensional Euclidean spaces, we consider species distribution models \cite{elith2009species}, which are used in ecology to specify the environmental parameters within which the occurrence of a given species can be sustained. Typically, the distribution model of a given organism is encoded as a convex region in a Euclidean space in which the dimensions correspond to  environmental parameters (e.g.\ related to climate and land cover). Several ecological constraints can naturally be expressed using RCC8 relations between these models. For example, knowing that the mountain lion is a predator of the bighorn sheep\footnote{\url{http://umyosemite22.wikispaces.com/Biotic+Factors}, accessed 29 November 2013.}, we know that $\po{\textit{mountain-lion}}{\textit{bighorn-sheep}}$ needs to hold for the corresponding distribution models. Similarly, since a harlequin ladybird is a kind of ladybird, we should require $\pp{\textit{harlequin-ladybird}}{\textit{ladybird}}$. RCC8 could thus allow us to reason about ecological information extracted from text (e.g.\ the encyclopedia of life\footnote{\url{http://eol.org}, accessed 29 November 2013}), which could be useful to derive constraints that can be taken into account when learning distribution models from sparse data.

As a final example, we consider imprecise probability theory \cite{walley1991statistical}, which is a theory of uncertainty in which the beliefs of an agent are encoded as a convex region called a credal set. Specifically, these regions are subsets of the standard simplex and each point of this simplex corresponds to a probability distribution in the domain of discourse. For example if $A$ and $B$ are the regions encoding the beliefs of agents $a$ and $b$, then $\pp{a}{b}$ means that $a$ is better informed than $b$, i.e.\ the beliefs of $a$ and $b$ are compatible and $a$ could not learn anything by sharing its information with $b$. Similarly, $\dc{a}{b}$ means that the beliefs of $a$ and $b$ are incompatible. RCC8 could thus be used as the basis for a qualitative counterpart to the theory of imprecise probabilities.

A natural question is then: how does the requirement that all regions be convex affect the realizability of RCC8 constraint networks? It is easy to see that deciding whether an RCC8 network has a convex solution in $\mathbb{R}$ can be reduced to consistency checking in the interval algebra \cite{Allen83}, which is NP-complete \cite{Vilain:1989:CPA:93913.93996}. It is moreover well known that many consistent RCC8 constraint networks cannot be realized by convex regions in $\mathbb{R}^2$. 
\begin{example}\label{exNot2Dconvex}
Consider the following set of constraints:
\begin{align*}
&\ec{a}{b}  &&\tpp{x}{a} && \tpp{y}{a} && \tpp{u}{b} && \tpp{v}{b}\\
&\dc{x}{y}  && \dc{u}{v} && \ec{x}{u} && \ec{x}{v} && \ec{y}{u} && \ec{y}{v}
\end{align*}
Figure \ref{figEx1} shows a configuration which satisfies all of these constraints except for $\ec{x}{v}$. Since $a\cap u$, $a\cap v$, $x\cap b$ and $y\cap b$ are necessarily all intervals of the same line, it is easy to show that additionally satisfying $\ec{x}{v}$ is not possible (see Section \ref{secLowerboundNumberRegions}).
\end{example}
In \cite{Davis:1999} it was shown that deciding whether an RCC8 constraint network can be realized as convex regions in $\mathbb{R}^2$ is decidable, but computationally hard. Specifically, this problem was shown to be as hard as deciding consistency in the existential theory of the real numbers \cite{Basu:1996:CAC:235809.235813}. \new{From this result, we can derive a similar result for regions in $\mathbb{R}^k$, for any fixed $k\geq 2$. We provide a proof in the appendix. Moreover, \cite{Davis:1999} shows that any RCC8 network which has a convex solution in $\mathbb{R}^k$ can be realized using convex polytopes.}

In this paper, we show that despite these intractability results, RCC8 can still be used to efficiently reason about convex regions. The key insight is that every consistent RCC8 constraint network can be realized using convex regions, provided that the number of dimensions is sufficiently high. Note that as the required number of dimensions depends on the number of variables in the constraint network, this result does not contradict the result from \cite{Davis:1999}, which applies only if the number of dimensions $k$ is fixed in advance. In particular, we make the following contributions:
\begin{enumerate}
\item For all subsets of the RCC8 relations, we derive the minimal number of dimensions $k$ which are needed to guarantee that all consistent constraint networks which only use relations from this subset can be realized by convex regions in $\mathbb{R}^k$. For the full RCC8, it is clear that a finite bound $k$ does not exist. However, among others we show that if $\bpo$ is not allowed, all consistent networks can be realized in $\mathbb{R}^4$. If only the relations in $\{\bpo,\bec,\bdc,\beq\}$ or only the relations in $\{\bec,\bdc,\bntpp,\bntppi,\beq\}$ are allowed, all consistent networks can be realized in $\mathbb{R}^3$. 
\item We analyze how the required number of dimensions relates to the number of variables. Among others we show that all consistent RCC8 constraint networks with $2n+1$ regions can be realized using convex regions in $\mathbb{R}^n$.
\end{enumerate}
In this way, we establish important sufficient conditions to use standard RCC8 reasoners for sound and complete reasoning about convex regions. In computational linguistics, for instance, it is common to consider high-dimensional Euclidean spaces of 100 to 500 dimensions to represent the meaning of natural language terms \cite{Erk:2009:RWR:1596374.1596387,turney2010frequency}. If we use RCC8 to reason about convex regions in these spaces, we can use standard RCC8 reasoners \cite{westphal2009qualitative} if the number of considered regions is at most $2n+1 =1001$ (in the case of 500 dimensions). In fact, $2n+1$ is the worst-case bound we obtain, and for most problem instances with more than 1001 regions, the methodology from Section \ref{secRegionsDimensions} will still allow us to derive that convex solutions in $\mathbb{R}^n$ must exist. 

In other words, the aim of this paper is to identify sufficient conditions on the required number of dimensions to guarantee that a convex solution exists. 
The paper is structured as follows. In the next section, we recall some basic results about RCC8 and RCC5. Then in Section \ref{secOverview}, we give an overview of the results that are established in this paper. The remainder of the paper  focuses on the proofs of these results. In particular, in Section \ref{secRestricting}, we derive the minimal number of dimensions that are needed to guarantee convex solutions when only a subset of the RCC8 relations is allowed. Based on these results, we then show in Section \ref{secRegionsDimensions} how the number of regions and the number of occurrences of particular RCC8 relations can be used to derive finite bounds on the number of regions even for RCC8 constraint networks in which all RCC8 relations are used.

Finally note that Section \ref{secRegionsDimensions} of this paper presents a substantially revised and extended version of the results from \cite{ecai2012paper}. The results presented in Section \ref{secRestricting} are completely new.


\begin{figure}
\centering
\includegraphics[width=200pt]{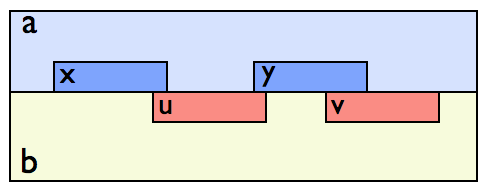}
\caption{It is not possible to additionally satisfy $\ec{x}{v}$.\label{figEx1}}
\end{figure}


%
%
%
%


\section{Preliminaries}\label{secPreliminaries}

\new{In this section, we provide some background about the region connection calculus, convex polytopes and the moment curve. We will assume that the reader is familiar with basic notions from topology such as open and closed sets and the interior and closure operators, and with basic notions from geometry such as the convex hull.}

\new{Throughout the paper, we will use Greek uppercase letters such as $\Theta$ and $\Psi$ for sets of RCC8 or RCC5 formulas, and uppercase letters such as $S$ for other sets. We will use lowercase Greek letters such as $\lambda$ and $\theta$ for real numbers, and lowercase letters such as $p$ and $q$ for points. We will use bold lowercase letters such as $\mathbf{h}$ for vectors.}

\subsection{\new{The region connection calculus}}

\begin{sidewaystable}
\centering
\caption{RCC-8 composition table \cite{Cui:1993}.  \label{tableRCC8composition}}
\begin{tabular}{@{}|c||c|c|c|c|c|c|c|@{}}
\hline
				    & $\bdc$ & $\bec$ & $\bpo$ & $\btpp$ & $\bntpp$ & $\btpp^{-1}$ & $\bntpp^{-1}$ \\
\hline
\hline
$\bdc$        & $\mathcal{R}_8$ & $\bdc,\bec,$   & $\bdc,\bec,$  & $\bdc,\bec,$  & $\bdc,\bec,$  &  $\bdc$    & $\bdc$     \\
            &     & $\bpo,\btpp,$ &  $\bpo,\btpp,$ & $\bpo,\btpp,$ & $\bpo,\btpp,$ &          & \\
            &     & $\bntpp $   &  $\bntpp$    & $\bntpp$    & $\bntpp$    &          &      \\
\hline
$\bec$        & $\bdc,\bec,$       & $\bdc,\bec,$      & $\bdc,\bec, $ & $\bec,\bpo,$ & $\bpo,\btpp,$ & $\bdc,\bec$ & $\bdc$ \\
            & $\bpo,\btpp^{-1},$ & $\bpo,\btpp,$     & $\bpo,\btpp,$ & $\btpp,$   & $\bntpp$    &         &       \\
            & $\bntpp^{-1}$    & $\btpp^{-1},\beq$ & $\bntpp$    & $\bntpp$   &           &         &       \\
\hline
$\bpo$        & $\bdc,\bec,$       & $\bdc,\bec,$        & $\mathcal{R}_8$ & $\bpo,\btpp,$ & $\bpo,\btpp,$ & $\bdc,\bec,$       & $\bdc,\bec,$   \\
            & $\bpo,\btpp^{-1},$ & $\bpo,\btpp^{-1},$  &     & $\bntpp$    & $\bntpp$    & $\bpo,\btpp^{-1},$ & $\bpo,\btpp^{-1},$  \\
            & $\bntpp^{-1}$    & $\bntpp^{-1}$     &     &           &           & $\bntpp^{-1}$    & $\bntpp^{-1}$  \\
\hline
$\btpp$       & $\bdc$ & $\bdc,\bec$ & $\bdc,\bec,$  & $\btpp,$   & $\bntpp$ & $\bdc,\bec,$      & $\bdc,\bec,$  \\
            &      &         & $\bpo,\btpp,$ & $\bntpp$   &        & $\bpo,\btpp,$     & $\bpo,\btpp^{-1},$ \\
            &      &         & $\bntpp$    &          &        & $\btpp^{-1},\beq$ & $\bntpp^{-1}$ \\
\hline
$\bntpp$      & $\bdc$ & $\bdc$ & $\bdc,\bec,$  & $\bntpp$ & $\bntpp$ & $\bdc,\bec,$  &  $\mathcal{R}_8$ \\
            &      &      & $\bpo,\btpp,$ &        &        & $\bpo,\btpp,$ &      \\
            &      &      & $\bntpp$    &        &        & $\bntpp$    &      \\
\hline
$\btpp^{-1}$  & $\bdc,\bec,$        & $\bec,\bpo,$      & $\bpo,$          & $\bpo,\beq,$     & $\bpo,\btpp,$ & $\btpp^{-1},$ & $\bntpp^{-1}$ \\
            & $\bpo,\btpp^{-1},$  & $\btpp^{-1},$   & $\btpp^{-1},$    & $\btpp,$       & $\bntpp$    & $\bntpp^{-1}$ &   \\
            & $\bntpp^{-1}$     & $\bntpp^{-1}$   & $\bntpp^{-1}$    & $\btpp^{-1}$   &           &             &   \\
\hline
$\bntpp^{-1}$  & $\bdc,\bec,$       & $\bpo,\btpp^{-1},$ & $\bpo,$        & $\bpo,$        & $\bpo,\btpp^{-1},$ & $\bntpp^{-1}$ & $\bntpp^{-1}$\\ 
                                         & $\bpo,\btpp^{-1},$ & $\bntpp^{-1}$    & $\btpp^{-1},$  & $\btpp^{-1},$  & $\btpp,\bntpp,$    &             &  \\ 
                                         & $\bntpp^{-1}$    &                & $\bntpp^{-1}$  & $\bntpp^{-1}$  & $\bntpp^{-1},\beq$ &             &  \\           
\hline
\end{tabular}

\end{sidewaystable}%
\noindent We write:
\begin{align*}
\mathcal{R}_8 &= \{\beq, \bdc, \bec, \bpo,  \btpp, \bntpp, \btppi, \bntppi\}\\
\mathcal{R}_5 &= \{\beq, \bdr,  \bpo,  \bpp, \bppi\}
\end{align*}
A non-empty subset $R \subseteq \mathcal{R}_8$ is called an RCC8 relation.  Without cause for confusion, we can identify a singleton $\{\mathbf{r}\}$ with the corresponding relation $\mathbf{r}$ from $\mathcal{R}_8$ or $\mathcal{R}_5$. These singletons are called the RCC8 and RCC5 base relations respectively. For convenience, we also use the following abbreviations:
\begin{align*}
\bpp &= \{\btpp,\bntpp\} \\
\bppi &= \{\btppi,\bntppi\} \\
\bp &=\{\bpp,\beq\} \\
\bpi &=\{\bppi,\beq\} \\
\bo &= \{\btpp,\bntpp,\btppi,\bntppi,\bpo,\beq\}\\
\bc &= \{\btpp,\bntpp,\btppi,\bntppi,\bpo,\bec,\beq\}
\end{align*}
Let $V = \{v_1,...,v_m\}$ be a fixed set of variables.  An RCC8 constraint is an expression of the form $v_i \,R\, v_j$ where $R$ is an RCC8 relation.  Intuitively, such a constraint imposes that the spatial relation between $v_i$ and $v_j$ is among those in $R$. If $R= \{\mathbf{r}\}$ is base relation, we usually write $v_i \,\mathbf{r}\, v_j$ instead of $v_i \{\mathbf{r}\} v_j$.  A set of RCC8 constraints is called an RCC8 network. An RCC8 network $\Theta$ is called atomic if for each $i\neq j$ it contains a constraint of the form $v_i \,R\, v_j$ or $v_j \,R\, v_i$ where $R$ is a singleton, i.e.\ an atomic network explicitly specifies for each pair of variables what is the corresponding RCC8 base relation. Let $\mathcal{R}$ be a non-empty subset of $\mathcal{R}_8$.  If every constraint $v_i \,R\, v_j$ is such that $R\subseteq \mathcal{R}$, then we call $\Theta$ a network over $\mathcal{R}$. For $V' \subseteq V$, we write $\Theta{\downarrow}V'$ for the restriction of $\Theta$ to the variables in $V'$. i.e.
$$
\Theta{\downarrow}V' = \{v_i \,R\, v_j \,|\, (v_i \,R\, v_j)\in \Theta, v_i\in V', v_j \in V' \}
$$

\new{Recall that a set $X$ is called regular closed if $cl(i(X)=X$, where $cl$ and $i$ are the closure and interior operators.}
We say that an RCC8 network $\Theta$ is consistent iff there exists a mapping  $\mathcal{S}$ from $V$ to the set of nonempty regular closed subsets of $\mathbb{R}^n$, for a given $n\geq 1$, such that:
\begin{itemize}
\item For each $v\in V$, $\mathcal{S}(v)$ is regular closed.
\item For each constraint $v_i \,R\, v_j$ in $\Theta$, it holds that the unique RCC8 base relation which holds between $\mathcal{S}(v_i)$ and $\mathcal{S}(v_j)$ in the sense of Table \ref{tabDefRCC8} is among those in $R$.
\end{itemize}
In such a case, $\mathcal{S}$ is called an $n$-dimensional solution of $\Theta$. If $\Theta$ has an $n$-dimensional solution, we also say that $\Theta$ is realizable in $\mathbb{R}^n$.  It can be shown that the number of dimensions does not affect consistency: if $\Theta$ is realizable in $\mathbb{R}^n$, it  is realizable in $\mathbb{R}^k$ for any $k\geq 1$ \cite{Renz:2002}. Moreover, there are several alternative ways of defining consistency, e.g.\ in terms of topological spaces or algebras \cite{Stell2000111}, which are equivalent to the aforementioned notion of consistency \cite{Li:2006}.

For atomic networks, consistency can be decided in $O(n^3)$ by checking path consistency. Specifically, it suffices to check for each $i,j,k$ whether $\rho_{ik}$ is among the relations in $\rho_{ij} \circ \rho_{jk}$, where the composition $\rho_{ij} \circ \rho_{jk}$ of $\rho_{ij}$ and $\rho_{jk}$ is defined by the RCC8 composition table, shown in Table \ref{tableRCC8composition}. For example, $\Theta = \{\tpp{v_1}{v_2},\po{v_2}{v_3},\eq{v_1}{v_3}\}$ is not consistent, because $\beq$ is not among $\btpp \circ \bpo = \{\bdc,\bec,\bpo,\btpp,\bntpp\}$. For arbitrary RCC8 networks, consistency can be decided by  using a combination of backtracking and path consistency checking \cite{Renz:1999a}.

We say that $\Theta$ entails a constraint $v_i \, R \, v_j$, written $\Theta \models v_i \, R \, v_j$, if every solution of $\Theta$ is also a solution of $\{v_i \, R \, v_j\}$. Similarly, we say that $\Theta$ entails a network $\Psi$, written $\Theta \models \Psi$ if $\Theta$ entails every constraint in $\Psi$.  If $\Theta\models \Psi$ and $\Psi \models \Theta$, we say that $\Theta$ and $\Psi$ are equivalent, written $\Theta \equiv \Psi$. For atomic networks, it holds that $\Theta \models v_i \, R \, v_j$ iff $\Theta$ contains a constraint $v_i \,\{r\}\, v_j$ such that $r\in R$ or a constraint $v_j \,\{r\}\, v_i$ such that $r^{-1} \in R$, where we define $\bdc^{-1} = \bdc$, $\bec^{-1} = \bec$, $\bpo^{-1} = \bpo$, $\beq^{-1} = \beq$, $(\btpp^{-1})^{-1} = \btpp$ and $(\bntpp^{-1})^{-1} = \bntpp$.

The notion of realizability can be refined by imposing additional requirements on the kind of regions which are considered. For example, if all regions are required to be internally connected (i.e.\ no region is the union of two disjoint non-empty regions), some consistent RCC8 networks cannot be realized in $\mathbb{R}$ or in $\mathbb{R}^2$, although all consistent RCC8 networks can be realized using internally connected regions in $\mathbb{R}^3$ \cite{Renz:2002}.  In this paper, we analyze the realizability of consistent RCC8 networks when all regions are required to be convex. A solution of a network $\Theta$ which maps every region to a convex set in $\mathbb{R}^n$ is called a convex solution. If such a convex solution exists, we say that $\Theta$ is convexly realizable in $\mathbb{R}^n$.

A non-empty subset $R\subseteq \mathcal{R}_5$ is called an RCC5 relation. Entirely analogously as for RCC8 relation, we can define the notion of RCC5 constraint, RCC5 network, consistency, entailment and equivalence.

\subsection{\new{Convex polytopes and the moment curve}\label{secGeometry}}

We first recall a number of notions and results about convex polytopes.  A convex polytope $V$ in $\mathbb{R}^n$ is the intersection of a finite set of half-spaces, i.e.\ 
$$
V = \bigcap_{i=1}^n \new{\{\mathbf{x} \,|\, \mathbf{h_i} \cdot \mathbf{x} \leq c_i, x\in \mathbb{R}^n\}}
$$
\new{where $\mathbf{h_i} \in \mathbb{R}^n$ is a vector and $c_i \in \mathbb{R}$ is a constant.}
The corresponding hyperplanes $\new{H_i = \{\mathbf{x} \,|\, \mathbf{h_i} \cdot \mathbf{x} = c_i, \mathbf{x}\in \mathbb{R}^n\}}$ are called the bounding hyperplanes of $V$. We call $F$ an $n-1$ dimensional face of $V$ if $F=H_i \cap V$ and $\dim(F) = n-1$, where $H_i$ is one of the bounding hyperplanes of $V$. A family of convex polytopes $\{V_1,...,V_m\}$ in $\mathbb{R}^n$ is called \emph{neighborly} if $\dim(V_i \cap V_j) = n-1$ for all $1\leq i < j \leq m$. It is well known that neighborly families in $\mathbb{R}^2$ can have at most 3 elements.  A key result which we will use \new{in this paper (cf.\ Propositions \ref{propECDCPO3d} and \ref{propRCDCNTPP3d})} is that there exist neighborly families of arbitrary size in $\mathbb{R}^3$  \cite{dewdney1977convex}.

\new{The Voronoi diagram induced by a set of points $p_1,...,p_m$ in $\mathbb{R}^n$ is the set of convex polytopes $V_1,...,V_m$ defined as follows:
$$
V_i = \{q \,|\, d(q,p_i)\leq \min_{j\neq i} d(q,p_j) \}
$$
}
\noindent The $n$-dimensional \emph{moment curve} is the set of points $\{ (t,t^2,...,t^n) \,|\, t\in \mathbb{R} \}$.  Let us write \new{$M(t)$} for the point $(t,t^2,...,t^n)$. \new{Note that the two-dimensional moment curve is a parabola. In higher dimensional spaces, the moment curve has a number of interesting properties. }For any $n\geq 3$, it can be shown that the Voronoi diagram of the points $\{(t,t^2,...,t^n) \,|\, t\in \{1,...,m\}\}$ is a neighborly family \cite{dewdney1977convex}, and more generally that the Voronoi diagram of any finite set of points \new{on} the positive half of the moment curve is a neighborly family \cite{s-eubnf-91}. Using a similar construction based on the helix $\{(t,\cos t, \sin t) \,|\, t\in \mathbb{R}\}$, \cite{erickson2003arbitrarily} shows that an arbitrarily large neighborly family of congruent convex polyhedra in $\mathbb{R}^3$ can be constructed. 

In particular, it can be shown that any $n$ points on the $n$-dimensional moment curve are linearly independent (i.e.\ in general linear position). Moreover, a hyperplane $H$ intersects the moment curve in at most $n$ points (\cite{edelsbrunner1987algorithms}, Lemma 6.4) and if $H$ intersects the moment curve in exactly $n$ points, the moment curve crosses $H$ in each of these points (\cite{matouvsek2003using}, Lemma 1.6.4).

\section{Overview of the results}\label{secOverview}

In some applications, particular RCC8 relations do not occur.  For example, when modelling natural language concepts, it is often possible to organize  concepts in a single subsumption hierarchy. It is common to interpret such hierarchies in the following way: if concept $A$ is a descendant of concept $B$, then $A\{\btpp,\bntpp\}B$ and if  $A$ is (the descendant of) a sibling of $B$, then $A\{\bec,\bdc\}B$. In particular, in such applications, there are no concepts which partially overlap. As we will show, any consistent atomic RCC8 network without $\bpo$ has a convex realization in $\mathbb{R}^4$. When different (sets of) base relations are excluded, similar upper bounds on the required number of dimensions can be found. Figure \ref{overviewRCC8} presents an overview of the results which will be established in Section \ref{secRestricting}. For completeness, we also consider subsets of RCC5 base relations, for which the results are summarized in Figure \ref{overviewRCC5}. \new{For some fragments, we will show that there exists no finite upper bound on the number of dimensions.}

Note that throughout this paper we will not consider RCC8 and RCC5 networks in which the relation $\beq$ occurs, as we can easily avoid using this relation by appropriately renaming variables. Furthermore, note that if $\bntpp$ is allowed then we also need to allow $\bntppi$, and similarly if $\btpp$ is allowed then $\btppi$ needs to be allowed. This means that the total number of non-empty subsets of RCC8 relations that we need to consider is $2^5 - 1=31$. Figure \ref{overviewRCC8} shows the result \new{for each of these subsets. Similarly, for RCC5}, the total number of subsets to consider is $2^3 -1 =7$, all of which are shown in Figure \ref{overviewRCC5}.

\begin{figure}
\centering
\includegraphics[width=400pt]{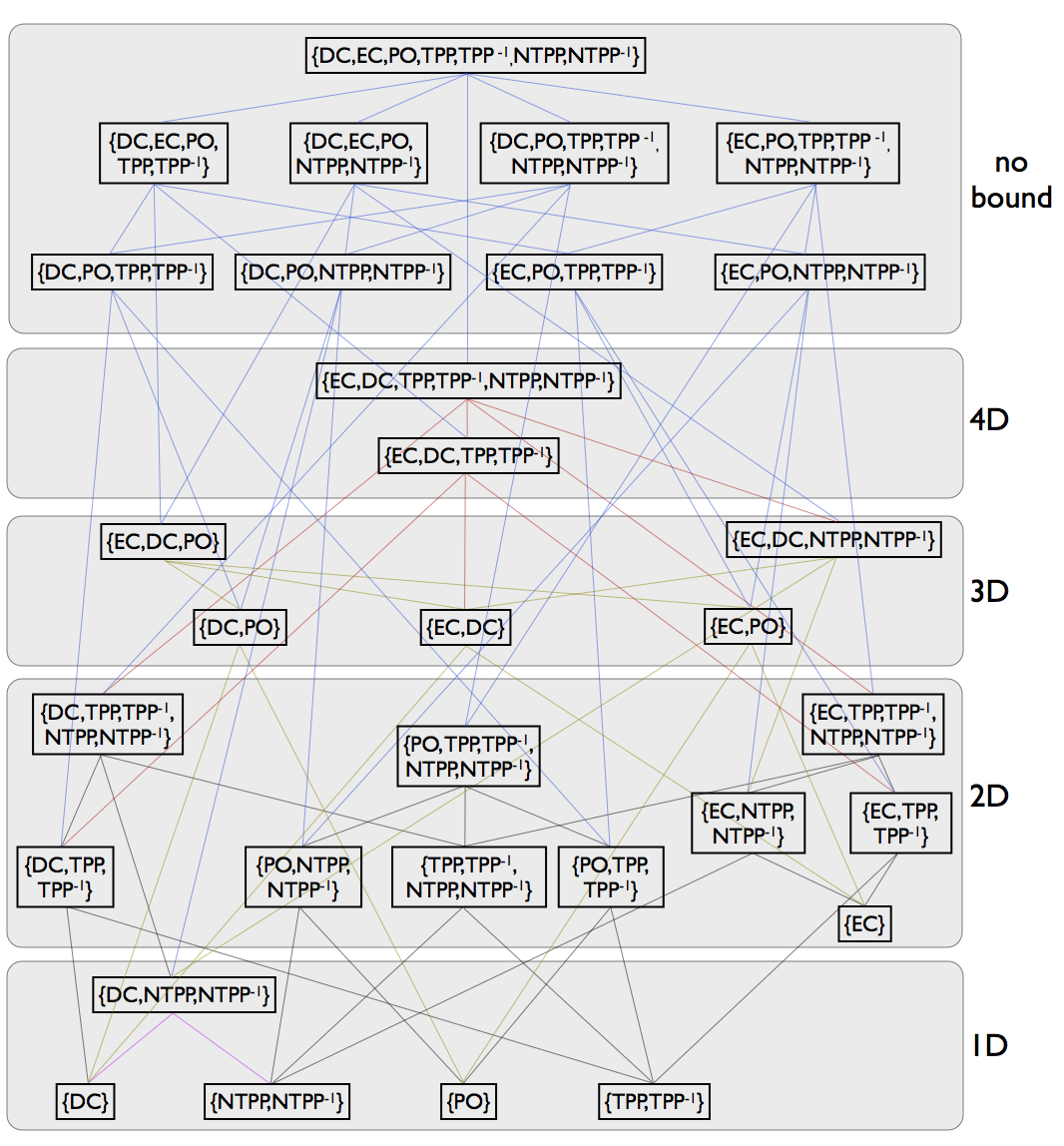}
\caption{Maximal number of dimensions needed to convexly realize atomic networks over a restricted set of RCC8 base relations.\label{overviewRCC8}} 
\end{figure}

\begin{figure}
\centering
\includegraphics[width=370pt]{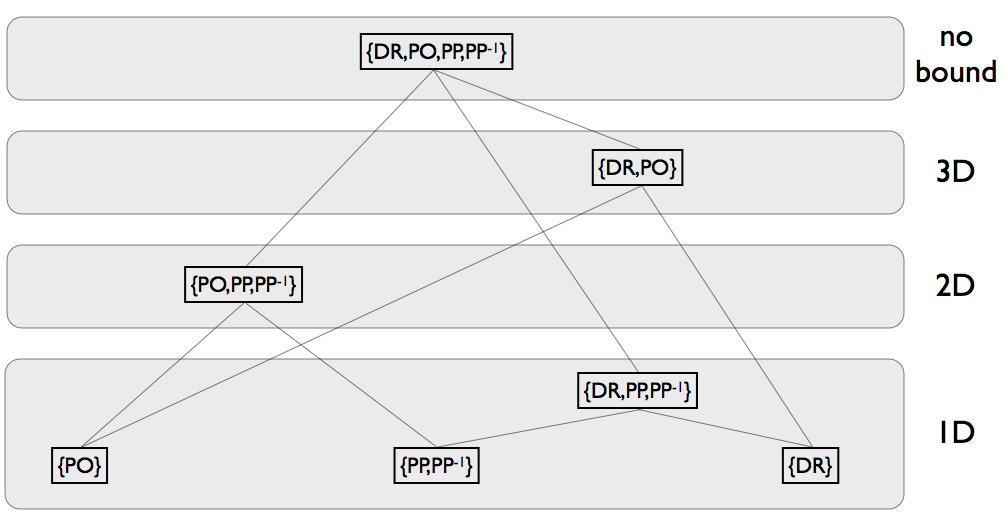}
\caption{Maximal number of dimensions needed to convexly realize atomic networks over a restricted set of RCC5 base relations.\label{overviewRCC5}} 
\end{figure}

\begin{remark}\label{remarkAtomic}
The results from Figures \ref{overviewRCC8} and \ref{overviewRCC5}  only apply to atomic networks. For example, as shown in \cite{Davis:1999}, it is possible to encode an RCC8 network using only $\bec$ and $\btpp$ which is not realizable in $\mathbb{R}^2$. Indeed, let $\Theta$ be an arbitrary network over $\{\bec, \bpo, \btpp, \btppi\}$ which is not realizable in $\mathbb{R}^2$ (we will show in Section \ref{secFragmentsArbitrary} that such networks exist), then we can construct a set of constraints $\Psi$ as follows. 
\begin{itemize}
\item For each pair of variables $a,b$ such that $\Theta \models \tpp{a}{b}$, we add to $\Psi$  the constraint $\tpp{a}{b}$. 
\item For each pair of variables $a,b$ such that $\Theta \models \ec{a}{b}$, we add to $\Psi$ the constraint $\ec{a}{b}$. 
\item For each pair of variables $a,b$ such that $\Theta \models \po{a}{b}$, we add to $\Psi$ the constraints $\tpp{x}{a}$, $\ec{x}{b}$, $\tpp{y}{b}$, $\ec{y}{a}$, $\tpp{z}{a}$ and $\tpp{z}{b}$, where $x$, $y$ and $z$ are fresh variables. 
\end{itemize}
It is clear that $\Psi$ only contains the relations $\btpp$ and $\bec$. Since $\Theta$ is not convexly realizable in $\mathbb{R}^2$ and $\Psi\models \Theta$ it must moreover be the case that $\Psi$ is not convexly realizable in $\mathbb{R}^2$ either. 
\end{remark}

The correctness of the results which are summarized in Figures \ref{overviewRCC8} and \ref{overviewRCC5} is shown in Section \ref{secRestricting}. These results are then used in Section \ref{secRegionsDimensions} to obtain \new{an} upper bound for general RCC8 networks. The idea is to identify a subset $V'\subseteq V$ such that the restriction $\Theta{\downarrow}V'$ belongs to one of the fragments from Figure \ref{overviewRCC8} which is guaranteed to have a convex solution in $\mathbb{R}$, $\mathbb{R}^2$, $\mathbb{R}^3$ or $\mathbb{R}^4$. This partial solution of $\Theta$ is then incrementally extended to a full solution, based on the following key property, which is shown in Section \ref{secRegionsDimensions}:
\begin{quote}
Let $\Theta$ be a consistent atomic RCC8 network. Suppose that $\Theta{\downarrow}V'$ has a $k$-dimensional convex solution in which every \bo-clique\footnote{We say that a subnetwork $\Theta{\downarrow}Z$ ($Z\subseteq V$) is an \emph{\bo-clique} if $\Theta\models {z_1}\bo{z_2}$ for every $z_1,z_2\in Z$.} in $\Theta{\downarrow}V'$ has a common part, and suppose that $\Theta$ entails the following constraints:
\begin{align*}
&\ntpp{a_1}{a_2} && \ntpp{a_2}{a_3} && ... &&\ntpp{a_{l-1}}{a_r}\\
&\ntpp{b_1}{b_2} && \ntpp{b_2}{b_3} && ... &&\ntpp{b_{l-1}}{b_s}\\
&\dc{a_r}{b_s}
\end{align*}
It holds that $\Theta{\downarrow}(V' \cup \{a_1,...,a_r,b_1,...,b_s\})$ has a $(k+1)$-dimensional convex solution in which every \bo-clique has a common part.
\end{quote}
Among others, this property allows us to show that if $V' \subseteq V$ is such that $\Theta{\downarrow}V'$ does not contain any occurrences of $\bpo$, $\Theta$ has a convex solution in $\mathbb{R}^{|V\setminus V'|+4}$. Moreover, we can also establish that every consistent RCC8 network with $2n+1$ variables, for $n\geq 2$, can be convexly realized by using at most $n$ dimensions.

\section{Restricting the set of RCC8 and RCC5 base relations}\label{secRestricting}

The aim of this section is to prove the results which are summarized in Figures \ref{overviewRCC8} and \ref{overviewRCC5}. Specifically, in Section \ref{secFragments1D} we focus on the fragments which can be convexly realized in $\mathbb{R}$;  in Section \ref{secFragments2D} we focus on the fragments which can be convexly realized in $\mathbb{R}^2$ and we show that these fragments cannot be convexly realized in $\mathbb{R}$ in general. Similarly, in Section \ref{secFragments3D} we focus on the fragments which can be convexly realized in $\mathbb{R}^3$ but not in $\mathbb{R}^2$. The most substantial result is presented in Section \new{\ref{secFragments4D}} where we present a construction based on the four-dimensional moment curve to show that consistent atomic networks over $\{\bec,\bdc,\btpp,\btppi,\bntpp,\bntppi\}$ can be convexly realized in $\mathbb{R}^4$. We also provide a counterexample to show that networks over $\{\bec,\bdc,\btpp,\btppi\}$ cannot be convexly realized in $\mathbb{R}^3$ in general. Finally, in Section \ref{secFragmentsArbitrary} we use Radon's theorem to identify fragments which do not in general allow convex realizations in $\mathbb{R}^n$ for any fixed $n$.

\subsection{Fragments with convex solutions in $\mathbb{R}$ \label{secFragments1D}}

In this section, we identify a number of fragments of RCC8 and of RCC5 for which consistent networks always have a convex solution in $\mathbb{R}$, i.e.\ for which consistent networks can always be realized as intervals. The proofs are all constructive and more or less straightforward.

\begin{proposition}\label{propRealizablePO1d}
Every atomic RCC8 network over $\{\bpo\}$ has a convex solution in $\mathbb{R}$.
\end{proposition}
\begin{proof}
We have $\Theta \equiv \{\po{v_i}{v_j} \,| \, 1 \leq i \neq j \leq m\}$. Let $\mathcal{S}$ be defined as:
$$
\mathcal{S}(v_i) = [l_i,u_i]
$$
where $l_1< l_2 < ...  < l_m < u_1 < u_2 < ... < u_m$. It is clear that $\mathcal{S}(v_i)$ is convex. Moreover, for $i<j$ we have that $l_i \in \mathcal{S}(v_i)\setminus \mathcal{S}(v_j)$, $\frac{l_j+u_i}{2} \in i(\mathcal{S}(v_i))\cap i(\mathcal{S}(v_j))$ and $u_j \in \mathcal{S}(v_j)\setminus \mathcal{S}(v_i)$, i.e.\ we have $\po{\mathcal{S}(v_i)}{\mathcal{S}(v_j)}$. In other words, $\mathcal{S}$ is a convex solution in $\mathbb{R}$.
\end{proof}

\begin{proposition}
Every consistent atomic RCC8 network over $\{\btpp,\btppi\}$ has a convex solution in $\mathbb{R}$.
\end{proposition}
\begin{proof}
In a consistent network over $\{\btpp,\btppi\}$, we can always order the variables such that $\tpp{v_i}{v_j}$ iff $i<j$. A convex solution of $\Theta$ is then given by the mapping $\mathcal{S}$, defined as:
$$
\mathcal{S}(v_i) = [l,u_i]
$$
where $l< u_1 < u_2 < ... < u_n$. Clearly, $\mathcal{S}$ is a convex solution of $\Theta$.
\end{proof}

\begin{proposition}\label{propRealizableDCNTPP1D}
Every consistent atomic RCC8 network over $\{\bdc,\allowbreak\bntpp,\allowbreak\bntppi\}$ has a convex solution in $\mathbb{R}$.
\end{proposition}
\begin{proof}
For an atomic network over $\{\bdc,\bntpp,\bntppi\}$ we can partition the set of variables $V$ as follows:
\begin{itemize}
\item $A^1=\{a^1_1,...,a^1_{n_1}\}$ corresponds to the set of regions which are not contained in any of the other regions, i.e.\ for any $a\in A^1$ and any $v\in V\setminus \{a\}$, it holds that $\Theta \not\models \ntpp{a}{v}$.
\item $A^i=\{a^i_1,...,a^i_{n_i}\}$ ($i\geq 2$) corresponds to the set of regions which are contained in some region of $A^{i-1}$ and which are not contained in any region from $V\setminus (A^1\cup ... \cup A^{i-1})$, i.e.\ for any $a^i\in A^i$ there is an $a^{i-1}\in A^{i-1}$ such that $\Theta \models \ntpp{a^i}{a^{i-1}}$, and for any $a^i\in A^i$ and any $v \in V \setminus (A^1 \cup ... \cup A^{i-1}\cup \{a^i\})$, it holds that $\Theta \not\models \ntpp{a^i}{v}$.
\end{itemize}
\new{Note that if $a,b\in A^i$, $a\neq b$, it holds that $\Theta\models \dc{a}{b}$. As a result, for any variable $a\in A^i$ ($i>1$) there is a unique $b\in A^{i-1}$ such that $\Theta\models \ntpp{a}{b}$.}
We define a convex solution $\mathcal{S}$ as follows.  The variables $a^1_1,...,a^1_{n_1}$ from $A^1$ are realized as:
$$
\mathcal{S}(a^1_{j}) = [2{j},2{j}+1]
$$
Assume that $\mathcal{S}$ has been defined for all variables in $A^1 \cup ... \cup A^{i-1}$. Now consider the variables in $A^i$.  Let $a^i_{j}\in A^i$ and let $a^{i-1}_{k} \in A^{i-1}$ be the unique variable for which 
$$
\Theta \models \ntpp{a^i_{j}}{a^{i-1}_{k}}
$$
Let $\mathcal{S}(a^{i-1}_{k}) = [l,u]$.  We realize $a^i_{j}$ as follows:
$$
\mathcal{S}(a^i_{j}) = [l + \frac{2{j}}{2\cdot |A^i| +2} \cdot (u - l),l + \frac{2{j}+1}{2\cdot |A^i| +2} \cdot (u - l)]
$$
It is straightforward to verify that $\mathcal{S}$ is indeed a convex solution of $\Theta$.

\end{proof}

\begin{corollary}
Every consistent atomic RCC8 network over $\{\bdc\}$ or over $\{\bntpp,\allowbreak\bntppi\}$ has a convex solution in $\mathbb{R}$.
\end{corollary}
\begin{corollary}
\new{Every consistent atomic RCC5 network over the following sets of base relations has a convex solution in $\mathbb{R}$}:
\begin{itemize}
\item $\{\bpo\}$
\item $\{\bpp,\bppi\}$
\item $\{\bdr\}$
\item $\{\bdr,\bpp,\bppi\}$
\end{itemize}
\end{corollary}

\subsection{Fragments with convex solutions in $\mathbb{R}^2$ \label{secFragments2D}}

In this section, we look at restrictions on the set of RCC8 or RCC5 base relations which guarantee that networks can be convexly realized in $\mathbb{R}^2$.  To show that this bound on the number of dimensions cannot be strengthened in general, we first show in Section \ref{secLowerboundsR2} that the considered fragments cannot always be convexly realized in $\mathbb{R}$.

\subsubsection{Lower bounds}\label{secLowerboundsR2}
The following propositions are shown by identifying examples of consistent atomic RCC8 networks which cannot be realized as intervals of the real line.

\begin{proposition}
There exists a consistent atomic RCC8 network $\Theta$ over $\{\bdc,\btpp,\btppi\}$ which has no convex solution in $\mathbb{R}$.
\end{proposition}
\begin{proof}
It is easy to see that the following atomic network is consistent, but it cannot be realized by intervals:
$$
\Theta = \{\tpp{a}{d},\tpp{b}{d},\tpp{c}{d},\dc{a}{b},\dc{a}{c},\dc{b}{c} \}
$$
Indeed, if $\mathcal{S}(d)$ were an interval, then $\tpp{a}{d}$, $\tpp{b}{d}$ and $\tpp{c}{d}$ mean that $\mathcal{S}(a)$, $\mathcal{S}(b)$ and $\mathcal{S}(c)$ all share one of the two end points of the interval $\mathcal{S}(d)$, which means that $\mathcal{S}(a)$, $\mathcal{S}(b)$ and $\mathcal{S}(c)$ cannot all be disconnected from each other.
\end{proof}

\begin{figure}
\centering
\includegraphics[width=250pt]{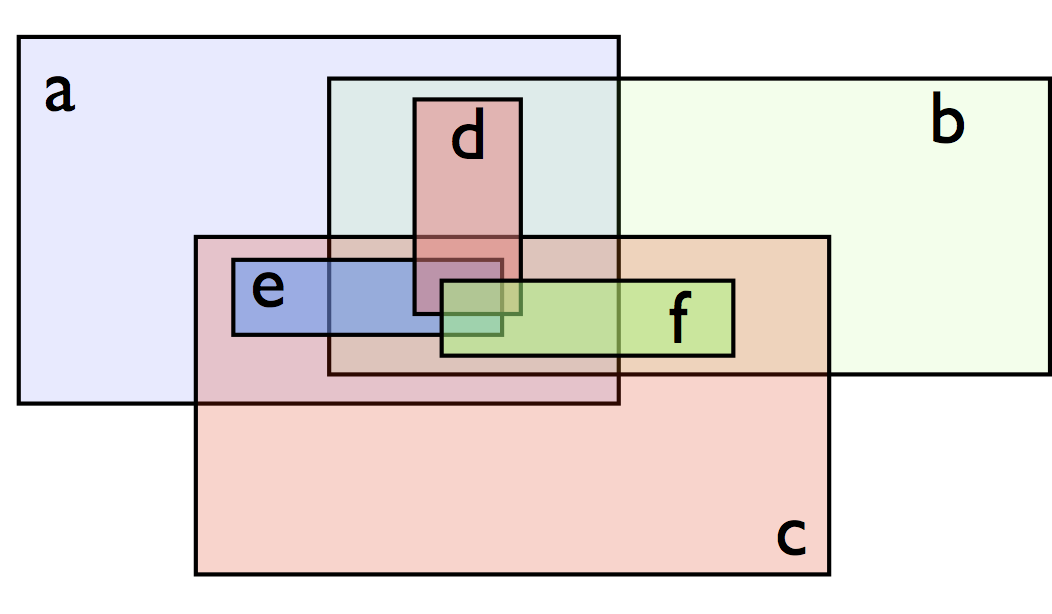}
\caption{\new{Convex realization of the network $\Theta$ from Proposition \ref{prop5}. This network has no convex realization in $\mathbb{R}$.\label{figProp5}}}
\end{figure}

\begin{proposition}\label{prop5}
There exists a consistent atomic \new{RCC5 network over $\{\bpo,\bpp,\bppi\}$} which has no convex solution in $\mathbb{R}$.
\end{proposition}
\begin{proof}
Consider the following atomic network:
\begin{align*}
\Theta = \{&\po{a}{b},\po{a}{c},\po{b}{c}, \po{d}{e},\po{d}{f},\po{e}{f},\\
&\new{\pp{d}{a},\pp{d}{b}},\po{d}{c},
\new{\pp{e}{a},\pp{e}{c}},\po{e}{b},\\
&\new{\pp{f}{b},\pp{f}{c}},\po{f}{a} \}
\end{align*}
\new{Figure \ref{figProp5} shows a convex solution of this network in $\mathbb{R}^2$. We now show that a convex solution in $\mathbb{R}$ does not exist.}
Suppose $\mathcal{S}$ were a convex solution of $\Theta$ in $\mathbb{R}$.
Let $\mathcal{S}(a)=[a^-,a^+]$ and similar for $b,c,d,e,f$.  Due to the symmetry of the problem, we can assume that $d^- \new{\leq} \min(e^-,f^-)$ and $e^+\new{\geq}\max(d^+,f^+)$. Because $d$ and $e$ are both contained in $a$, this would mean that $f$ were also contained in $a$, a contradiction since $\po{f}{a}$ is required.

\end{proof}

\new{The previous results remains valid if we replace $\bpp$ by either $\bntpp$ or $\btpp$. Hence we obtain the following corollaries.
\begin{corollary}\label{propRealizationPOTPP}
There exists a consistent atomic RCC8 network over $\{\bpo,\btpp,\btppi\}$ which has no convex solution in $\mathbb{R}$.
\end{corollary}
\begin{corollary}\
There exists a consistent atomic RCC8 network over $\{\bpo,\bntpp,\bntppi\}$ which has no convex solution in $\mathbb{R}$.
\end{corollary}}

\begin{proposition}
There exists a consistent atomic RCC8 network over $\{\btpp,\btppi,\bntpp,\bntppi\}$ which has no convex solution in $\mathbb{R}$.
\end{proposition}
\begin{proof}
Consider the following atomic network:
\begin{align}\label{eqThetaTPPNTPP}
\Theta = \{&\tpp{a_1}{a_2},\tpp{a_1}{a_3},\ntpp{a_1}{a_4},\\
&\tpp{a_2}{a_3},\tpp{a_2}{a_4},\tpp{a_3}{a_4}\}\notag
\end{align}
Suppose $\mathcal{S}$ were a convex solution of $\Theta$ in $\mathbb{R}$.
Let $\mathcal{S}(a_i)=[u_i,v_i]$.  From $\tpp{a_1}{a_2}$ we derive that either $u_1= u_2$ or $v_1=v_2$. Assume for instance $u_1=u_2$ (the other case is entirely analogous). Then we also have $v_1 < v_2$.  From $\tpp{a_2}{a_3}$ we moreover find $v_1 < v_2\leq v_3$. Since $\tpp{a_1}{a_3}$, this means that $u_1=u_2=u_3$ and thus $v_2 < v_3$. Similarly, from $\tpp{a_2}{a_4}$ we find $u_2 = u_4$. However, since $u_1=u_2$ it follows that $u_1 = u_4$ and thus $\mathcal{S}$ violates the constraint $\ntpp{a_1}{a_4}$.
\end{proof}

\begin{proposition}
There exists a consistent atomic RCC8 network over $\{\bec\}$ which has no convex solution in $\mathbb{R}$.
\end{proposition}
\begin{proof}
It is easy to verify that $\{\ec{a}{b},\ec{a}{c},\ec{b}{c}\}$ has no convex solution in $\mathbb{R}$.
\end{proof}

\begin{corollary}
There exist consistent atomic RCC8 networks over the following set of base relations which have no convex solution in $\mathbb{R}$:
\begin{itemize}
\item $\{\bdc,\btpp,\btppi,\bntpp,\bntppi\}$
\item $\{\bpo,\btpp,\btppi,\bntpp,\bntppi\}$
\item $\{\bec,\bntpp,\bntppi\}$
\item $\{\bec,\btpp,\btppi\}$
\item $\{\bec,\btpp,\btppi,\bntpp,\bntppi\}$
\end{itemize}
\end{corollary}

\subsubsection{Upper bounds}\label{subsecUpper2D}

\begin{figure}
\centering
\includegraphics[width=240pt]{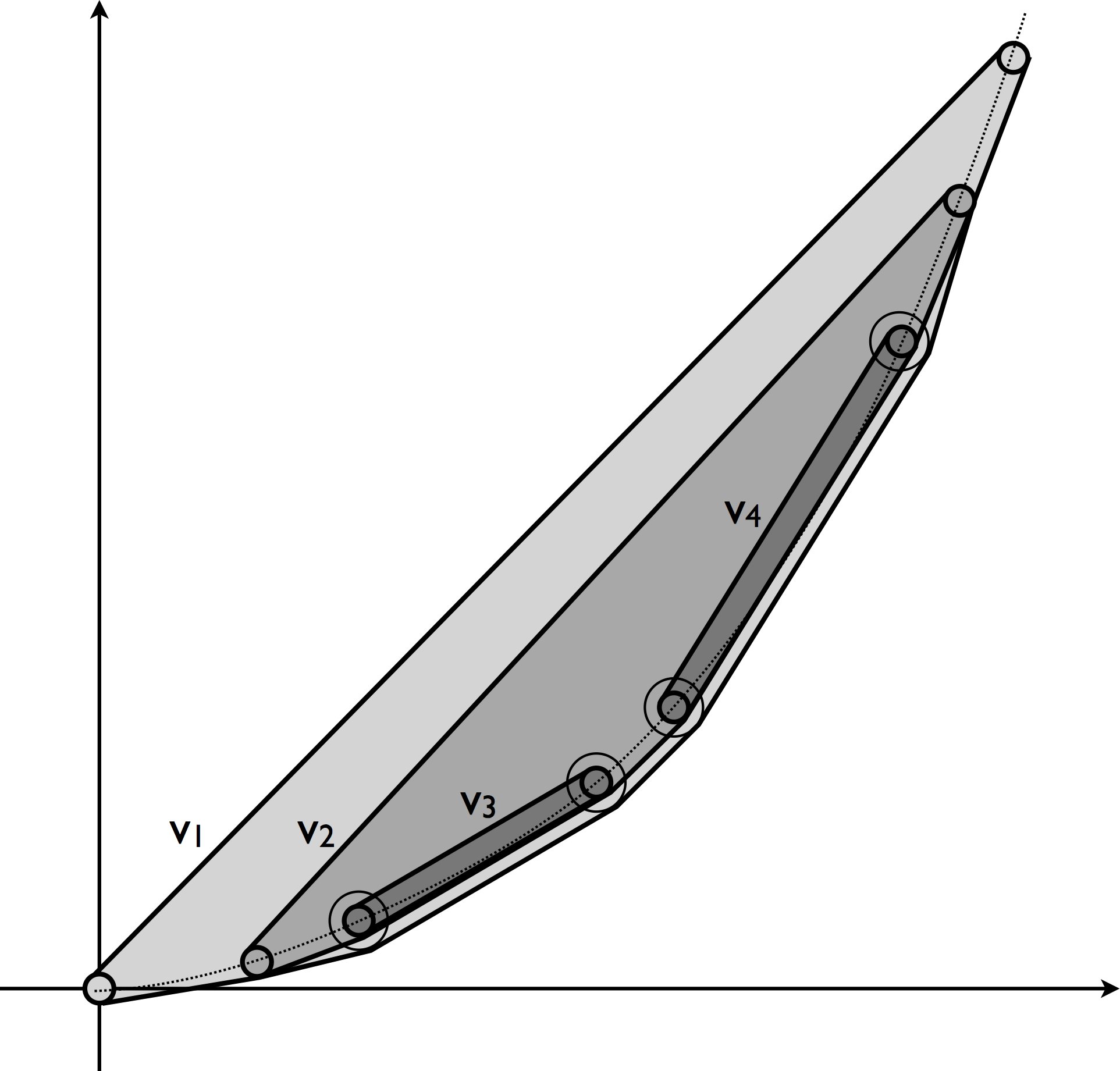}
\caption{\new{Convex solution of the network $\Theta$ from Example \ref{Prop9ex}.\label{figProp9ex}}}
\end{figure}

\new{We now prove for a number of fragments of RCC8 that consistent atomic networks can always be convexly realized in $\mathbb{R}^2$. The following example illustrates the basic construction which is used in the proofs.
\begin{example}\label{Prop9ex}
Consider the RCC8 network $\Theta$ defined by
\begin{align*}
\Theta &= \{\tpp{v_2}{v_1},\tpp{v_3}{v_2},\tpp{v_4}{v_2},\ntpp{v_3}{v_1},\ntpp{v_4}{v_1}, \dc{v_3}{v_4}\}
\end{align*}
It is easy to verify that this network does not have a convex solution in $\mathbb{R}$. However, as Figure \ref{figProp9ex} shows, it is possible to construct a convex solution of $\Theta$ in $\mathbb{R}^2$.
\end{example}}

\noindent \new{In a similar way, we can realize any consistent atomic RCC8 network $\Theta$ over $\{\bdc,\btpp, \btppi,\bntpp,\bntppi\}$. To show this, we first consider of a weakened version $\Psi$ of $\Theta$, in which $\btpp$ and $\btppi$ are weakened to $\bpp$ and $\bppi$. From Proposition \ref{propRealizableDCNTPP1D} we know that $\Psi$ can be realized using intervals of the real line. Let $[l_i,u_i]$ be the realization of $v_i$ and assume that $l_i,u_i \geq 0$. We can then construct a convex solution $\mathcal{T}$ of $\Theta$ as follows:
$$
\mathcal{T}(v_i) = \ch(\bigcup\{ S_{\theta_{ij}}(l_j,l_j^2) \cup S_{\theta_{ij}}(u_j,u_j^2)  \,|\, \Theta \models \p{v_j}{v_i}\}
$$
where $\ch$ denotes the convex hull and $S_{\theta_{ij}}(l_j,l_j^2)$ is the disc with radius $\theta_{ij}$ and center $(l_j,l_j^2)$. By carefully choosing the values of $\theta_{ij}$ we can ensure that $\mathcal{T}$ is indeed a solution of $\Theta$, as we show in the proof of the following proposition. Note that the points $(l_i,l_i^2)$ and $(u_i,u_i^2)$ are on the positive half of the two-dimensional moment curve. Sections \ref{secFragments3D} and \ref{secFragments4D} will use constructions that are based on choosing points on the three-dimensional and four-dimensional moment curve, respectively.} 

\new{We say that a region $v\in V'$ is at level $t$ if there is a chain of $t-1$ regions $u_1,...,u_{t-1}$ in $V'$ such that $\Theta\models \{\ntpp{u_1}{u_2},...,\ntpp{u_{t-2}}{u_{t-1}},\ntpp{u_{t-1}}{v}\}$ and there is no chain of $t$ regions for which this is the case.   Let $\Delta(v)$ be the level of region $v\in V'$. Note that $\Delta(v)=1$ iff there is no region $u\in V'$ such that $\Theta\models \ntpp{u}{v}$.  } 
\begin{proposition}\label{propRealizationDCTPPNTPP2d}
Every consistent atomic RCC8 network  over $\{\bdc,\btpp, \btppi,\bntpp,\bntppi\}$ has a convex solution in $\mathbb{R}^2$.
\end{proposition}
\begin{proof}
Let $\Theta$ be a consistent atomic RCC8 network over $\{\bdc,\btpp, \btppi,\bntpp,\bntppi\}$. We show that $\Theta$ can be convexly realized in $\mathbb{R}^2$. Let $\Psi$ be the atomic RCC8 network which is obtained from $\Theta$ by replacing every constraint of the form $\tpp{v_i}{v_j}$ by $\ntpp{v_i}{v_j}$ and  every constraint of the form $\tppi{v_i}{v_j}$ by $\ntppi{v_i}{v_j}$. Clearly, $\Psi$ is consistent, and hence by Proposition \ref{propRealizableDCNTPP1D}, we know that $\Psi$ has a convex solution $\mathcal{S}$ in $\mathbb{R}$. Let $\mathcal{S}(v_i) = [l_i,u_i]$. \new{Without loss of generality, we can assume that $l_i>0$ for each $i$.} With each region $v_i$ we associate two points $p_i, q_i \in \mathbb{R}^2$:
\begin{align*}
p_i &= (l_i,l_i^2)\\
q_i &= (u_i,u_i^2)
\end{align*}
\new{Let $\theta>0$ be a constant. We can then define a convex solution $\mathcal{T}$ as follows:
\begin{align*}
\mathcal{T}(v_j) = \ch \big(S_{\theta}(p_j) \cup S_{\theta}(q_j)   &\cup \bigcup \{S_{\theta}(p_{l}) \cup S_{\theta}(q_{l}) \,|\, \Theta\models \tpp{v_l}{v_j}\}  \\
& \cup \bigcup \{S_{\Delta(v_j) \cdot \theta}(p_{l}) \cup S_{\Delta(v_j) \cdot \theta}(q_{l}) \,|\, \Theta\models \ntpp{v_l}{v_j}\}  \big)
\end{align*}
where for a point $p$, $S_{\theta}(p) = \{q \,|\, d(q,p)\leq \theta\}$ is a disc with radius $\theta>0$ centered around $p$. By choosing $\theta$ sufficiently small, we can ensure that $S_{m \cdot \theta}(p') \cap S_{m \cdot \theta}(q')=\emptyset$ if $p'\not=q'$, where $m$ denotes the number of regions in $\Theta$. Accordingly, we can assume that $\dc{\mathcal{T}(v_i)}{\mathcal{T}(v_j)}$ holds for any regions $v_i, v_j$ such that $\Theta \models \dc{v_i}{v_j}$. By construction, it is immediate that $\p{\mathcal{T}(v_i)}{\mathcal{T}(v_j)}$ if $\Theta \models \p{v_i}{v_j}$. If $\theta$ is sufficiently small, $S_{\theta}(p_j)$ and $S_{\theta}(q_j)$ will not be included in $\ch(\bigcup_{l\neq j} S_{m\cdot \theta}(p_l) \cup S_{m\cdot \theta}(q_l)$), for any $j$. Then we have that $\pp{\mathcal{T}(v_i)}{\mathcal{T}(v_j)}$ if $\Theta \models \pp{v_i}{v_j}$. If $\Theta \models \ntpp{v_i}{v_j}$ we have that $\Delta(v_j)>\Delta(v_i)$, from which it is easy to see that $\ntpp{\mathcal{T}(v_i)}{\mathcal{T}(v_j)}$ will hold. Finally, if $\Theta \models \tpp{v_i}{v_j}$, $\mathcal{T}(v_i)$ and $\mathcal{T}(v_j)$ will share a boundary point of $S_{\theta}(p_i)$ and $S_{\theta}(q_i)$, and thus we can assume that $\tpp{\mathcal{T}(v_i)}{\mathcal{T}(v_j)}$.}

\end{proof}

\new{The next fragment we consider is $\{\bpo,\btpp,\btppi,\bntpp,\bntppi\}$. Before presenting the proof, we first give an example.
\begin{example}\label{Prop10ex}
Consider the RCC8 network $\Theta$ defined by
\begin{align*}
\Theta &= \{\tpp{v_2}{v_1},\tpp{v_3}{v_2},\tpp{v_4}{v_2},\ntpp{v_3}{v_1},\ntpp{v_4}{v_1}, \po{v_3}{v_4}\}
\end{align*}
An example of a convex solution in $\mathbb{R}^2$ is shown in Figure \ref{figProp10ex}.
\end{example}}
\new{More generally, we fix two points $q_0$ and $q_1$ on the positive half of the two-dimensional moment curve, as well as one point $p_i$ for each variable $v_i$. As we will show in the proof of Proposition \ref{prop10}, we can always find a solution $\mathcal{S}$ of the following form:
\begin{align*}
\mathcal{S}(v_j) = \ch \big(\{p_j\} \cup S_{\theta'}(q_0) \cup S_{\theta''}(q_1)  \cup \bigcup\{S_{\theta_l}(p_{l}) \,|\, \Theta \models \pp{v_l}{v_j}\}  \big)
\end{align*}
The proof of Proposition \ref{prop10} will make clear how $\theta', \theta'', \theta_l $ can be chosen such that $\mathcal{S}$ is a solution of $\Theta$.}

\begin{figure}
\centering
\includegraphics[width=240pt]{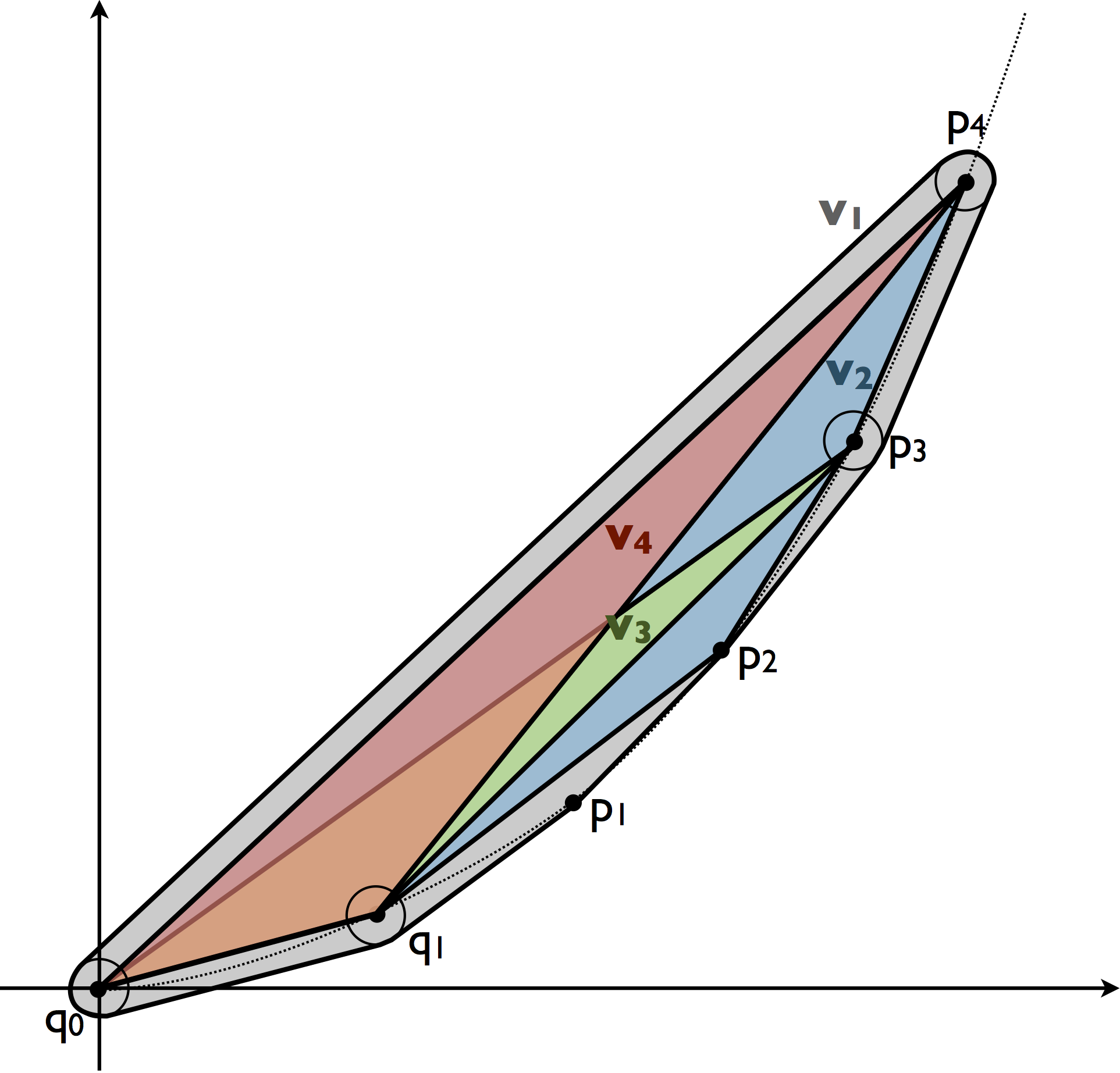}
\caption{\new{Convex solution of the network $\Theta$ from Example \ref{Prop10ex}.\label{figProp10ex}}}
\end{figure}

\begin{proposition}\label{prop10}
Every consistent atomic RCC8 network over $\{\bpo,\btpp,\btppi,\bntpp,\bntppi\}$ has a convex solution in $\mathbb{R}^2$.
\end{proposition}
\begin{proof}
\new{Let $\Theta$ be a consistent atomic RCC8 network over $\{\bpo,\btpp,\btppi,\bntpp,\bntppi\}$.}
Let $q_0 = (0,0)$, $q_1 = (1,1)$ and $p_i = (i,i^2)$\new{, and let $\theta>0$ be a constant. We define $\mathcal{S}$ for $v_j\in \Theta$ as follows:
\begin{align*}
\mathcal{S}(v_j) = \ch \big(&\{p_j\} \cup S_{\Delta(v_j) \cdot \theta}(q_0) \cup S_{\Delta(v_j)\cdot \theta}(q_1)  \cup \bigcup\{p_{l} \,|\, \Theta \models \tpp{v_l}{v_j}\} \\
&\cup \bigcup\{S_{ \Delta(v_j)\cdot \theta}(p_{l}) \,|\, \Theta \models \ntpp{v_l}{v_j}\}  \big)
\end{align*}
Entirely analogously as in the proof of Proposition \ref{propRealizationDCTPPNTPP2d}, we can then show that $\mathcal{S}$ is a solution of $\Theta$, provided that $\theta$ is chosen sufficiently small.
}

\end{proof}

\new{The last fragment we consider in this section is $\{\bec,\btpp,\btppi,\bntpp,\bntppi\}$. Again we start with an example.
\begin{example}\label{Prop11ex}
Consider the RCC8 network $\Theta$ defined by
\begin{align*}
\Theta &= \{\tpp{v_2}{v_1},\tpp{v_3}{v_1}, \ec{v_2}{v_3},\tpp{v_1}{v_4},\ntpp{v_2}{v_4},\tpp{v_3}{v_4}\}
\end{align*}
An example of a convex solution in $\mathbb{R}^2$ is shown in Figure \ref{figProp11ex}.
\end{example}}
\new{More generally, note that we can always partition the set of variables $V$ in two sets $V'$ and $V\setminus V'$ such that $\Theta{\downarrow}V'$ is a network over $\{\bec,\btpp,\btppi\}$, $\Theta{\downarrow}(V\setminus V')$ is a network over $\{\btpp,\btppi,\bntpp,\bntppi\}$ and $\pp{v_i}{v_j}$ for any $v_i \in V'$ and $v_j \in V\setminus V'$ (see the following proof of Proposition \ref{propRealizability2D-EC-TPP-BNTPP}). It is then straightforward to find a convex solution of  $\Theta{\downarrow}V'$. As we show in the proof of Proposition \ref{propRealizability2D-EC-TPP-BNTPP}, this solution of $\Theta{\downarrow}V'$ can always be extended to a solution of $\Theta$.}

\begin{figure}
\centering
\includegraphics[width=240pt]{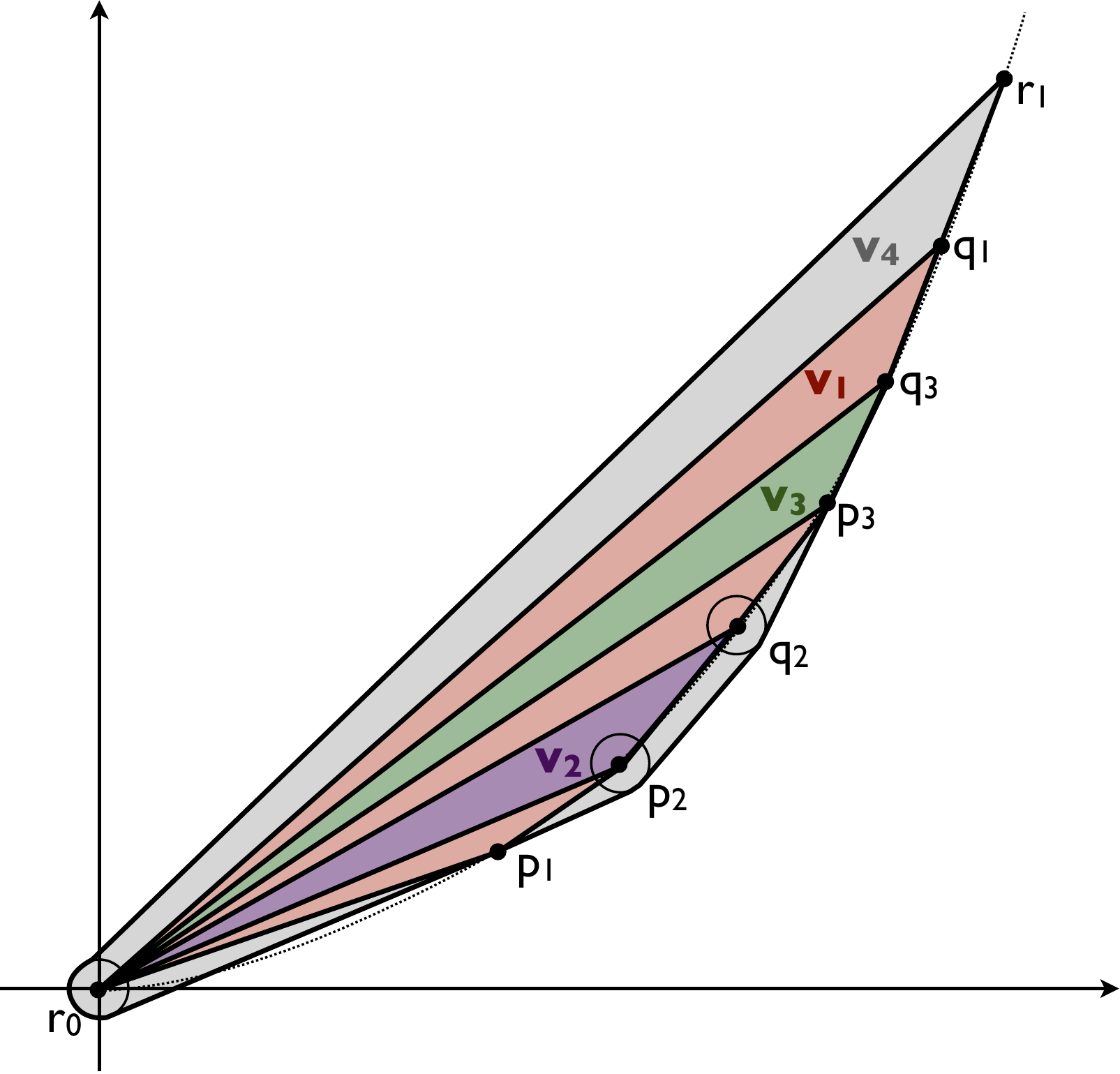}
\caption{\new{Convex solution of the network $\Theta$ from Example \ref{Prop11ex}.\label{figProp11ex}}}
\end{figure}

\begin{proposition}\label{propRealizability2D-EC-TPP-BNTPP}
Every consistent atomic RCC8 network over $\{\bec,\btpp,\btppi,\bntpp,\bntppi\}$ has a convex solution in $\mathbb{R}^2$.
\end{proposition}
\begin{proof}
\new{Let $\Theta$ be a consistent atomic RCC8 network over $\{\bec,\btpp,\btppi,\bntpp,\bntppi\}$.}
First note that it is not possible to have  $\ec{a}{b}$ and $\ntpp{c}{a}$ for any $a,b,c \in V$ since this would entail $\dc{c}{b}$, i.e.\ $\Theta$ would not be an atomic network over $\{\bec,\btpp,\btppi,\bntpp,\bntppi\}$. Let $V' = \{v \,|\, v\in V, \exists u\in V\,.\, \Theta\models \ec{u}{v}\}$. Then $\Theta{\downarrow}V'$ is a network over $\{\bec,\btpp,\btppi\}$,  $\Theta{\downarrow}(V\setminus V')$ is a network over $\{\btpp,\btppi,\bntpp,\bntppi\}$ and for each $v_i \in V'$ and $v_j \in V\setminus V'$ it holds that $\Theta \models \pp{v_i}{v_j}$.

Let $\Psi$ be the network obtained from $\Theta{\downarrow}V'$ by replacing all occurrences of $\bec$ by $\bdc$, all occurrences of $\btpp$ by $\bntpp$ and all occurrences of $\btppi$ by $\bntppi$. From Proposition \ref{propRealizableDCNTPP1D} we know that $\Psi$ has a convex solution $\mathcal{S}$ in $\mathbb{R}$. Assume that $\mathcal{S}(v_i) = [l_i,u_i]$.  
\new{Without loss of generality, we can assume that $l_i>0$ for each $i$. With each region $v_i$ we associate two points $p_i, q_i \in \mathbb{R}^2$:
\begin{align*}
p_i &= (l_i,l_i^2)\\
q_i &= (u_i,u_i^2)
\end{align*}}
Now we define a solution $\mathcal{T}$ in $\mathbb{R}^2$.  For $v_j \in V'$, define:
$$
\mathcal{T}(v_j) = \ch\big(\{r_0,p_j,q_j\} \cup \{p_l \,|\, \Theta \models \tpp{v_l}{v_j}\} \cup \{q_l \,|\, \Theta \models \tpp{v_l}{v_j}\} \big)
$$
where $r_0 = (0,0)$. Clearly, $\mathcal{T}$ is a solution of $\Theta{\downarrow}V'$.  Now we extend $\mathcal{T}$ to a solution of $\Theta$.  
\new{Let $V\setminus V' = \{a_1,...,a_k\}$. With each $a_i$ we associate an arbitrary point $r_i = (w_i,w_i^2)$ on the positive half of the two-dimensional moment curve such that all the points $p_i$, $q_j$ and $r_l$ are distinct. Let $\theta>0$ be a constant. For $a_j \in V\setminus V'$, we define:
\begin{align*}
\mathcal{T}(a_j) = \ch\big(&\{r_j\}  \cup S_{\Delta(a_j)\cdot \theta}(r_0) \cup \bigcup \{p_{l}, q_{l} \,|\, \Theta\models \tpp{v_l}{a_j}\}  \\
& \cup \bigcup \{S_{\Delta(a_j) \cdot \theta}(p_{l}) \cup S_{\Delta(a_j) \cdot \theta}(q_{l}) \,|\, \Theta\models \ntpp{v_l}{a_j}\}  \big)
\end{align*}
}
%
%
As in the proof of Proposition \ref{propRealizationDCTPPNTPP2d} we can show that $\mathcal{T}$ is a solution of $\Theta$, \new{provided that $\theta$ is chosen sufficiently small}.
\end{proof}

\begin{corollary}
\new{Every consistent atomic RCC8 network over the following sets of base relations has a convex solution in $\mathbb{R}^2$}:
\begin{itemize}
\item $\{\bdc,\btpp,\btppi\}$
\item $\{\bpo,\bntpp,\bntppi\}$
\item $\{\btpp,\btppi,\bntpp,\bntppi\}$
\item $\{\bpo,\btpp,\btppi\}$
\item $\{\bec,\bntpp,\bntppi\}$
\item $\{\bec,\btpp,\btppi\}$
\item $\{\bec\}$
\end{itemize}
\end{corollary}

\begin{corollary}
Every consistent atomic RCC5 \new{network} over $\{\bpo,\bpp,\bppi\}$ has a convex \new{solution} in $\mathbb{R}^2$.
\end{corollary}

\subsection{Fragments with convex solutions in $\mathbb{R}^3$ \label{secFragments3D}}

We now turn our attention to restrictions on the set of RCC8 or RCC5 base relations which guarantee that networks can be convexly realized in $\mathbb{R}^3$. Again, we first show in Section \ref{secLowerboundD3} that these results cannot be strengthened in general.

\begin{figure}
\centering
\includegraphics[width=150pt]{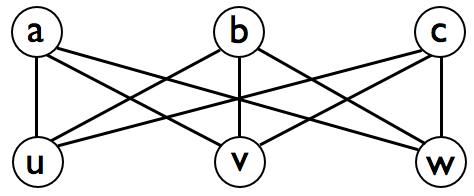}
\caption{The graph $K_{3,3}$ is not planar.\label{figK33}}
\end{figure}

\subsubsection{Lower bounds}\label{secLowerboundD3}

Recall that a graph $(N,V)$ defined by a set of nodes $N$ and a set of edges $V \subseteq N\times N$ is called \textit{planar} if it is possible to identify every node $n_i\in N$ with a point $p_i \in \mathbb{R}^2$ and every edge $(n_i,n_j)$ with a planar curve segment $A_{ij} \subseteq \mathbb{R}^2$ with endpoints $p_i$ and $p_j$, such that for every $A_{ij} \neq A_{rs}$ it holds that $A_{ij} \cap A_{rs} = \{p_i,p_j\} \cap \{p_r,p_s\}$. The graph $K_{3,3}$ depicted in Figure \ref{figK33} is a well-known example of a non-planar graph.

The proofs of Propositions \ref{propDRPOD2lower} and \ref{propECDCD2lower} below proceed by showing that if particular RCC8 networks were convexly realizable in $\mathbb{R}^2$, then $K_{3,3}$ would be planar, thus showing the propositions by contradiction.

\begin{proposition}\label{propDRPOD2lower}
There exists a consistent atomic RCC5 network over $\{\bdr,\bpo\}$ which has no convex solution in $\mathbb{R}^2$.
\end{proposition}
\begin{proof}
Consider the consistent atomic RCC5 network $\Theta$ containing the following constraints:
\begin{align*}
 \po{r_{au}}{a}&&  \po{r_{au}}{u}&& \po{r_{av}}{a}&&\po{r_{av}}{v}&&\po{r_{aw}}{a}&&\po{r_{aw}}{w}\\
 \po{r_{bu}}{b}&&\po{r_{bu}}{u}&& \po{r_{bv}}{b}&&\po{r_{bv}}{v}&&\po{r_{bw}}{b}&&\po{r_{bw}}{w},\\
 \po{r_{cu}}{c}&&\po{r_{cu}}{u}&& \po{r_{cv}}{c}&&\po{r_{cv}}{v}&&\po{r_{cw}}{c}&&\po{r_{cw}}{w} 
\end{align*}
and moreover the constraint $\dr{x}{y}$ for every pair of variables $(x,y)$ from $\{a,b,c,u,v,w,r_{au},r_{av},r_{aw},r_{bu},r_{bv},r_{bw},r_{cu},r_{cv},r_{cw}\}$ which does not appear in the list above.

Assume that a convex solution $\mathcal{S}$ of $\Theta$ in the plane exists.  We can then draw $K_{3,3}$ in the plane as follows. \new{For the} nodes of $K_{3,3}$, \new{we can choose arbitrary points $p_a, p_b, p_c, p_u, p_v, p_w$ in the interiors of respectively} $\mathcal{S}(a), \mathcal{S}(b), \mathcal{S}(c), \mathcal{S}(u), \mathcal{S}(v), \mathcal{S}(w)$.  To connect $p_a$ with $p_u$, first, we connect $p_a$ by a line segment to an arbitrary point of $\delta(\mathcal{S}(a)) \cap i(\mathcal{S}(r_{au}))$, where $\delta(X)$ denotes the boundary of $X$ and $i(X)$ is the interior of $X$. This point is connected by a line segment to an arbitrary point of $\delta(\mathcal{S}(u)) \cap i(\mathcal{S}(r_{au}))$. The latter point is finally connected with a line segment to $p_u$.  In a similar way we connect each of $p_a,p_b,p_c$ to each of $p_u,p_v,p_w$.  It is clear that the edges thus drawn can only intersect at their endpoints, i.e.\ we have shown that $K_{3,3}$ is planar, a contradiction.
\end{proof}
\new{Note that this proof moreover shows that there are RCC5 networks over $\{\bdr,\bpo\}$ which have no solution in $\mathbb{R}^2$ using \emph{internally connected} regions.}
The previous proof remains valid if we replace $\bdr$ by either $\bec$ or $\bdc$, hence we obtain the following corollary.
\begin{corollary}
There exist consistent atomic RCC8 networks over $\{\bdc,\bpo\}$ and over $\{\bec,\bpo\}$ which have no convex solution in $\mathbb{R}^2$.
\end{corollary}

\begin{proposition}\label{propECDCD2lower}
There exists a consistent atomic RCC8 network over $\{\bec,\bdc\}$ which has no convex solution in $\mathbb{R}^2$.
\end{proposition}
\begin{proof}
Consider the consistent atomic RCC8 network $\Theta$ containing the following constraints:
\begin{align*}
 \ec{a}{u} && \ec{a}{v} && \ec{a}{w} \\
 \ec{b}{u} && \ec{b}{v} && \ec{b}{w}\\
 \ec{c}{u} && \ec{c}{v} && \ec{c}{w}
\end{align*}
as well as the constraint $\dc{x}{y}$ for every pair of variables $(x,y)$ from $\{a,b,c,u,v,w\}$ which does not appear in the list above.

Assuming that a convex solution $\mathcal{S}$ of $\Theta$ in the plane exists, we can draw $K_{3,3}$ as follows.  \new{For the nodes of $K_{3,3}$, we choose arbitrary points  $p_a, p_b, p_c, p_u, p_v, p_w$ in the interiors of respectively} $\mathcal{S}(a), \mathcal{S}(b), \mathcal{S}(c), \mathcal{S}(u), \mathcal{S}(v), \mathcal{S}(w)$. For $x\in \{a,b,c\}$ and $y\in \{u,v,w\}$, let $q_{xy}$ be an arbitrary point from $\mathcal{S}(x) \cap \mathcal{S}(y)$ (which must exist because of the constraint $\ec{x}{y}$). Note that $q_{xy} \notin \mathcal{S}(x')$ and $q_{xy} \notin \mathcal{S}(y')$ for $x'\in \{a,b,c\} \setminus \{x\}$ and $y'\in \{u,v,w\} \setminus \{y\}$ because of the constraints $\dc{x}{x'}$ and $\dc{y}{y'}$.  Then we can connect $x$ and $y$ as follows:  draw a line segment from $p_x$ to $q_{x,y}$ and a second line segment from $q_{x,y}$ to $p_y$.  It is clear that the edges thus drawn can only intersect at their endpoints, which means that $K_{3,3}$ would be planar, a contradiction.
\end{proof}
\new{Again the proof shows a slightly stronger property than what we need, i.e.\ that there are consistent RCC8 networks over $\{\bec,\bdc\}$ which have no solution in $\mathbb{R}^2$ using internally connected regions.}
\subsubsection{Upper bounds}

\new{We now show that consistent atomic RCC8 networks over $\{\bec,\bdc, \bpo\}$ and  $\{\bec,\bdc, \bntpp, \bntppi\}$ can be realized in $\mathbb{R}^3$. Where the results in Section \ref{subsecUpper2D} rely on a construction based on the two-dimensional moment curve, here we will use a property of the three-dimensional moment curve.}

\begin{proposition}\label{propECDCPO3d}
Every consistent atomic RCC8 network over $\{\bec,\bdc, \bpo\}$ has a convex solution in $\mathbb{R}^3$.
\end{proposition}
\begin{proof}
Let $\Theta$ be a consistent atomic RCC8 network over $\{\bec,\bdc, \bpo\}$ involving the variables in $V = \{v_1,...,v_m\}$.  Let $\{A_1,...,A_m\}$ be a neighborly family of convex polytopes in $\mathbb{R}^3$ such that each two polytopes $A_i, A_j$ share a face. \new{It is known that an arbitrarily large family of this kind is obtained by taking the cells of the Voronoi diagram induced by any set of points on the positive half of the three-dimensional moment curve (see Section \ref{secGeometry}).}

Note that $\ec{A_i}{A_j}$ holds for any $i\neq j$. We now modify these polytopes $A_1$,...,$A_n$ to obtain a solution of $\Theta$. Let $H_{ij}$ be the unique plane which contains $A_i\cap A_j$.  Let $H_{ij}^{\varepsilon}$ be the plane which is parallel to $H_{ij}$, at distance $\varepsilon$ from $H_{ij}$ and which is in the same half-space (induced by $H_{ij}$) as $A_i$. Let $H_{ij}^{\geq\varepsilon}$ be the half-space induced by $H_{ij}^{\varepsilon}$ which does not contain $H_{ij}$. Let $c_{ij}$ be \new{an arbitrary interior point of $A_i \cap A_j$}.  Finally, let $L_{ij}$ be the line through $c_{ij}$ which is orthogonal to $H_{ij}$ and let $c_{ij}^{\varepsilon} = L_{ij} \cap H_{ij}^{\varepsilon}$. In the following we assume that $\varepsilon>0$ is \new{chosen small enough such that} $c_{ij}^{\varepsilon} \in i(A_i)$.

If $\Theta\models \dc{a_i}{a_j}$, we replace $A_i$ by $A_i \cap H_{ij}^{\geq\varepsilon}$; the region $A_j$ can be similarly replaced by a smaller region, but it is sufficient that one of $A_i, A_j$ is modified. It is clear that then $\dc{A_i}{A_j}$ and that this operation does not affect the RCC8 relations that hold between the other pairs of regions (assuming $\varepsilon$ is sufficiently small). \new{Indeed, if $\ec{A_i}{A_z}$ then $A_i$ and $A_z$ are sharing a two-dimensional face ($z\neq i,j$). By replacing $A_i$ by $A_i \cap H_{ij}^{\geq\varepsilon}$, this face may be replaced by a smaller face, but provided that $\varepsilon$ is sufficiently small, $A_i\cap A_z$ will still be non-empty, and a two-dimensional shared face.}

\new{Once all $\bdc$ relations have been made to hold using this process, we turn to the relations of the form $\bpo$.}
If $\Theta\models \po{a_i}{a_j}$, we replace $A_j$ by $\cvx(A_j \cup \{c_{ij}^{\varepsilon}\})$; we may similarly replace $A_i$ by a larger region, but it is again sufficient that one of $A_i,A_j$ is modified. Then we have that $c_{ij}^{\varepsilon}$ is an interior point of $A_j$ and thus $\po{A_i}{A_j}$, and this operation does not affect the RCC8 relations that hold between the other pairs of regions \new{if we choose $\varepsilon$ to be sufficiently small. Indeed, for a sufficiently small $\varepsilon$ it holds that $\cvx(A_j \cup \{c_{ij}^{\varepsilon}\}) \subseteq A_j \cup i(A_i)$, from which it easily follows that $\ec{(A_j \cup \{c_{ij}^{\varepsilon}\})}{A_z}$ iff $\ec{A_j}{A_z}$ and $\dc{(A_j \cup \{c_{ij}^{\varepsilon}\})}{A_z}$ iff $\dc{A_j}{A_z}$, for any $A_z$ ($z\neq i,j$)}.
\end{proof}

\begin{corollary}
Every consistent atomic RCC5 network over $\{\bdr, \bpo\}$ has a convex solution in $\mathbb{R}^3$.
\end{corollary}

\begin{proposition}\label{propRCDCNTPP3d}
Every consistent atomic RCC8 network over $\{\bec,\bdc, \bntpp, \bntppi\}$ has a convex solution in $\mathbb{R}^3$.
\end{proposition}
\begin{proof}
\new{Let $\Theta$ be a consistent atomic RCC8 network over $\{\bec,\bdc, \bntpp, \bntppi\}$.}
We show this proposition by induction on the number $k$ of $\bntpp$ or $\bntppi$ relations in $\Theta$. If $k=0$, then the result follows from Proposition \ref{propECDCPO3d}. Assume that $k>0$.  Let $V'$ be the set of regions which are not contained in any other region, i.e. $v\in V'$ iff there is no region $u\in V$ such that $\Theta \models \ntpp{v}{u}$. The network $\Psi = \Theta{\downarrow}V'$ satisfies the conditions of Proposition \ref{propECDCPO3d} and thus has a convex solution $\mathcal{S}$ in $\mathbb{R}^3$.  We extend $\mathcal{S}$ to a solution of $\Theta$ as follows. Let $v\in V'$ and let $U_v = \{u \in V \,|\, \Theta\models \ntpp{u}{v}\}$. Then $\Theta{\downarrow}U_v$ has strictly fewer occurrences of $\bntpp$ and $\bntppi$ and thus has a convex solution $\mathcal{T}_v$ by induction. We can modify this solution $\mathcal{T}_v$ by a linear transformation to a solution $\mathcal{S}_v$ such that for all $u\in U_v$ it holds that $\mathcal{S}_v(u) \subseteq i(\mathcal{S}(v))$. \new{Indeed, any linear transformation will preserve convexity and topological relations (including those in RCC8).} Noting that $\Theta\models \dc{u_1}{u_2}$  for $u_1 \in U_{v_1}$ and $u_2 \in U_{v_2}$, with $v_1 \neq v_2$, it follows easily that $\mathcal{S}$ and $\{\mathcal{S}_v \,|\, v \in V' \}$ together define a convex solution of $\Theta$.
\end{proof}

\subsection{Fragments with convex solutions in $\mathbb{R}^4$ \label{secFragments4D}}

Somewhat surprisingly, there are two fragments in Figure \ref{overviewRCC8} for which consistent atomic networks are only guaranteed to have a convex solution if the number of dimensions is at least four. The largest of these fragments contains all relations apart from $\bpo$, which is an important fragment from a practical point of view, since there are many applications in which the relation $\bpo$ is not used. 

\subsubsection{Lower bounds}

While it is relatively straightforward to find examples of consistent RCC8 networks which do not have a convex solution $\mathbb{R}^2$, it is much harder to find consistent RCC8 networks which do not have a convex solution in $\mathbb{R}^3$ (but which have a convex solution in $\mathbb{R}^4$). The main idea of the following proof is to start with a constraint $\ec{a}{b}$ and then essentially construct a 2D counterexample in $a\cap b$.

\begin{figure}
\centering
\includegraphics[width=200pt]{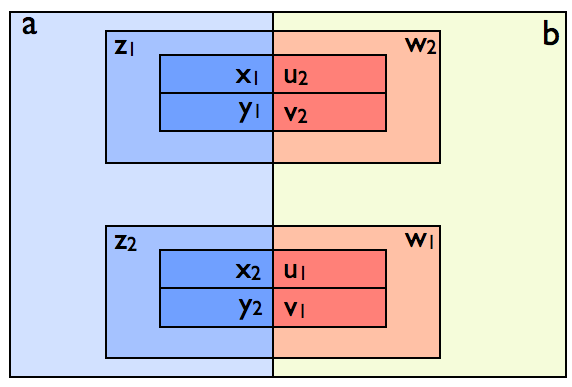}
\caption{Realization of $\Theta{\downarrow}\{a,b,x_1,x_2,y_1,y_2,z_1,z_2,u_1,u_2,v_1,v_2,w_1,w_2\}$ with $\Theta$ as in the proof of Proposition \ref{propECDCTPP3lower}. \label{figpropECDCTPP3lower}}
\end{figure}
\begin{proposition}\label{propECDCTPP3lower}
There exists a consistent atomic RCC8 network over $\{\bec,\bdc,\btpp,\btppi\}$ which has no convex solution in $\mathbb{R}^3$.
\end{proposition}
\begin{proof}
Let $\Theta$ contain the following constraints (for all $1\leq i \neq j \leq 5$):
\footnotesize{
\begin{align*}
&\ec{w_i}{a}&&
\ec{u_i}{a}&&
\ec{v_i}{a}&&
\tpp{z_i}{a}&&
\tpp{x_i}{a}&&
\tpp{y_i}{a}&&
\ec{a}{b}\\[0.5em]
&\tpp{w_i}{b}&&
\tpp{u_i}{b}&&
\tpp{v_i}{b}&&
\ec{z_i}{b}&&
\ec{x_i}{b}&&
\ec{y_i}{b}\\[0.5em]
&\dc{z_i}{z_j}&&
\dc{z_i}{x_j}&&
\dc{z_i}{y_j}&&
\dc{z_i}{w_i}&&
\dc{z_i}{u_i}&&
\dc{z_i}{v_i} &&
\ec{z_i}{w_j}\\
&\ec{z_i}{u_j}&&
\ec{z_i}{v_j} &&
\tpp{x_i}{z_i}&&
\tpp{y_i}{z_i}
\\[0.5em]
&\dc{w_i}{w_j}&&
\dc{w_i}{u_j}&&
\dc{w_i}{v_j}&&
\dc{w_i}{x_i}&&
\dc{w_i}{y_i}&&
\ec{w_i}{x_j}&&
\ec{w_i}{y_j}\\
& \tpp{u_i}{w_i}&&
 \tpp{v_i}{w_i}&&
\\[0.5em]
& \ec{x_i}{y_i} &&
\ec{x_i}{u_j} &&
\ec{x_i}{v_j} &&
\dc{x_i}{x_j} &&
\dc{x_i}{y_j} &&
\dc{x_i}{u_i} &&
\dc{x_i}{v_i} \\[0.5em]
&  \ec{y_i}{u_j} &&
\ec{y_i}{v_j} &&
\dc{y_i}{y_j} &&
\dc{y_i}{u_i} &&
\dc{y_i}{v_i} \\[0.5em]
& \ec{u_i}{v_i} &&
\dc{u_i}{u_j} &&
\dc{u_i}{v_j} &&
\dc{v_i}{v_j}
\end{align*}}
\normalsize
For illustration purposes, Figure \ref{figpropECDCTPP3lower} depicts a realization of a restriction of $\Theta$.
Assume that a convex solution $\mathcal{S}$ in $\mathbb{R}^3$ existed.
For each $i$, let $H_i$ be a plane separating $\mathcal{S}(x_i)$ from $\mathcal{S}(y_i)$ and let $G_i$ be a plane separating $\mathcal{S}(u_i)$ from $\mathcal{S}(v_i)$.  Let $K$ be a plane separating $\mathcal{S}(a)$ and $\mathcal{S}(b)$.

For $i\neq j$, let $p_{ij}$ be a point in $H_i \cap G_j \cap K \cap \mathcal{S}(z_i) \cap \mathcal{S}(w_j)$.  To see why such a point exists, note that because of the constraints $\ec{x_i}{u_j}$, $\ec{x_i}{v_j}$, $\ec{y_i}{u_j}$, $\ec{y_i}{v_j}$ there exist points $p_{x_iu_j} \in \mathcal{S}(x_i) \cap \mathcal{S}(u_j) \cap K$, $p_{x_iv_j} \in \mathcal{S}(x_i) \cap \mathcal{S}(v_j) \cap K$, $p_{y_iu_j} \in \mathcal{S}(y_i) \cap \mathcal{S}(u_j)\cap K$ and $p_{y_iv_j} \in \mathcal{S}(y_i) \cap \mathcal{S}(v_j)\cap K$.  Since $p_{x_iu_j}$ and $p_{x_iv_j}$ are at different sides of the hyperplane $G_j$ and both points are contained in the convex region $\mathcal{S}(x_i) \cap \mathcal{S}(w_j)$, there exists a point $q_1 \in \mathcal{S}(x_i) \cap G_j \cap K \cap \mathcal{S}(w_j)$ which is collinear with and between $p_{x_iu_j}$ and $p_{x_iv_j}$.  Similarly we find that there must exist a point  $\mathcal{S}(q_2) \in \mathcal{S}(y_i) \cap G_j \cap K \cap \mathcal{S}(w_j)$ which is between $p_{y_iu_j}$ and $p_{y_iv_j}$. Moreover, the points $q_1$ and $q_2$ are at opposite sides of $H_i$ and both points are contained in the convex set $z_i$. It follows that there must be a point $p_{ij} \in H_i \cap G_j \cap K \cap \mathcal{S}(z_i) \cap \mathcal{S}(w_j)$.

We define the line segments $L_i = \ch(\{p_{ij}\,|\, j \neq i\}) \subseteq H_i \cap  K \cap \mathcal{S}(z_i)$ and $M_j = \ch(\{p_{ij}\,|\, i\neq j\}) \subseteq G_j \cap K  \cap \mathcal{S}(w_j)$. It is clear that the line segments $L_1,...,L_5$ are all disjoint and that the line segments $M_1,...,M_5$ are all disjoint. Indeed, if $L_i \cap L_j$ for $i\neq j$ contained a point $q$, we would have $q \in \mathcal{S}(z_i) \cap \mathcal{S}(z_j)$, contradicting the assumption $\dc{\mathcal{S}(z_i)}{\mathcal{S}(z_i)}$ and similarly $M_i\cap M_j\neq \emptyset$ would contradict  $\dc{\mathcal{S}(z_i)}{\mathcal{S}(z_i)}$.   

The points $p_{i1},...,p_{i5}$ occur in the same order on each line segment $L_i$.  Indeed, assume that this were not the case, and e.g.\ on the line segment $L_4$ the points $p_{41}, p_{42},p_{43}$ occur in that order, while on $L_5$ the corresponding points occur in the order $p_{51}, p_{52},p_{53}$. Since the points $p_{41}, p_{42},p_{43},p_{51}, p_{52},p_{53}$ are coplanar, this means that at least two of the line segments $p_{41}p_{51}$, $p_{42}p_{52}$, $p_{43}p_{53}$ would have a non-empty intersection,  contradicting the fact that the line segments $M_1$, $M_2$ and $M_3$ are disjoint. Similarly, we find that the points $p_{1j},...,p_{5j}$ occur in the same order on each line segment $M_j$. 

In particular, there exists a permutation $\sigma_1,...,\sigma_5$ of $\{1,...,5\}$ such that the points $p_{i\sigma_1},...,p_{i\sigma_5}$ occur in that order on each line segment $L_i$ (where $p_{i\sigma_j}$ for $\sigma_j = i$ is excluded). Similarly, there exists a permutation $\tau_1,...,\tau_5$ such that the points $p_{\tau_1j},...,p_{\tau_5j}$ occur in that order on each line segment $M_j$ (excluding $p_{\tau_ij}$ for $\tau_i=j$).
First assume that $\{\sigma_1,\sigma_5\} \cap \{\new{\tau_1,\tau_5}\} = \emptyset$ and let $l \in \{1,...,5\} \setminus \{\sigma_1,\sigma_5,\tau_1,\tau_5\}$.  By the previous construction, the following betweenness relations hold:
\begin{itemize}
\item $p_{\sigma_1l}$ is between  $p_{\sigma_1 \tau_1}$ and $p_{\sigma_1 \tau_5}$. 
\item $p_{\sigma_5l}$ is between  $p_{\sigma_5 \tau_1}$ and $p_{\sigma_5 \tau_5}$. 
\item $p_{l \tau_1}$ is between  $p_{\sigma_1 \tau_1}$ and $p_{\sigma_5 \tau_1}$.
\item $p_{l \tau_5}$ is between  $p_{\sigma_1 \tau_5}$ and $p_{\sigma_5 \tau_5}$. 
\end{itemize}
It follows that $L_l$ (which contains $p_{l \tau_1}$ and $p_{l \tau_5}$) intersects with $M_l$ (which contains $p_{\sigma_1l}$ and $p_{\sigma_5l}$), and thus that $\ec{\mathcal{S}(z_l)}{\mathcal{S}(w_l)}$, a contradiction since $\Theta$ contains the constraint $\dc{z_l}{w_l}$.
Now suppose $\tau_1 \in \{\sigma_1,\sigma_2\}$ but $\{\tau_2,\tau_5\} \cap \{\sigma_1,\sigma_2\} = \emptyset$, then we can choose $l \in \{1,...,5\} \setminus \{\sigma_1,\sigma_5,\tau_2,\tau_5\}$ and derive that $\ec{\mathcal{S}(z_l)}{\mathcal{S}(w_l)}$ by replacing $\tau_1$ by $\tau_2$ in the preceding argument.  In general, we can always find $\tau,\tau' \in \{1,...,5\}$ such that $\{\sigma_1,\sigma_5\} \cap \{\tau,\tau'\}=\emptyset$ and such that $l$ is (strictly) between $\tau$ and $\tau'$.
\end{proof}

\subsubsection{Upper bounds}

\new{The main aim of this section is to prove the following result.}
\begin{proposition}\label{propECDCTPPd}
Every consistent atomic RCC8 network over $\{\bec,\bdc, \btpp, \btppi\}$ has a convex solution in $\mathbb{R}^4$.
\end{proposition}
\new{A similar result for networks over $\{\bec,\bdc, \btpp, \btppi, \bntpp,\allowbreak \bntppi\}$ will then follow easily.}

\new{Before we present the details of the proof, we briefly sketch its intuition. We will assume that the variables in $V$ can be organized in a binary tree, such that the descendants of a region $x$ are those regions that are contained in $x$. It is always possible to ensure that such a binary tree exists by introducing additional regions in a way which does not affect the consistency of $\Theta$.}

\new{Then we will associate each region $x$ with a point $M(t_x)$ on the four-dimensional moment curve. Let $x$ and $y$ be siblings in the tree, and let $X_x$ and $X_y$ be the cells of the Voronoi diagram that is induced by $\{M(t_x),M(t_y)\}$. Then, as we will show in Lemma \ref{lemmaOrderingPoints}, we can choose the values $t_a$ such that $M(t_a) \in X_x$ for any descendant $a$ of $x$ (for any choice of $x$).}

\new{For each $x$, we can then consider the set $Y_x = X_x \cap \{X_u \,|\, \textit{$u$ is an ancestor of $x$}\}$. For each region $a$ and each region $b$ such that $a$ is not a descendant of $b$ and vice versa, it is clear that $\dr{Y_a}{Y_b}$. We will show in Lemmas \ref{lemmaProof4DDescendantQab} and \ref{lemmaProof4DAncestorQab} that for each such regions $a$ and $b$, we can find a point $q_{ab}$ which belongs to $Y_a\cap Y_b$, showing that $\ec{Y_a}{Y_b}$. The proof of these lemmas crucially relies on a property of the moment curve which we show in Lemma \ref{lemmaDerivative2}, being a variant of the well-known property that for any four points on the four-dimensional moment curve there exists a hyperplane that intersects the moment curve at exactly these points, and moreover crosses the moment curve at these points.}

\new{The points $q_{ab}$ will allow us to construct a convex solution of $\Theta$ by realizing each leaf node $a$ as the convex hull of a small sphere centered around $M(t_a)$ and the points $\{q_{ab} \,|\, \Theta \models \ec{a}{b}\}$. The realization of a non-leaf node is  defined as the convex hull of the union of the realizations of its children, which trivially guarantees that all required $\btpp$ relations are satisfied. A key result is shown in Lemma \ref{lemmaSeparateDCregions}, which guarantees that these convex hulls cannot introduce spurious $\bec$ relations. }

\new{The remainder of this section is organized in three parts. First we show two properties of the four-dimensional moment curve which we will rely on in the proof. Then we discuss how the points $M(t_a)$ and $q_{ab}$ are chosen, and we show that they have the required properties. Finally, we show how we can use these points to define a convex solution of $\Theta$.}


\paragraph{Preliminaries}

The proof of Proposition \ref{propECDCTPPNTPPd} below relies on a number of properties about the moment curve in $\mathbb{R}^4$ \new{(see Section \ref{secGeometry})}.  We will also use the following lemma.
\begin{lemma}\label{lemmaDerivative2}
Let $m_1 < t_1 < t_2 < m_2$ and consider the hyperplane $H = \{\new{\mathbf{x}} \,|\, \mathbf{h} \cdot \new{\mathbf{x}} = \theta\}$.  It is possible to choose $\mathbf{h}$ and $\theta$ such that $H$ intersects the four-dimensional moment curve at $M(t_1)$ and $M(t_2)$ and such that the function $f:\mathbb{R}\rightarrow \mathbb{R}$ defined by $f(t) = \mathbf{h} \cdot M(t) - \theta$ reaches an extremum at $m_1$ and $m_2$.  
\end{lemma}
\begin{proof}
The function $f$ takes the following form ($a,b,c,d \in \mathbb{R}$):
\begin{align*}
f(t) = a\cdot t + b\cdot t^2 + c\cdot t^3 +  d\cdot t^4 - \theta
\end{align*}
Note in particular that $f$ is a polynomial in $t$ of degree 4.  The required conditions are that $f(t_1)=f(t_2)=f'(m_1)=f'(m_2)=0$. To see why a suitable polynomial $f$ exists, consider the family $g_{(\delta_1,\delta_2)}$ of quartic polynomials defined as ($\delta_1,\delta_2\geq 0$):
$$
g_{(\delta_1,\delta_2)}(t) = (t- (m_1-\delta_1))\cdot(t-t_1)\cdot(t-t_2)\cdot (t-(m_2+\delta_2))
$$
Clearly $g'_{(\delta_1,\delta_2)}$ has exactly one root $c^-_{(\delta_1,\delta_2)}$ between $m_1-\delta_1$ and $t_1$ and exactly one root $c^+_{(\delta_1,\delta_2)}$ between $t_2$ and $m_2+\delta_2$. From the root dragging theorem \cite{anderson1993}, we know that increasing the value of $\delta_1$ will continuously decrease the values of $c^-_{(\delta_1,\delta_2)}$ and $c^+_{(\delta_1,\delta_2)}$.  From the polynomial root motion theorem \cite{Frayer:2010-08-01T00:00:00:0002-9890:641}, we moreover know that $c^-_{(\delta_1,\delta_2)}$ will decrease faster than $c^+_{(\delta_1,\delta_2)}$. Similarly, when increasing the value of $\delta_2$,  $c^+_{(\delta_1,\delta_2)}$ will increase faster than $c^-_{(\delta_1,\delta_2)}$. Since the values of the $c^-_{(\delta_1,\delta_2)}$ and $c^+_{(\delta_1,\delta_2)}$ depend continuously on $\theta_1$ and $\theta_2$, and $c^-_{0}\geq m_1$ and $c^+_{0}\leq m_2$, it follows that values of $\theta_1$ and $\theta_2$ must exist such that $c^-_{(\delta_1,\delta_2)}=m_1$ and $c^+_{(\delta_1,\delta_2)}=m_2$.

\end{proof}
\noindent Note that if $\|\mathbf{h}\|=1$, $f(t)$ is the signed distance between $M(t)$ and $H$, i.e.\ $| f(t) |= d(M(t),H)$. Also note that in the case of the previous lemma, $H$ will also intersect the moment curve at some points $M(t_0)$ and $M(t_3)$ such that $t_0 < m_1 < t_1 < t_2 < m_2 < t_3$ and $f$ will reach another extremum between $t_1$ and $t_2$.
We will also use the following lemma.
\begin{lemma}\label{lemmaHyperplane1}
Let $P$ and $Q$ be sets of points in $\mathbb{R}^n$ which are separated by the hyperplane $H$, with $Q\cap H = \emptyset$.  Let $r\in H$  and let $\mathbf{h} \neq \mathbf{0}$ be a vector which is orthogonal to $H$ such that $r+ \mathbf{h}$ is in the same half-space induced by $H$ as $P$. For sufficiently large $\lambda >0$ it holds that 
$$
\max_{p\in P}d(p, r+\lambda\cdot \mathbf{h}) < \min_{q\in Q}d(q, r+\lambda\cdot \mathbf{h})
$$
\end{lemma}
\begin{proof}
Without loss of generality, we can assume that $\|\mathbf{h}\| = 1$. Let $p\in P$ and $q\in Q$.  Let $p_0,q_0 \in H$, $\lambda_p \geq 0$ and $\lambda_q>0$ be such that $p = p_0 + \lambda_p \cdot \mathbf{h}$ and $q = q_0 - \lambda_q \cdot \mathbf{h}$. It holds that
\begin{align*} 
d(q, r+\lambda\cdot \mathbf{h})^2 - d(p, r+\lambda\cdot \mathbf{h})^2
 &= d(q_0,r)^2 + (\lambda +\lambda_q)^2 - d(p_0,r)^2 - (\lambda - \lambda_p)^2\\
  &= d(q_0,r)^2 +  \lambda_q^2 - d(p_0,r)^2 - \lambda_p^2 +  2\lambda(\lambda_p + \lambda_q)
\end{align*}
Since $\lambda_p + \lambda_q>0$ it follows that by taking $\lambda$ sufficiently large the latter expression can always be made positive, in which case we have:
$$
d(p, r+\lambda\cdot \mathbf{h})^2 < d(q, r+\lambda\cdot \mathbf{h})^2
$$
\end{proof}
\noindent Note that in the lemma above we do not require $P\cap H = \emptyset$.

\begin{corollary}\label{corHyperplane1}
Let $H$ be a hyperplane, $r\in H$, $\mathbf{h}\neq \mathbf{0}$ a vector which is orthogonal to $H$. Let $p$ and $q$ be two points which are both in the opposite half-space induced by $H$ as $r+\mathbf{h}$.  If $d(p,H) < d(q,H)$, it holds for sufficiently large $\lambda>0$ that $d(p,r+\lambda\cdot \mathbf{h}) < d(q,r+\lambda\cdot \mathbf{h})$. 
\end{corollary}
\begin{proof}
\new{Indeed, we can consider a hyperplane $H'$ which is parallel to $H$ and which separates $p$ and $q$. The result then easily follows from the previous lemma by taking $H:=H'$, $P=\{p\}$ and $Q=\{q\}$.}
\end{proof}

\begin{corollary}\label{corHyperplane2}
Let $H$ be a hyperplane, $r\in H$, $\mathbf{h}\neq \mathbf{0}$ a vector which is orthogonal to $H$. Let $p$ and $q$ be two points which are both in the same half-space induced by $H$ as $r+\mathbf{h}$. If $d(p,H) > d(q,H)$, it holds for sufficiently large $\lambda>0$ that $d(p,r+\lambda\cdot \mathbf{h}) < d(q,r+\lambda\cdot \mathbf{h})$.
\end{corollary}

\paragraph{Associating variables with points on the moment curve}
 Let $\Theta$ be a network over $\{\bec,\bdc, \btpp, \btppi\}$.  The set of variables $V$ can then be organized in a set of \new{ordered} trees $\mathcal{T}_1,...,\mathcal{T}_k$, such that $a$ is a descendant of $b$ in a given tree iff $\Theta \models\{\tpp{a}{b}\}$. For the ease of presentation, we will consider a single \new{ordered} tree $\mathcal{T}$ (whose root does not correspond to any region) which has the trees $\mathcal{T}_1,...,\mathcal{T}_k$ as its subtrees directly below the root.  For example, Figure \ref{figTreeTraversal} shows the tree corresponding to a network $\Theta$ containing the following constraints:
\begin{align*}
&\tpp{v_1}{v_5}  && \tpp{v_3}{v_2} && \tpp{v_4}{v_2} && \tpp{v_2}{v_5} && \tpp{v_3}{v_5} && \tpp{v_4}{v_5}\\
&\tpp{v_7}{v_9} && \tpp{v_8}{v_9} && \tpp{v_9}{v_6} && \tpp{v_{10}}{v_6} && \tpp{v_7}{v_6} && \tpp{v_8}{v_6}
\end{align*}
as well as $\tpp{v_i}{v_{11}}$ for all $i\in\{1,...,10\}$ and either the constraint $\ec{v_i}{v_j}$ or $\dc{v_i}{v_j}$ for every pair of regions $v_i,v_j$ not mentioned.
Without loss of generality, we can make the following assumptions.
\begin{assumption}\label{assumptionTwoChildren}
Every non-leaf node in this tree has exactly two children. 
\end{assumption}
\begin{assumption}\label{assumptionWitnessedEC}
If $\Theta\models \ec{a}{b}$ and $a$ is not a leaf node, then there is a leaf node $c$ such that $\Theta\models \{\tpp{c}{a}, \ec{c}{b}\}$.
\end{assumption}
\begin{assumption}\label{assumptionLeafEC}
\new{If $\Theta\models \ec{a}{b}$ and $a$ is a leaf node, then there is a non-leaf node $c$ such that $\Theta\models \{\tpp{a}{c}, \ec{c}{b}\}$.}
\end{assumption}
\begin{assumption}\label{assumptionAllECregion}
\new{There are regions $z_1,z_2,z_3$ in $V$ such that $z_1$ and $z_2$ are leaf nodes, $\Theta\models \{\tpp{z_1}{z_3},\tpp{z_2}{z_3}\}$ and
for every other region $a$, it holds that $\Theta\models \{\ec{z_1}{a},\ec{z_2}{a},\ec{z_3}{a}\}$.}
\end{assumption}
\new{Indeed, we can always make Assumption \ref{assumptionTwoChildren} satisfied by adding extra regions (e.g.\ if $a$ has three children $x,y,z$, we could add a new region $u$ together with the constraints $\tpp{y}{u}$, $\tpp{z}{u}$ and $\tpp{u}{a}$).  Assumption \ref{assumptionWitnessedEC} can be made satisfied by introducing two new regions $c_1$ and $c_2$ (to ensure that Assumption \ref{assumptionTwoChildren} remains satisfied) and adding the constraints $\tpp{c_1}{a}$, $\tpp{c_2}{a}$, $\dc{c_1}{c_2}$ and $\ec{c_1}{b}$ to $\Theta$. Similarly, we can always satisfy Assumption \ref{assumptionLeafEC}, by introducing new regions $c$ and $a'$ and adding the constraints $\tpp{a}{c}$, $\tpp{a'}{c}$, $\dc{a}{a'$} and $\ec{c}{b}$ as well as $\tpp{c}{v}$ for every region $v\in V$ such that $\Theta\models \tpp{a}{v}$. Note that Assumption \ref{assumptionLeafEC} implies that when $a$ and $b$ are siblings and leaf nodes, it follows that $\Theta \models \dc{a}{b}$.
Finally, Assumption \ref{assumptionAllECregion} can be satisfied by introducing new regions $z_1,z_2,z_3$ and adding to $\Theta$ the constraints $\tpp{z_1}{z_3}$, $\tpp{z_2}{z_3}$, $\dc{z_1}{z_2}$, as well as $\ec{z_1}{a}$, $\ec{z_2}{a}$ and $\ec{z_3}{a}$ for every other region $a$ in $V$. It is clear that adding these latter relations cannot introduce any inconsistencies, as $\Theta$ does not contain any relations of the form $\bntpp$ or $\bntppi$ and the composition $R\circ \bec$ includes $\bec$ for any other RCC8 base relation $R$ (see Table \ref{tableRCC8composition}).}

\begin{figure}
\centering
\includegraphics[width=250pt]{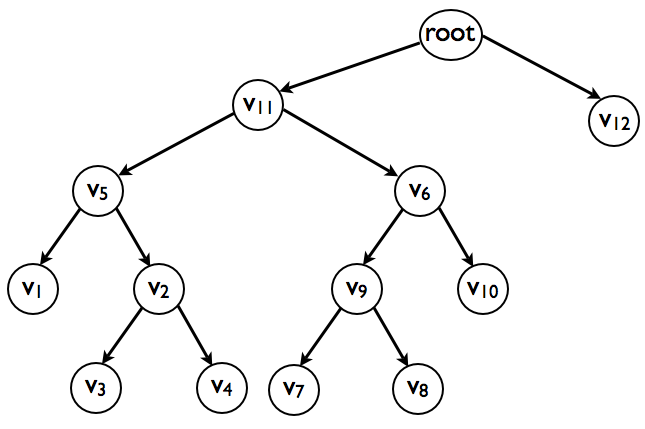}
\caption{Traversal of the tree $\mathcal{T}$.\label{figTreeTraversal}}
\end{figure}

Each variable $v$ in $V$ will be associated with a point $M(t_v)$ on the four-dimensional moment curve (with $t_v> 0$).  
\new{To choose the values $t_v$ we use the following recursive procedure \textbf{Assign} starting with $x=\textit{root}$ and arbitrary values for $\lambda_1$ and $\lambda_2$ satisfying $0 < \lambda_1 < \lambda_2$:}
\new{
\begin{quote}
\textbf{Assign}($x$,$\lambda_1$,$\lambda_2$)
\begin{itemize}
\item If $x$ is not a leaf node:
\begin{itemize}
\item Let $a$ and $b$ be the two children of $x$.
\item choose values $\delta_{x}$ and $\varepsilon_{x}$ such that $\lambda_1< \lambda_1 + \delta_{x} < \lambda_2 -  \varepsilon_{x} < \lambda_2$; as will be explained further, the values $\delta_{x}$ and $\varepsilon_{x}$ will need to be chosen sufficiently small in some sense.
\item choose values $t_{a}$ and $t_{b}$ such that $\lambda_1 < t_{a} < t_{b}<\lambda_1 + \delta_{x}$.
\item \textbf{Assign}($a$,$\lambda_1$,$t_{a}$)
\item \textbf{Assign}($b$,$v -\varepsilon_{x}$,$\lambda_2$)
\end{itemize}
\end{itemize}
\end{quote}}
\noindent \new{The variables in the tree in Figure \ref{figTreeTraversal} are labelled such that \new{they reflect the order in which the corresponding points occur on the moment curve, i.e.\ } the following condition is satisfied:}
$$
t_{v_1} < t_{v_2} < ... < t_{v_m}
$$
\new{Below, we will impose particular constraints on the values $\delta_x$ and $\varepsilon_x$. For now, the main observation is that}
 for every left child $x$ and its sibling $y$, it holds that the points corresponding to regions which are contained in $x$ occur before $M(t_x)$ on the moment curve and the points corresponding to regions which are contained in $y$ occur after $M(t_y)$ on the moment curve. In other words, we have the following lemma.

\begin{lemma}\label{lemmaOrderingPoints}
Let $x$ and $y$ be siblings in $\mathcal{T}$, with $x$ being left of $y$. Let $a$ be a descendant of $x$ and let $b$ be a descendant of $y$. It holds that $t_a < t_x < t_y < t_b$.
\end{lemma}
\begin{corollary}\label{corOrderingPoints}
\new{Let $x$ and $y$ be siblings in $\mathcal{T}$, with $x$ being left of $y$. Let $a$ be a descendant of $x$ and let $b$ be a descendant of $y$.} It holds that \new{$d(M(t_a),M(t_x)) < d(M(t_a),M(t_y))$ and $d(M(t_b),M(t_y)) < d(M(t_b),M(t_x))$}.
\end{corollary}


Let $x$ be an ancestor of $a$ \new{($a\neq x$)} and $y$  an ancestor of $b$ \new{($b\neq y$)} such that $x$ and $y$ are siblings and $t_a < t_x < t_y < t_b$.  Let $H_{ab} = \{\new{\mathbf{w}} \,|\, \mathbf{h_{ab}}\cdot \new{\mathbf{w}} = \theta_{ab}, \new{\mathbf{w}\in \mathbb{R}^4}\}$, \new{for given $\mathbf{h_{ab}}\in \mathbb{R}^4$ and $\theta_{ab}\in \mathbb{R}$}, be a hyperplane intersecting the moment curve at $M(t_x)$ and $M(t_y)$ and such that the function $f_{ab}:\mathbb{R}\rightarrow \mathbb{R}$ defined by $f_{ab}(t) = \mathbf{h_{ab}} \cdot M(t) - \theta_{ab}$ reaches an extremum at $t_a$ and $t_b$. Note that the hyperplane $H_{ab}$ is guaranteed to exist by Lemma \ref{lemmaDerivative2}.  \new{Let $\lambda^*>0$}. We define:
\begin{align}
q_{ab} =  p_{xy} + \new{\lambda^*} \cdot \mathbf{h_{ab}}
\end{align}
where
\begin{align}\label{defPXY}
p_{xy} = \frac{M(t_x) + M(t_y)}{2}
\end{align}
Without loss of generality, we can assume that $\mathbf{h_{ab}}$ is such that $M(t_a)$ and $M(t_b)$ are at the same side of $H_{ab}$ as $q_{ab}$. 

\begin{lemma}\label{lemmaProof4DDescendantQab}
\new{Let $x$ and $y$ be siblings in $\mathcal{T}$, with $x$ being left of $y$. Let $a$ be a descendant of $x$ ($a \neq x$) and let $b$ be a descendant of $y$ ($b \neq y$).}
Let $u$ be a descendant of $x$ ($u\neq a$, $u\neq x$). \new{If $\lambda^*$ is sufficiently large}, it holds that
\begin{align}\label{lemmaProof4DDescendantQabA}
d(q_{ab},M(t_a)) < d(q_{ab},M(t_u))
\end{align}
Similarly, if $v$ is a descendant of $y$ ($v\neq b$, $v\neq y$), then for a sufficiently large \new{$\lambda^*$} we have
\begin{align}\label{lemmaProof4DDescendantQabB}
d(q_{ab},M(t_b)) < d(q_{ab},M(t_v))
\end{align}
\end{lemma}
\begin{proof}
Because of the way in which $H_{ab}$ was chosen, we have that $f_{ab}$ reaches its only extremum left of $t_x$ at $t_a$. Since $t_v < t_x$, one of the following cases must hold:
\begin{enumerate}
\item $t_v$ and $t_a$ are located at the same side of $H_{ab}$ (which is the side at which $q_{ab}$ is located), and $d(t_v,H_{ab}) < d(t_a,H_{ab})$; or
\item $t_v$ is located at the opposite side of $H$ as $t_a$ and $q_{ab}$.
\end{enumerate}
In the former case, \eqref{lemmaProof4DDescendantQabA} follows from Corollary \ref{corHyperplane2}, while in the latter case, \eqref{lemmaProof4DDescendantQabA} follows from Lemma \ref{lemmaHyperplane1}.
In entirely the same way, we can also show \eqref{lemmaProof4DDescendantQabB}.
\end{proof}

\begin{lemma}\label{lemmaProof4DAncestorQab}
\new{Let $x$ and $y$ be siblings in $\mathcal{T}$, with $x$ being left of $y$. Let $a$ be a descendant of $x$ ($a \neq x$) and let $b$ be a descendant of $y$ ($b \neq y$).}
Let $u\neq a$ be an ancestor of $a$ or $b$ and let $v$ be the sibling of $u$. If \new{$\lambda^*$} is chosen to be sufficiently large, it holds that
\begin{align}\label{lemmaProof4DAncestorQabA}
d(q_{ab},M(t_u)) \leq d(q_{ab},M(t_v))
\end{align}
\end{lemma}
\begin{proof}
Let $x$ and $y$ be defined as before. We consider the following cases:
\begin{itemize}
\item If $t_u < t_a$, then $a$, $b$, $x$ and $y$ are all descendants of $u$. It follows that $t_v < t_u < t_a$. Then  \eqref{lemmaProof4DAncestorQabA} easily follows from the fact that $t_a$ is the only maximum of $f_{ab}$ left of $t_x$ (as in the proof of Lemma \ref{lemmaProof4DDescendantQab}).
\item If $t_a < t_u < t_x$, then by construction, $u$ is a descendant of $x$ and an ancestor of $a$. It then holds that $t_a < t_u < t_v < t_x$. Again we obtain \eqref{lemmaProof4DAncestorQabA}from the fact that $t_a$ is the only maximum of $f_{ab}$ left of $t_x$.
\item If $t_u=t_x$ then $t_v=t_y$ and it follows that $d(q_{ab},M(t_u)) = d(q_{ab},M(t_v))$ by construction of $q_{ab}$.
\item The case where $t_x < t_u<t_y$ is not possible.
\item If $t_u=t_y$ then $t_v=t_x$ and it again follows that $d(q_{ab},M(t_u)) = d(q_{ab},M(t_v))$.
\item If $t_y < t_u < t_b$ then $u$ is a descendant of $y$ and an ancestor of $b$, and we have $t_y < t_v < t_u < t_b$, and \eqref{lemmaProof4DAncestorQabA} follows from the fact that $t_b$ is the only maximum of $f_{ab}$ right of $t_y$.
\item If $t_b < t_u$ then $a$, $b$, $x$ and $y$ are all descendants of $u$ and it follows that $t_b < t_u < t_v$, from which we can show \eqref{lemmaProof4DAncestorQabA} using the fact that $t_b$ is the only maximum of $f_{ab}$ right of $t_y$.
\end{itemize}
\end{proof}

\noindent \new{In the following, we assume that $\lambda^*$ is chosen sufficiently large such that \eqref{lemmaProof4DDescendantQabA}--\eqref{lemmaProof4DAncestorQabA} hold for every $a$ and $b$.}

\new{We now turn to the question how the values $\delta_x$ and $\varepsilon_x$ in the procedure \textbf{Assign} need to be chosen.}
\new{Intuitively, if $x$ and $y$ are siblings in the tree $\mathcal{T}$ such that $x$ is left of $y$, we want to ensure that the values $t_v$ corresponding to the descendants of $y$ are sufficiently close to each other, and that they are sufficiently far from the values $t_v$ corresponding to the descendants of $x$. In particular, the}
hyperplane $H_{ab}$ intersects the moment curve at points $M(s_{ab})$, $M(t_x)$, $M(t_y)$ and $M(r_{ab})$, where $s_{ab} < t_a$ and $r_{ab} > t_b$. \new{Let $z$ be the parent of $x$ and $y$ in the tree $\mathcal{T}$.} \new{Note that $s_{ab}$ and $r_{ab}$ depend on the choice of $t_a$ and $t_b$: the smaller $t_a$ and $t_b$, the smaller the values of $s_{ab}$ and $r_{ab}$ will be.  However, from the polynomial root motion theorem (\cite{Frayer:2010-08-01T00:00:00:0002-9890:641}, see also the proof of Lemma \ref{lemmaDerivative2}), we know that moving $t_a$ will have a greater impact on $s_{ab}$ than on $r_{ab}$. In particular, if $r_{ab}$ is sufficiently far from $s_{ab}$, $t_x$ and $t_y$, the impact on $r_{ab}$ of moving the value of $t_a$ will be negligible. By choosing $\delta_z$ sufficiently small, we can ensure that the points $r_{a_1b},...,r_{a_kb}$, where  $a_1,...,a_k$ are the descendants of $x$, are arbitrarily close to each other. } This means that the vectors $\mathbf{h_{a_1b}},...,\mathbf{h_{a_kb}}$ are all approximately located in (i.e.\ arbitrarily close to) the same two-dimensional plane $G$. As a result, it is clear that the following assumption will be satisfied.
\begin{assumption}\label{assumption4}
Let $a_1$, $a_2$ and $c$ be descendants of $x$ such that $t_c<t_{a_1}$ or $t_c>t_{a_2}$. Let $d$ be a descendant of $y$ and let $\mathbf{h^*}$ be a vector bisecting $\mathbf{h_{a_1d}}$ and $\mathbf{h_{a_2d}}$. It holds that
\begin{align*}
\new{\cos(\mathbf{h^*},\mathbf{h_{cd}}) < \min(\cos(\mathbf{h^*},\mathbf{h_{a_1d}}), \cos(\mathbf{h^*},\mathbf{h_{a_2d}}))}
\end{align*}
\end{assumption}
If moreover $\varepsilon_z$ is sufficiently small, the angles $(\mathbf{h_{ab}},\mathbf{h_{ab'}})$ will be negligible compared to $(\mathbf{h_{ab}},\mathbf{h_{a'b}})$, where $a$ and $a'$ are descendants of $x$ and $b$ and $b'$ are descendants of $y$. This means that the following stronger assumption will also be satisfied.
\begin{assumption}\label{assumptionOrderingNegligible}
Let $a_1$, $a_2$ and $c$ be descendants of $x$ such that $t_c<t_{a_1}$ or $t_c>t_{a_2}$. Let $d$ be a descendant of $y$ and let $\mathbf{h^*}$ be a vector bisecting $\mathbf{h_{a_1d}}$ and $\mathbf{h_{a_2d}}$. It holds that
\begin{align*}
\new{\cos(\mathbf{h^*},\mathbf{h_{cd}}) < \min(\cos(\mathbf{h^*},\mathbf{h_{a_1d'}}), \cos(\mathbf{h^*},\mathbf{h_{a_2d'}}))}
\end{align*}
for every descendant $d'$ of $y$.
\end{assumption}

\noindent This latter assumption allows us to prove the following lemma.

\begin{lemma}\label{lemmaSeparateDCregions}
Let $x$ and $y$ be siblings such that $t_x < t_y$ and let $a$ be a descendant of $x$ (\new{or $a=x$}).  Consider the following regions:
\begin{align*}
A^+ = \ch \{q_{cd} \,|\, &\textit{$c$ is a descendant of $a$ (or $c=a$)}, \new{c\neq x},\\
& \textit{ $d$ is a descendant of $y$ (\new{$d\neq y$})} \}\\
A^- = \ch \{q_{cd} \,|\, &\textit{$c$ is a descendant of $x$ \new{($c\neq x$)} but not a descendant of $a$ ($c\neq a$)},\\
& \textit{$d$ is a descendant of $y$ (\new{$d\neq y$})}\}
\end{align*}
It holds that $A^+ \cap A^- = \emptyset$.
\end{lemma}
\begin{proof}
There exist values $t_a^-$ and $t_a^+$ such that for every region $c$, it holds that $c$ is a descendant of $a$ iff $t_c \in [t_a^-,t_a^+]$. Fix a descendant $d^*$ of $y$. Let $\mathbf{h^*}$ be the unique vector of unit length bisecting $\mathbf{h_{a^-d^*}}$ and $\mathbf{h_{a^+d^*}}$ and let $q^* = p_{xy} + \lambda^*\cdot \mathbf{h^*}$.  Let $H^*$ be the hyperplane which is orthogonal to $\mathbf{h^*}$ and which contains $p_{xy}$ (where $p_{xy}$ is defined as in \eqref{defPXY}). Furthermore, we define $H^*_{\theta}$ as the hyperplane which is parallel to and at distance $\theta$ from $H^*$, where $H^*_{\theta}$ for $\theta>0$ is located in the same half-space induced by $H^*$ as $q_{xy}$.  Let us in particular consider the following choice for $\theta$:
\begin{align}\label{eqMinDistProofSeparationDC}
\theta = \min_{c,d}   d(H^*, q_{cd}) =  \min_{c,d} \lambda^* \cdot \cos(\mathbf{h^*},\mathbf{h_{cd}})
\end{align}
where the minimum is taken over all descendants $c$ of $a$ and all descendants $d$ of $y$.
We will show that $H^*_{\theta}$ is a hyperplane separating $A^-$ from $A^+$. It is sufficient to show that for every $c$ which is not a descendant of $a$ and every descendant $d$ of $y$, it holds that
\begin{align}\label{eqMinDistProofSeparationDC2}
d(H^*, q_{cd}) = \lambda^* \cdot \cos(\mathbf{h^*},\mathbf{h_{cd}}) < \theta
\end{align}
We have that either $t_{c} < t_a^-$ or $t_{c} > t_a^-$, which by Assumption \ref{assumptionOrderingNegligible} means that for all descendants $d$ and $d'$ of $y$
$$
\cos(\mathbf{h^*},\mathbf{h_{cd}}) < \min(\cos(\mathbf{h^*},\mathbf{h_{a^-d'}}), \cos(\mathbf{h^*},\mathbf{h_{a^+d'}}))
$$
The minimum in \eqref{eqMinDistProofSeparationDC} is attained for either $c=a^-$ or $c=a^+$. It follows that \eqref{eqMinDistProofSeparationDC2} holds, and thus that $H^*_{\theta}$ is a hyperplane separating $A^-$ and $A^+$.
\end{proof}

\paragraph{Associating variables with convex regions}

For each region $x$ and its sibling $y$ in the tree, we can consider the Voronoi diagram induced by the corresponding points on the moment curve $M(t_x)$ and $M(t_y)$. Let us denote the corresponding cells as $X_x$ and $X_y$.  For each node $a$ we define:
$$
Y_a = X_a \cap \bigcap \{X_x \,|\, \textit{$x$ is an ancestor of $a$}\}
$$
\new{If $a$ and $b$ are leaf nodes such that $\Theta\models \ec{a}{b}$, it follows from Assumption \ref{assumptionLeafEC} that $a$ has an ancestor $x$ ($a\neq x$) and $b$ has an ancestor $y$ ($b\neq y$) such that $x$ and $y$ are siblings in $\mathcal{T}$. We then have that $q_{ab}$ is well-defined.}
\begin{lemma}\label{lemmaQabInVoronoiY}
For all \new{leaf} nodes $a$ and $b$ such that $\Theta\models \ec{a}{b}$, it holds that $q_{ab} \in Y_a \cap Y_b$.
\end{lemma}
\begin{proof}
Let $a'$ be the sibling of $a$ in the tree. From Lemma \ref{lemmaProof4DDescendantQab}, it follows that $d(q_{ab},M(t_a))\leq d(q_{ab},M(t_{a'}))$ which means $q_{ab}\in X_a$.  Now let $u$ be an arbitrary ancestor of $a$ and let $v$ be the sibling of $x$.  By Lemma \ref{lemmaProof4DAncestorQab} we then know that $d(q_{ab},M(t_u))\leq d(q_{ab},M(t_v))$, which implies $q_{ab}\in X_u$. Since this holds for every ancestor of $a$, it follows that $q_{ab}\in Y_a$. In entirely the same way, we find $q_{ab}\in Y_b$.
\end{proof}
\noindent Let $S_a$ be a small hypersphere centered around $M(t_a)$:
$$
S_a = \{ p \,|\, d(p,M(t_a)) \leq \varepsilon\}
$$
for $\varepsilon>0$ a sufficiently small constant. We realize each variable $a$ corresponding to a leaf node in the tree as the convex region $Z_a$ defined as:
$$
Z_a = \ch( S_a \cup \{\new{q_{ab}}\,|\, \Theta \models \ec{a}{b}, \textit{ $b$ is a leaf node} \})
$$
Each region $a$ which is not a a leaf node is realized as
$$
Z_a = \ch(\{Z_b \,|\, \textit{$b$ is a leaf node and a descendant of $a$}\})
$$
From Lemma \ref{lemmaQabInVoronoiY}, it follows that $Z_a \subseteq Y_a$, which implies $\dr{Z_a}{Z_b}$ unless $a$ is a descendant or ancestor of $b$. Moreover, from Assumption \ref{assumptionWitnessedEC}, we know that when $\Theta\models \ec{a}{b}$ it will hold that $\ec{Z_a}{Z_b}$. Moreover, by construction, it is clear that $\p{a}{b}$ when $a$ is a descendant of $b$ ($a\neq b$). From Assumption \ref{assumptionTwoChildren}, it follows that then $\pp{a}{b}$, and because of Assumption \ref{assumptionAllECregion} we moreover have $\tpp{a}{b}$.

It remains to be shown that $\dc{Z_a}{Z_b}$ if $\Theta\models \dc{a}{b}$. Assume $\Theta\models \dc{a}{b}$. \new{Let $x$ be an ancestor of $a$ (or $x=a$) and $y$ be an ancestor of $b$ (or $y=b$) such that $x$ and $y$ are siblings.} \new{First assume that neither of $x$ and $y$ is a leaf node.} Now consider the sets $P_a$ and $P_b$ defined as follows:
\begin{align*}
P_a = \{\new{q_{cd}} \,|\, &\textit{$c$ is a descendant of $a$ (or $c=a$)}, \textit{$d$ is a descendant of $y$},\\
& \textit{$c$ and $d$ are leaf nodes}, \Theta\models \ec{c}{d}\}\\
P_b = \{\new{q_{cd}} \,|\, &\textit{$c$ is a descendant of $x$ }, \textit{$d$ is a descendant of $b$ (or $d=b$)},\\
& \textit{$c$ and $d$ are leaf nodes}, \Theta\models \ec{c}{d}\}
\end{align*}
\new{Let $G_{xy}$ be the hyperplane which separates $X_x$ from $X_y$. Noting that $G_{xy}$ then also separates $Z_a$ from $Z_b$, we have}
 $\ch(P_a) = Z_a \cap \new{G_{xy}}$ and $\ch(P_b) = Z_b \cap \new{G_{xy}}$. \new{To show that $Z_a \cap Z_b = \emptyset$, it thus} suffices to show that $\ch(P_a) \cap \ch(P_b)  = \emptyset$.  Note that if \new{$q_{cd}\in P_b$} it must be the case that $c$ is not a descendant of $a$, since $\new{q_{cd}}\in P_b$ implies $\ec{c}{d}$ and $d\{\btpp,\beq\}b$ and thus $\ec{c}{b}$. \new{Also note that $q_{cd}\in P_a$ or $q_{cd}\in P_b$ means that $c\neq x $ and $d\neq y$ since $x$ and $y$ are not leaf nodes, i.e.\ all elements of $P_a$ and $P_b$ are well-defined. It follows} that $\ch(P_a) \subseteq A^+$ and $\ch(P_b) \subseteq A^-$ in the sense of Lemma \ref{lemmaSeparateDCregions}, from which we can derive that indeed $\ch(P_a) \cap \ch(P_b)  = \emptyset$.  \new{Now assume that $x$ is a leaf node (and thus $x=a$). Then it is clear that $y$ cannot have any descendant $c$ such that $\Theta \models \ec{x}{c}$, as by Assumption \ref{assumptionLeafEC}, this would mean that $x$ has an ancestor $z$ such that $\Theta \models\ec{z}{c}$ and thus $\Theta \models\ec{z}{y}$, which means that $x$ and $y$ would not be siblings. Hence we have that $Z_a \cap G_{xy} = Z_b \cap G_{xy} = \emptyset$ from which we again find $\dc{Z_a}{Z_b}$. Finally, the case where $y$ is a leaf node is entirely similar.} 
 

\new{This completes the proof of Proposition \ref{propECDCTPPd}. Additionally considering $\bntpp$ relations does not introduce any additional difficulties, i.e.\ we also have the following result.}
\begin{proposition}\label{propECDCTPPNTPPd}
Every consistent atomic RCC8 network over $\{\bec,\bdc, \btpp, \btppi, \bntpp, \bntppi\}$ has a convex solution in $\mathbb{R}^4$.
\end{proposition}
\begin{proof}
The proof proceeds by induction on the number of $\bntpp$ or $\bntppi$ relations in $\Theta$, entirely analogously to the proof of Proposition \ref{propRCDCNTPP3d}. 
\end{proof}

\subsection{Fragments requiring an arbitrarily high number of dimensions\label{secFragmentsArbitrary}}

To prove the existence of particular types of RCC5 or RCC8 networks which cannot be realized in $\mathbb{R}^n$, for a fixed $n$, the following well-known result is particularly useful.
\begin{proposition}[Radon's theorem]
For any set $\{p_1,...,p_{n+2}\}$ of points in $\mathbb{R}^n$ there exists a partition $P\cup Q = \{p_1,...,p_{n+2}\}$ such that $\ch(P) \cap \ch(Q) \neq \emptyset$.
\end{proposition}
\noindent Using Radon's theorem we can show the following result.
\begin{proposition}\label{propRCC5arbitrary}
For every $n\geq 1$ there exists a consistent atomic RCC5 network $\Theta$ over $\{\bdr,\bpo,\bpp,\bppi\}$ which has no convex solution in $\mathbb{R}^n$.
\end{proposition}
\begin{proof}
\noindent Let $n\geq 1$ be arbitrary and consider the following set of RCC5 constraints $\Theta$.  For $1 \leq i < j \leq n+2$, add the following constraints to $\Theta$:
\begin{align}\label{eqECPOTPPTPPI-c}
\dr{a_i}{a_j}
\end{align}
For each subset $I \subseteq \{1,...,n+2\}$ we introduce a variable $b_{I}$ and for each $i\in I$ we add the following constraint to $\Theta$:
\begin{align}\label{eqECPOTPPTPPI-e}
\pp{a_i}{b_I}
\end{align}
For each $j\notin I$, we moreover add the following constraint:
\begin{align}
\dr{a_j}{b_I}
\end{align}
For any two (different) subsets $I$ and $J$ of $\{1,...,n+2\}$:
\begin{itemize}
\item If $I\subset J$, we add to $\Theta$ the constraint $\pp{b_I}{b_J}$.
\item Else, if $I\supset J$, we add to $\Theta$ the constraint $\ppi{b_I}{b_J}$.
\item Else, if $I\cap J\neq \emptyset$, we add to $\Theta$ the constraint $\po{b_I}{b_J}$.
\item Else, we add to $\Theta$ the constraint $\dr{b_I}{b_J}$.
\end{itemize}
It is clear that $\Theta$ is a consistent atomic network over $\{\bdr,\bpo,\bpp,\bppi\}$.

Suppose $\Theta$ had a convex solution $\mathcal{S}$ in $\mathbb{R}^n$. Fix an arbitrary point $p_i$ in the interior of each region $\mathcal{S}(a_i)$. It is clear that the points $p_1,...,p_{n+2}$ are distinct, because of \eqref{eqECPOTPPTPPI-c}.  Because of Radon's theorem, there exists a set  $I \subseteq \{1,...,n+2\}$ such that $\cvx(\bigcup_{i\in I} p_i) \cap \cvx(\bigcup_{i\notin I} p_i) \neq \emptyset$. As each $p_i$ is an interior point of $\mathcal{S}(a_i)$, it follows that there exists an $\varepsilon>0$ such that 
$$
\mathcal{S}(b_I) \cap \mathcal{S}(b_{coI}) \supseteq \{p\,|\, d(p,q)\leq \varepsilon, q \in \cvx(\bigcup_{i\in I} p_i) \cap \cvx(\bigcup_{i\notin I} p_i)\}
$$
In particular, this means that $\dim(\mathcal{S}(b_I) \cap \mathcal{S}(b_{coI})) = n$, where $coI = \{1,...,n+2\}\setminus I$. This contradicts the constraint $\dr{b_I}{b_{coI}}$ which is entailed by $\Theta$.
\end{proof}
The proof of Proposition \ref{propRCC5arbitrary} remains valid if we replace the role of $\bpp$ by either $\btpp$ or $\bntpp$ and if we replace $\bdr$ by either $\bec$ or $\bdc$. Thus we obtain the following corollary.
\begin{corollary}
For every $n\geq 1$, there exist consistent atomic RCC8 networks over the following sets of base relations which have no convex solutions in $\mathbb{R}^n$:
\begin{itemize}
\item$\{\bdc,\bpo, \btpp,\btppi\}$
\item$\{\bec,\bpo, \btpp,\btppi\}$
\item$\{\bdc,\bpo, \bntpp,\bntppi\}$
\item$\{\bec,\bpo, \bntpp,\bntppi\}$
\end{itemize}
\end{corollary}

\section{Upper bounds on the number of dimensions for general RCC8 networks}\label{secRegionsDimensions}

Throughout this section, let $\Theta$ be an atomic and consistent RCC8 network which does not contain any occurrences of $\beq$. The aim of this section is to establish an upper bound on the number of dimensions which are needed to guarantee that $\Theta$ can be convexly realized. Our strategy is to start with a restricted network $\Theta{\downarrow}V'$, with $V'\subseteq V$, for which the results from the previous section guarantee a convex solution in $\mathbb{R}$, $\mathbb{R}^2$, $\mathbb{R}^3$ or $\mathbb{R}^4$. Then we incrementally build a convex solution for $\Theta$ from the convex solution of $\Theta{\downarrow}V'$. The main technique we use to extend the initial solution for $\Theta{\downarrow}V'$ is the following proposition, where a subnetwork $\Theta{\downarrow}Z$ ($Z\subseteq V$) is called an \emph{\bo-clique} if $\Theta\models {z_1}\bo{z_2}$ for every $z_1,z_2\in Z$. We say that an \bo-clique  $\Theta{\downarrow}Z$ has a common part in a  solution $\mathcal{S}$ if $\bigcap_{z\in Z}\mathcal{S}(z)$ is a non-empty regularly closed region.

\begin{proposition}\label{propExtendNTPPDC}
Let $\Theta$ be a consistent atomic RCC8 network. Suppose that $\Theta{\downarrow}V'$ has a $k$-dimensional convex solution in which every \bo-clique has a common part, and moreover that $\Theta$ entails the following constraints:
\begin{align*}
&\ntpp{a_1}{a_2} && \ntpp{a_2}{a_3} && ... &&\ntpp{a_{l-1}}{a_r}\\
&\ntpp{b_1}{b_2} && \ntpp{b_2}{b_3} && ... &&\ntpp{b_{l-1}}{b_s}\\
&\dc{a_r}{b_s}
\end{align*}
It holds that $\Theta{\downarrow}(V' \cup \{a_1,...,a_r,b_1,...,b_s\})$ has a $(k+1)$-dimensional convex solution in which every \bo-clique has a common part.
\end{proposition}
\begin{proof}
Let $\mathcal{S}$ be a convex solution of $\Theta{\downarrow}V'$ in $\mathbb{R}^k$; we construct a convex solution $\mathcal{T}$ of $\Theta{\downarrow}(V' \cup \{a_1,...,a_r,b_1,...,b_s\})$ in $\mathbb{R}^{k+1}$. Let $A_{r+1}=B_{s+1}$ be an arbitrary non-empty convex region which strictly contains $\mathcal{S}(v)$ for every $v\in V'$.  Let $A_i$ and $B_j$ for $i\in\{1,...,r\}$ and $j\in\{1,...,s\}$ be defined as:
\begin{align*}
A_ i &= E_{\delta_i}(A_{i+1}) \cap  \bigcap\{\mathcal{S}(v) \,|\, \Theta \models \tpp{a_i}{v}, v\in V'\}\\
&\phantom{= E_{\delta_i}(A_{i+1}) } \cap \bigcap \{E_{\delta_i}(\mathcal{S}(v)) \,|\, \Theta \models \ntpp{a_i}{v}, v\in V'\}\\
B_ j &= E_{\delta'_j}(B_{j+1}) \cap  \bigcap\{\mathcal{S}(v) \,|\, \Theta \models \tpp{b_j}{v}, v\in V'\}\\
&\phantom{= E_{\delta'_j}(B_{j+1}) } \cap \bigcap \{E_{\delta'_j}(\mathcal{S}(v)) \,|\, \Theta \models \ntpp{b_j}{v}, v\in V'\}
\end{align*}
where $E_{\delta}$ is the erosion operator defined by $E_{\delta}(X) = \{p \,|\, \forall q \,.\, (d(p,q)\leq \delta) \Rightarrow (q\in X)\}$.  We will assume that $\delta_i$ is sufficiently small to ensure the following three conditions:
\begin{itemize}
\item $A_i$ is non-empty and $\dim(A_i)=k$.
\item For every $v,w\in V'$ if $\Theta \models \ntpp{v}{w}$ it holds that $\mathcal{S}(v) \subseteq E_{\delta_i}(\mathcal{S}(w))$.
\item For every $v\in V'$ and $i \in \{1,...,n\}$ such that $\mathcal{S}(v)\subseteq i(A_{i+1})$ it holds that $\mathcal{S}(v)\subset E_{\delta_i}(A_{i+1})$.
\end{itemize}
and analogously for $\delta'_j$ and $B_j$. Note that we can assume that $A_i$ is non-empty because every \bo-clique has a common part.
It is clear that we can make these assumptions without loss of generality. 
Note that since $A_ i \subseteq E_{\delta_i}(A_{i+1})$, it holds that $\ntpp{A_i}{A_j}$ for $1\leq i < j \leq r$ and similarly we have $\ntpp{B_i}{B_j}$ for $1\leq i < j \leq s$.

Let $0< a^-_r < a^-_{r-1} < ... < a^-_1 < a^+_1 < a^+_2 < ... < a^+_r$ and $b^-_s < b^-_{s-1} < ... < b^-_1 < b^+_1 < b^+_2 < ... < b^+_s < 0$.
We define $\mathcal{T}$ for $a_i$ and $b_j$ as:
\begin{align*}
\mathcal{T}(a_i) &= \{(x_1,...,x_k,y) \,|\, (x_1,...,x_k) \in A_i, y\in [a^-_i,a^+_i]\}\\
\mathcal{T}(b_j) &= \{(x_1,...,x_k,y) \,|\, (x_1,...,x_k) \in B_j, y\in [b^-_j,b^+_j]\}
\end{align*}
Since $\ntpp{A_i}{A_j}$ and $\ntpp{[a^-_i,a^+_i]}{[a^-_j,a^+_j]}$ for $i<j$ we immediately find $\ntpp{\mathcal{T}(a_i)}{\mathcal{T}(a_j)}$ and similarly we find $\ntpp{\mathcal{T}(b_i)}{\mathcal{T}(b_j)}$. Moreover, since $b^+_s < 0 < a^-_r$ we have $\dc{\mathcal{T}(a_i)}{\mathcal{T}(b_j)}$ for $1\leq i\leq r$ and $1\leq j\leq s$.
For $v \in V'$ we define:
$$
\mathcal{T}(v) =  \{(x_1,...,x_k,y) \,|\, (x_1,...,x_k) \in \mathcal{S}(v), y\in [v^-,v^+]\}
$$
where $v^- < v^+$ are defined as follows.  
\new{Let the level $\Delta(v)$ of a region $v$ be defined as in Section \ref{subsecUpper2D}.}
We define:
\begin{align*}
v^+ &=
\begin{cases}
a^+_r + \Delta(v) \cdot \varepsilon & \text{if $\Theta\models \ntpp{a_r}{v}$}\\
a^+_{i} & \text{if $\Theta\models \tpp{a_i}{v}$}\\
a^+_{i}+ \Delta(v) \cdot \varepsilon & \text{if $\Theta\models \ntpp{a_i}{v}$, $i<r$ and $\Theta\not\models\p{a_{i+1}}{v}$}\\
\frac{a^-_{1}+a^+_{1}}{2} + \Delta(v) \cdot \varepsilon & \text{if $\Theta\models {a_1}\{\bpo,\btppi,\bntppi\}{v}$}\\
a^-_{1} & \text{if $\Theta\models \ec{a_1}{v}$}\\
\frac{a^-_{i}+a^-_{i-1}}{2} + \Delta(v)\cdot \varepsilon & \text{if $\Theta\models \dc{a_{i-1}}{v}$, $i>1$ and $\Theta\models\ov{a_{i}}{v}$}\\
a^-_{i} & \text{if $\Theta\models \dc{a_{i-1}}{v}$, $i>1$ and $\Theta\models\ec{a_{i}}{v}$}\\
\Delta(v)\cdot \varepsilon & \text{if $\Theta\models \dc{a_{r}}{v}$ and $\Theta\not\models\p{v}{b_{s}}$}\\
b^+_j & \text{if $\Theta\models \tpp{v}{b_j}$}\\
\frac{b^+_j + b^+_{j-1}}{2} + \Delta(v)\cdot \varepsilon & \text{if $\Theta\models \ntpp{v}{b_j}$, $j>1$, $\Theta \not\models \p{v}{b_{j-1}}$}\\
\frac{b^+_1 + b^-_1}{2} + \Delta(v)\cdot \varepsilon & \text{if $\Theta\models \ntpp{v}{b_1}$}\\
\end{cases}\\
v^- &=
\begin{cases}
b^-_s - \Delta(v)\cdot \varepsilon & \text{if $\Theta\models \ntpp{b_s}{v}$}\\
b^-_{j} & \text{if $\Theta\models \tpp{b_j}{v}$}\\
b^-_{j}-\Delta(v)\varepsilon & \text{if $\Theta\models \ntpp{b_j}{v}$, $j<s$ and $\Theta\not\models\p{b_{j+1}}{v}$}\\
\frac{b^-_{1} + b^+_{1}}{2}-\Delta(v)\cdot\varepsilon & \text{if $\Theta\models {b_1}\{\bpo,\btppi,\bntppi\}{v}$ }\\
b^+_{1} & \text{if $\Theta\models \ec{b_1}{v}$}\\
\frac{b^+_{j}+b^+_{j-1}}{2} - \Delta(v)\cdot \varepsilon & \text{if $\Theta\models \dc{b_{j-1}}{v}$, $j>1$ and $\Theta\models\ov{b_{j}}{v}$}\\
b^+_{j} & \text{if $\Theta\models \dc{b_{j-1}}{v}$, $j>1$ and $\Theta\models\ec{b_{j}}{v}$}\\
-\Delta(v)\cdot \varepsilon & \text{if $\Theta\models \dc{b_{s}}{v}$ and $\Theta\not\models\p{v}{a_{r}}$}\\
a^-_i & \text{if $\Theta\models \tpp{v}{a_i}$}\\
\frac{a^-_i + a^-_{i-1}}{2} - \Delta(v)\cdot \varepsilon & \text{if $\Theta\models \ntpp{v}{a_i}$, $i>1$, $\Theta \not\models \p{v}{a_{i-1}}$}\\
\frac{a^-_1 + a^+_1}{2} - \Delta(v)\cdot \varepsilon & \text{if $\Theta\models \ntpp{v}{a_1}$}\\
\end{cases}
\end{align*}
where $\varepsilon>0$ is sufficiently small.
First note that $v^- < v^+$. Indeed, if $\Theta\not\models\p{v}{a_r}$ and $\Theta\not\models\p{v}{b_s}$, it holds that $v^- < 0 < v^+$. Now consider $\Theta\models\p{v}{a_r}$ and let $i$ be the smallest $i$ for which $\Theta\models\p{v}{a_i}$, then we have $a_i^- \leq v^- < \frac{a_i^- + a_{i-1}^- }{2} < v^+ < a_{i-1}^-$ if $i>1$ and $a_1^- \leq v^- < \frac{a_1^- + a_1^+ }{2} < v^+ < a_{1}^+$ otherwise. The case where $\Theta\models\p{v}{b_s}$ is entirely similar.

It is straightforward to verify that the RCC8 relations which hold between $\mathcal{T}(v)$, on the one hand, and $\mathcal{T}(a_1),...,\mathcal{T}(a_r),\mathcal{T}(b_1),...,\mathcal{T}(b_s)$, on the other, are those which are imposed by $\Theta$.  Finally, we need to show that $\mathcal{T}$ is also a solution of $\Theta{\downarrow}V'$, given that $\mathcal{S}$ is a solution of $\Theta{\downarrow}V'$. 
Let $u_1,u_2\in V'$. 
\begin{itemize}
\item The cases where $\Theta\models \dc{u_1}{u_2}$ or $\Theta\models \tpp{u_1}{u_2}$ are trivial. 
\item Assume that $\Theta\models \ntpp{u_1}{u_2}$. By construction it is clear that $\pp{\mathcal{T}(u_1)}{\mathcal{T}(u_2)}$. The fact that $\ntpp{\mathcal{T}(u_1)}{\mathcal{T}(u_2)}$ follows easily from the observation that $\Delta(u_1) < \Delta(u_2)$.
\item Assume that $\Theta\models \ec{u_1}{u_2}$ and let $(x_1,...,x_k) \in \mathcal{S}(u_1)\cap \mathcal{S}(u_2)$. It is clear that $\dr{\mathcal{T}(u_1)}{\mathcal{T}(u_2)}$; we show that $\mathcal{T}(u_1) \cap \mathcal{T}(u_2) \neq \emptyset$:
\begin{itemize}
\item If $\Theta\not\models \p{u_1}{a_r}$, $\Theta\not\models \p{u_1}{b_s}$, $\Theta\not\models \p{u_2}{a_r}$ and $\Theta\not\models \p{u_2}{b_s}$, it holds that $(x_1,...,x_k,0)\in \mathcal{T}(u_1) \cap \mathcal{T}(u_2)$.
\item Suppose that $\Theta\models \tpp{u_1}{a_i}$ and $\Theta\not\models \p{u_1}{a_{i-1}}$ (or $i=1$).  Note that then $u_1^-=a_i^-$ and $u_1^+> \frac{a_i^- + a_{i-1}^-}{2}$ if $i>1$ and $u_1^+> \frac{a_1^- + a_{1}^+}{2}$ otherwise.  Since $\Theta\models \ec{u_1}{u_2}$ it must hold that $\Theta\not\models\dc{u_2}{a_i}$. If $\Theta\models \ec{u_2}{a_i}$, it holds that $u_2^+=a_i^-$ and thus $(x_1,...,x_k,a_i^-) \in \mathcal{T}(u_1) \cap \mathcal{T}(u_2)$. Similarly, if $\Theta\models \ov{u_2}{a_i}$, it holds that $(x_1,...,x_k,\frac{a_i^- +a_{i-1}^-}{2}) \in \mathcal{T}(u_1) \cap \mathcal{T}(u_2)$ if $i>1$ and  $(x_1,...,x_k,\frac{a_1^- +a_1^+}{2}) \in \mathcal{T}(u_1) \cap \mathcal{T}(u_2)$ otherwise.
\item Suppose that $\Theta\models \ntpp{u_1}{a_i}$ and $\Theta\not\models \p{u_1}{a_{i-1}}$ (or $i=1$). Then $\Theta\models \ec{u_1}{u_2}$ implies $\Theta\models \ov{u_2}{a_i}$ and we again find $(x_1,...,x_k,\frac{a_i^- +a_{i-1}^-}{2}) \in \mathcal{T}(u_1) \cap \mathcal{T}(u_2)$ if $i>1$ and  $(x_1,...,x_k,\frac{a_1^- +a_1^+}{2}) \in \mathcal{T}(u_1) \cap \mathcal{T}(u_2)$ otherwise.
\item The case where $\Theta\models \p{u_2}{b_s}$ is entirely analogous.
\end{itemize}
\item Assume that $\Theta\models \po{u_1}{u_2}$ and let $(x_1,...,x_k) \in i(\mathcal{S}(u_1))\cap i(\mathcal{S}(u_2))$. We show that  $i(\mathcal{T}(u_1))\cap i(\mathcal{T}(u_2)) \neq \emptyset$ from which it easily follows that $\po{\mathcal{T}(u_1)}{\mathcal{T}(u_2)}$.
\begin{itemize}
\item If $\Theta\not\models \p{u_1}{a_r}$, $\Theta\not\models \p{u_1}{b_s}$, $\Theta\not\models \p{u_2}{a_r}$ and $\Theta\not\models \p{u_2}{b_s}$, it holds that $(x_1,...,x_k,0)\in i(\mathcal{T}(u_1)) \cap i(\mathcal{T}(u_2))$.
\item Suppose that $\Theta\models \tpp{u_1}{a_i}$ and $\Theta\not\models \p{u_1}{a_{i-1}}$ (or $i=1$).  Since $\Theta\models \po{u_1}{u_2}$ it must hold that $\Theta\models\ov{u_2}{a_i}$. It follows that $(x_1,...,x_k,\frac{a_i^- +a_{i-1}^-}{2}) \in i(\mathcal{T}(u_1)) \cap i(\mathcal{T}(u_2))$ if $i>1$ and  $(x_1,...,x_k,\frac{a_1^- +a_1^+}{2}) \in i(\mathcal{T}(u_1)) \cap i(\mathcal{T}(u_2))$ otherwise.
\item The cases where $\Theta\models \ntpp{u_1}{a_i}$ or $\Theta\models \p{u_2}{b_s}$ are entirely analogous.
\end{itemize}
\end{itemize}

\end{proof}

\begin{example}\label{exExtendNTPPDC}
Consider the consistent atomic network $\Theta$ which contains the following constraints
\begin{align*}
&\ntpp{x}{y} && \ntpp{u}{y} && \ntpp{y}{z} && \ec{x}{u}\\
&\ntpp{a_1}{a_2} && \ntpp{b_1}{b_2} && \dc{a_2}{b_2}\\
&\ec{u}{a_1} && \po{u}{a_2} && \po{u}{b_1}\\
&\ntpp{a_1}{y} && \po{y}{a_2} && \po{y}{b_1}\\
&\ntpp{a_2}{z} && \po{z}{b_1} && \tpp{x}{a_1}
\end{align*}
as well as all constraints which are implied by the above constraints. Figure \ref{figExtendNTPPDC1} shows a possible convex realization in $\mathbb{R}$ of the restricted network $\Theta{\downarrow}\{x,y,z,u\}$. The remaining regions, $a_1,a_2,b_1,b_2$, satisfy the condition of Proposition \ref{propExtendNTPPDC}. By applying the construction from the proof of Proposition \ref{propExtendNTPPDC}, a two-dimensional convex solution of $\Theta$ is obtained, which is shown in Figure \ref{figExtendNTPPDC2}. 
\end{example}
\begin{figure}
\centering
\includegraphics[width=300pt]{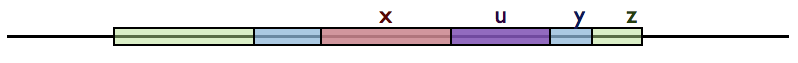}
\caption{Convex realization in $\mathbb{R}$ of the network $\Theta{\downarrow}\{x,y,z,u\}$ from Example \ref{exExtendNTPPDC}. \label{figExtendNTPPDC1}}
\end{figure}
\begin{figure}
\centering
\includegraphics[width=300pt]{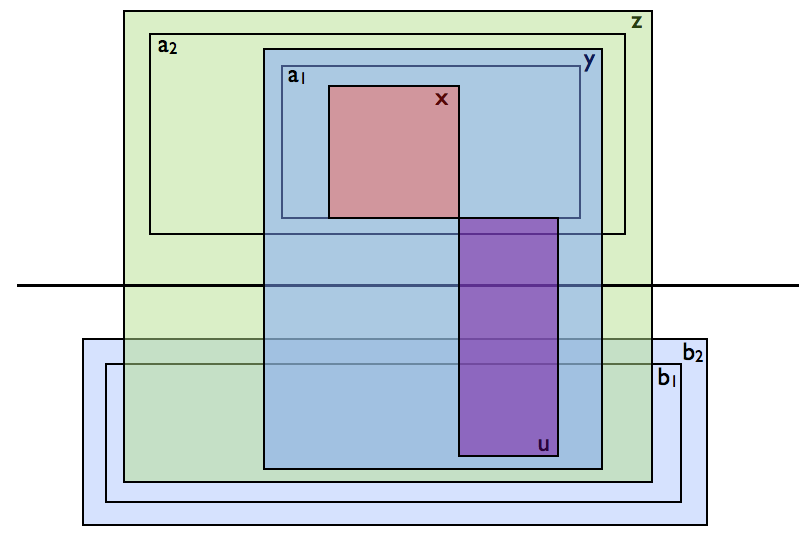}
\caption{Convex realization in $\mathbb{R}^2$ of the network $\Theta$ from Example \ref{exExtendNTPPDC}.\label{figExtendNTPPDC2}}
\end{figure}

Note that in the above proposition, it is possible that $r=0$, or $s=0$, or $r=s=1$. We have the following corollary.
\begin{corollary}\label{corollaryExtendWith2Regions}
Let $\Theta$ be a consistent atomic RCC8 network. If $\Theta{\downarrow}V'$ has a convex solution in $\mathbb{R}^k$ in which every \bo-clique has a common part and $\Theta\models a\{\bdc,\bntpp\} b$, then $\Theta{\downarrow}(V'\cup\{a,b\})$ has a convex solution $\mathbb{R}^{k+1}$ in which every \bo-clique has a common part.
\end{corollary}

Moreover,  we can generate a convex solution for any network, using the following corollary.
\begin{corollary}\label{corollaryExtendWith1Region}
Let $\Theta$ be a consistent atomic RCC8 network. If $\Theta{\downarrow}V'$ has a convex solution in $\mathbb{R}^k$ in which every \bo-clique has a common part, then $\Theta{\downarrow}(V'\cup\{a\})$ has a convex solution $\mathbb{R}^{k+1}$ in which every \bo-clique has a common part.
\end{corollary}
Together with Proposition \ref{propECDCTPPNTPPd}, we obtain the following upper bound.
\begin{corollary}
Let $\Theta$ be a consistent atomic RCC8 network and let $V'\subseteq V$ be such that $\Theta{\downarrow}V'$ does not contain any occurrences of $\bpo$.  It holds that $\Theta$ has a convex solution in $\mathbb{R}^{|V\setminus V'|+4}$.
\end{corollary}
Similar results can be obtained based on the fragments of RCC8, identified in Figure \ref{overviewRCC8}, which guarantee convex solutions in $\mathbb{R}$, $\mathbb{R}^2$ and $\mathbb{R}^3$.

\subsection{Upper bounds}
We now turn our attention to the problem of deriving an upper bound on the number of required dimensions which only depends on the number of variables. In particular, we will show that a consistent network with at most $2n+1$ regions can be convexly realized in $\mathbb{R}^n$, provided that $n\geq 2$. This result does not hold for $n=1$: although most consistent networks of three variables are convexly realizable on the real line, there are two exceptions, as made explicit in the following lemma.

\begin{lemma}\label{lemma1D3regions}
Let $\Theta$ be a consistent atomic RCC8 network over the set of variables $V$.  If $|V| = 3$ then $\Theta$ has a convex solution in $\mathbb{R}$ unless $\Theta$ is isomorphic to one of the following networks:
\begin{align*}
N_3^1 &= \{\ec{a}{b},\ec{b}{c},\ec{a}{c}\}\\
N_3^2 &= \{\po{a}{b},\ec{b}{c},\ec{a}{c}\}
\end{align*}
\end{lemma}
\begin{proof}
Let us write $V=\{a,b,c\}$. 
If $\Theta$ contains an instance of the relation $\bdc$ or $\bntpp$ then the result follows immediately from Proposition \ref{propExtendNTPPDC} (e.g.\ if $\Theta\models \dc{a}{b}$, then we can extend a 0-dimensional realization of $c$ to a 1-dimensional convex realization of $\Theta$). Assume that $\Theta$ contains an instance of $\btpp$, e.g.\ $\Theta \models \tpp{a}{b}$. We define the solution of $\Theta$ as follows, only considering the cases where $\Theta$ does not contain any instances of $\bdc$ or $\bntpp$:
\begin{itemize}
\item If $\tpp{c}{a}$ and $\tpp{c}{b}$: $\mathcal{S}(a)=[0,2]$, $\mathcal{S}(b)=[0,3]$ and $\mathcal{S}(c)=[0,1]$.
\item If $\tpp{a}{c}$ and $\tpp{b}{c}$: $\mathcal{S}(a)=[0,1]$, $\mathcal{S}(b)=[0,2]$ and $\mathcal{S}(c)=[0,3]$.
\item If $\tpp{a}{c}$ and $\po{c}{b}$: $\mathcal{S}(a)=[0,1]$, $\mathcal{S}(b)=[-1,1]$ and $\mathcal{S}(c)=[0,2]$.
\item If $\po{c}{a}$ and $\tpp{c}{b}$: $\mathcal{S}(a)=[0,2]$, $\mathcal{S}(b)=[-1,2]$ and $\mathcal{S}(c)=[-1,1]$.
\item If $\po{c}{a}$ and $\po{c}{b}$: $\mathcal{S}(a)=[0,2]$, $\mathcal{S}(b)=[0,3]$ and $\mathcal{S}(c)=[-1,1]$.
\item If $\ec{c}{a}$ and $\tpp{c}{b}$: $\mathcal{S}(a)=[0,1]$, $\mathcal{S}(b)=[0,2]$ and $\mathcal{S}(c)=[1,2]$.
\item If $\ec{c}{a}$ and $\ec{c}{b}$: $\mathcal{S}(a)=[0,1]$, $\mathcal{S}(b)=[0,2]$ and $\mathcal{S}(c)=[-1,0]$.
\item If $\ec{c}{a}$ and $\po{c}{b}$: $\mathcal{S}(a)=[0,1]$, $\mathcal{S}(b)=[0,2]$ and $\mathcal{S}(c)=[1,3]$.
\end{itemize}
If $\Theta$ does not contain any instances of $\bdc$, $\bntpp$ or $\btpp$, then either it is isomorphic to $N_3^1$ or it contains an instance of $\bpo$.  In the latter case, assuming that $\Theta \models \po{a}{b}$, we define the solution of $\Theta$ as follows:
\begin{itemize}
\item If $\po{c}{a}$ and $\po{c}{b}$: $\mathcal{S}(a)=[0,3]$, $\mathcal{S}(b)=[1,4]$ and $\mathcal{S}(c)=[2,5]$.
\item If $\po{c}{a}$ and $\ec{c}{b}$: $\mathcal{S}(a)=[1,3]$, $\mathcal{S}(b)=[2,4]$ and $\mathcal{S}(c)=[0,2]$.
\item If $\ec{c}{a}$ and $\po{c}{b}$: $\mathcal{S}(a)=[2,4]$, $\mathcal{S}(b)=[1,3]$ and $\mathcal{S}(c)=[0,2]$.
\item If $\ec{c}{a}$ and $\ec{c}{b}$ then $\Theta$ is isomorphic to $N_3^2$.
\end{itemize}
\end{proof}

To show that consistent networks with $2n+1$ variables have a convex solution in $\mathbb{R}^n$ we first need a number of technical lemmas. Let $\Theta'$ be the RCC8 network which is obtained from $\Theta$ by replacing every constraint of the form $\tpp{a}{b}$ by $\pp{a}{b}$ and every constraint of the form $\ec{a}{b}$ by $\dr{a}{b}$. We say that $\mathcal{S}$ is a weak solution of $\Theta$ iff $\mathcal{S}$ is a solution of $\Theta'$.

\begin{lemma}\label{lemmaExtendWeak}
Let $\Theta$ be a consistent atomic RCC8 network and assume that $\Theta{\downarrow}V'$ has a convex weak solution in $\mathbb{R}^k$ in which every \bo-clique has a common part. Let $u\in V\setminus V'$ be such that $\Theta{\downarrow}(V' \cup \{u\})$ is a network over $\{\bec,\bpo,\btpp,\btppi\}$ and assume that one of the following conditions is satisfied:
\begin{align}
&\text{for every $v\in V'$ it holds that $\Theta\not\models \tpp{v}{u}$ and $\Theta\not\models \tpp{u}{v}$} \label{eqLemma9A}\\
&\text{for every $v\in V'$ it holds that $\Theta\not\models \ec{v}{u}$ and $\Theta\not\models \tpp{u}{v}$}\label{eqLemma9B}
\end{align}
It holds that $\Theta{\downarrow}(V'\cup \{u\})$ has a convex solution in $\mathbb{R}^{k+1}$ in which every \bo-clique has a common part.
\end{lemma}
\begin{proof}
Let $\mathcal{S}$ be a $k$-dimensional convex weak solution of $\Theta{\downarrow}V'$. We now define a $(k+1)$ dimensional solution $\mathcal{T}$ of $\Theta{\downarrow}(V' \cup \{u\})$ as follows.

Let us first consider the case where $\Theta\not\models \tpp{v}{u}$ and $\Theta\not\models \tpp{u}{v}$ for all $v\in V'$.
Let $A$ be a non-empty  $k$-dimensional convex region which strictly contains $\mathcal{S}(v)$ for every $v\in V'$. We define:
$$
\mathcal{T}(u) = \{ (\lambda \cdot x_1,...,\lambda \cdot x_{k}, \lambda \,|\, (x_1,...,x_{k}) \in A, \lambda \in [1,2]  \}
$$
and for $v\in V'$:
$$
\mathcal{T}(v) = \{ (\lambda \cdot x_1,...,\lambda \cdot x_{k}, \lambda \,|\, (x_1,...,x_{k}) \in \mathcal{S}(v), \lambda \in [0,\lambda_v]  \}
$$
where $\lambda_v>0$ depends on the relation between $v$ and $u$:
\begin{align*}
\lambda_v = 
\begin{cases}
1 & \text{if $\Theta\models \ec{u}{v}$}\\
1.5 & \text{if $\Theta\models \po{u}{v}$}
\end{cases}
\end{align*}
Given that $(0,...,0)\in \bigcap_{v\in V'} \mathcal{T}(v)$, we then indeed have that $\mathcal{T}$ is a solution of $\Theta{\downarrow}V'$.
It is moreover clear that for each $v\in V'$, the RCC8 relation between $\mathcal{T}(u)$ and $\mathcal{T}(v)$ is the one required by $\Theta$. 

Now consider the case where $\Theta\not\models \ec{v}{u}$ and $\Theta\not\models \tpp{u}{v}$ for any $v$ in $V'$. Let $A$ be defined as before.  We define for all $v$ in $V'$:
\begin{align*}
\mathcal{T}(u) &= \{ (\lambda \cdot x_1,...,\lambda \cdot x_{k}, \lambda \,|\, (x_1,...,x_{k}) \in A, \lambda \in [0,2]  \}\\
\mathcal{T}(v) &= \{ (\lambda \cdot x_1,...,\lambda \cdot x_{k}, \lambda \,|\, (x_1,...,x_{k}) \in \mathcal{S}(v), \lambda \in [0,\lambda_v]  \}\\
\lambda_v &= 
\begin{cases}
1 & \text{if $\Theta\models \tpp{v}{u}$}\\
3 & \text{if $\Theta\models \po{v}{u}$}
\end{cases}
\end{align*}


\end{proof}

\begin{example}\label{exExtendTPPEC}
Consider the consistent atomic network $\Theta$ which contains the following constraints
\begin{align*}
\tpp{x}{y} && \ec{x}{z} && \po{y}{z} && \ec{x}{u} && \po{y}{u} && \po{z}{u} 
\end{align*}
Figure \ref{figExtendTPPEC1} shows a possible convex weak realization in $\mathbb{R}$ of the restricted network $\Theta{\downarrow}\{x,y,z\}$. By applying the construction from the proof of Lemma \ref{lemmaExtendWeak}, a two-dimensional convex solution of $\Theta$ is obtained, which is shown in Figure \ref{figExtendTPPEC2}. 
\end{example}

\begin{figure}[t]
\centering
\includegraphics[width=200pt]{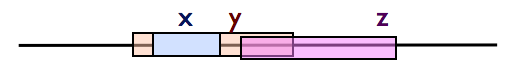}
\caption{Convex weak realization in $\mathbb{R}$ of the network $\Theta{\downarrow}\{x,y,z\}$ from Example \ref{exExtendTPPEC}. \label{figExtendTPPEC1}}
\end{figure}
\begin{figure}[t]
\centering
\includegraphics[width=250pt]{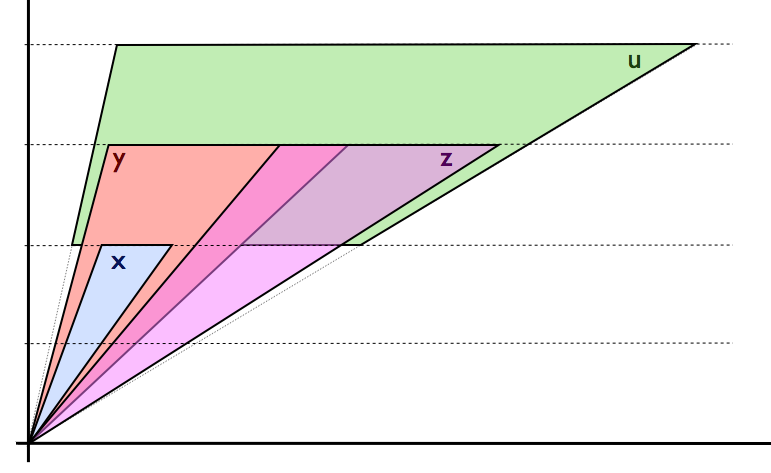}
\caption{Convex realization in $\mathbb{R}^2$ of the network $\Theta$ from Example \ref{exExtendTPPEC}.\label{figExtendTPPEC2}}
\end{figure}

\begin{remark}\label{lemmaExtendWeak-anti}
Let $\Theta$ be a consistent atomic RCC8 network and assume that $\Theta{\downarrow}V'$ is a network over $\{\bec,\bpo,\btpp,\btppi\}$. Then region $u\in V'$ does not satisfy \new{conditions \eqref{eqLemma9A} and \eqref{eqLemma9B}} if and only if one of the following conditions is satisfied:
\begin{itemize}
\item there exists  $v\in V'$ such that $\Theta\models \tpp{u}{v}$;
\item there exist $v_1,v_2\in V'$ such that $\Theta \models \tpp{v_1}{u}$ and $\Theta\models \ec{v_2}{u}$.
\end{itemize}
\end{remark}

\begin{lemma}
\label{lem:3-non}
Let $\Theta$ be an RCC8 constraint network over $\{\bpo,\btpp,\btppi,\bec\}$ that involves four variables. Suppose $\Theta$ is consistent but has no convex weak solution in $\mathbb{R}$. Then $\Theta$ is  isomorphic to one of the following three networks, where an arrow, a solid line, and a dotted line represent, respectively, a \btpp, \bpo, and \bec\ relation.
\interdisplaylinepenalty=10000
\begin{eqnarray*}
\label{eq:3-non-real-networks}
 \xymatrix{
    a \ar@{.}[r] \ar@{->}[d] \ar@{-}[dr]    & b \ar@{-}[d]  \ar@{->}[dl]   \\
    c\ar@{-}[r]  & d }
    &
    \xymatrix{
    a \ar@{-}[d] \ar@{.}[r] \ar@{-}[dr]    & b \ar@{-}[d]  \ar@{-}[dl]   \\
    c\ar@{.}[r]  & d  }
    &
    \xymatrix{
    a \ar@{.}[d] \ar@{.}[r] \ar@{-}[dr]    & b \ar@{-}[d]  \ar@{.}[dl]   \\
    c\ar@{-}[r]  & d }
    \\
    (N_4^1) \hspace*{7mm} & (N_4^2) & \hspace*{7mm} (N_4^3)
\end{eqnarray*}
\interdisplaylinepenalty=2500
\end{lemma}
\begin{proof}
This can be checked by case-by-case analysis. We verified this using a computer program. The code of the program is available from the authors.
\end{proof}
We note that each $N_4^i$ $(i=1,2,3)$ contains an \bec\ relation. The following result is clear.
\begin{corollary}
\label{coro:X1=4}
Let $\Theta$ be an RCC8 constraint network over $\{\bpo,\btpp,\btppi\}$ that involves four variables. Suppose $\Theta$ is consistent. Then $\Theta$ has a convex weak solution in $\mathbb{R}$.
\end{corollary}

We need a number of additional lemmas before we can present the main result.

\begin{lemma}\label{lemma5RealizableECPOTPP}
Let $\Theta$ be a consistent atomic RCC8 network over $\{\bec,\bpo,\btpp,\btppi\}$.  If $|V| = 5$ then $\Theta$ has a convex solution in $\mathbb{R}^2$ in which every \bo-clique has a common part.
\end{lemma}
\begin{proof}
The proof is provided in Appendix B.
\end{proof}

\begin{lemma}\label{lemma5RealizableFull}
Let $\Theta$ be a consistent atomic RCC8 network.  If $|V| = 5$ then $\Theta$ has a convex solution in $\mathbb{R}^2$ in which every \bo-clique has a common part.
\end{lemma}
\begin{proof}
The proof is provided in Appendix B.
\end{proof}

\begin{lemma}\label{lem:o+2ec-in2d}
Let $\Theta$ be a consistent atomic RCC8 network over $\{\bec,\bpo,\btpp,\btppi\}$ and $V=\{x,y,c_1,c_2,d_1,d_2\}$. If $\Theta\models \{\ec{c_1}{d_1},\ec{c_2}{d_2}\}$ then $\Theta$ has a convex weak solution in $\mathbb{R}^2$ in which every \bo-clique has a common part.
\end{lemma}
\begin{proof}
The proof is provided in Appendix B.
\end{proof}

\begin{lemma}\label{lem:7regions}
Suppose $\Theta$ is a basic RCC8 network over $\{\bec,\bpo,\btpp,\btppi\}$ and $V=\{x,y,z,a,b,c,d\}$. In addition, assume that $\Theta$ satisfies 
\begin{itemize}
\item $\Theta{\downarrow}\{a,b,c,d\}$ is $N_4^1,N_4^2$ or $N_4^3$;
\item there is no $u$ in $V$ which satisfies the conditions of Lemma \ref{lemmaExtendWeak};
\item $\Theta\models \{\tpp{y}{x},\ec{z}{x},\ec{z}{y}\}$ and  $\Theta\not\models \tpp{x}{v}$ for $v\in\{a,b,c,d\}$.
\end{itemize}
Then $\Theta$ has a convex weak solution in $\mathbb{R}^2$ in which every \bo-clique has a common part.
\end{lemma}
\begin{proof}
The proof is provided in Appendix B.
\end{proof}

\begin{proposition}\label{propRealizable2nPlus1}
Let $\Theta$ be a consistent atomic RCC8 network over the set of variables $V$.  If $|V| \leq 2n+1$ with $n\geq 2$ then $\Theta$ has a convex solution in $\mathbb{R}^n$.
\end{proposition}
\begin{proof}
Suppose $|V|=2n+1$ (the case where $|V|<2n+1$ follows trivially) and $n\geq 2$. 
Let $W=\{a_1,...,a_l,b_1,...,b_l\}$ be such that:
\begin{itemize}
\item for each $i\in \{1,...,l\}$ it holds that $\Theta \models a_i\{\bdc,\bntpp\}b_i$;
\item $\Theta{\downarrow}(V\setminus W)$ is a network over $\{\bec,\bpo,\btpp,\btppi\}$.
\end{itemize}

Note that the number of elements in $V\setminus W$ is odd. If $|V\setminus W|=1$, then $l=n$, and we can start with a zero-dimensional solution of the unique element in $V\setminus W$ and extend it to an $n$-dimensional convex solution of $\Theta$ by repeatedly applying Proposition \ref{propExtendNTPPDC}. If $|V\setminus W| = 3$, we know from Lemma \ref{lemma5RealizableFull} that a 2-dimensional convex solution of $\Theta{\downarrow}((V\setminus W) \cup \{a_1,b_1\})$ exists. This solution can then be extended to a $2+(l-1)$ dimensional convex solution of $\Theta$ by repeatedly applying Proposition \ref{propExtendNTPPDC}, where $2l = 2n+1-3$ and thus $2 + (l-1) = n$. Similarly, if $|V\setminus W| = 5$ it follows from Lemma \ref{lemma5RealizableECPOTPP} that a 2-dimensional convex solution of $\Theta{\downarrow}(V\setminus W)$ exists, which can be extended to an $n$-dimensional convex solution of $\Theta$ by repeated application of Proposition \ref{propExtendNTPPDC}.

Now assume that $|V\setminus W| \geq 7$. First assume that there is a region $z \in V\setminus W$ which satisfies the conditions of Lemma \ref{lemmaExtendWeak}. Then there exists a $U = \{c_1,...,c_k,d_1,...,d_k\} \subseteq V\setminus (W \cup \{z\})$ such that 
\begin{itemize}
\item for each $i\in \{1,...,k\}$ it holds that $\Theta \models \ec{c_i}{d_i}$;
\item $\Theta{\downarrow}(V\setminus (W\cup U \cup\{z\}))$ is a network over $\{\bpo,\btpp,\btppi\}$.
\end{itemize}
Note that $V\setminus(W\cup U \cup \{z\})$ contains an even number of elements. If $|V\setminus(W\cup U \cup \{z\})| \geq 6$ we start with a 2-dimensional convex solution of $\Theta{\downarrow}(V\setminus(W\cup U \cup \{z\}))$, whose existence is guaranteed by \new{Corollary} \ref{propRealizationPOTPP}. This solution is then extended to a $k+2$ dimensional weak convex solution of $\Theta{\downarrow}(V\setminus(W \cup \{z\}))$, by repeatedly applying the method from Proposition \ref{propExtendNTPPDC}, and then extended to a $k+3$ dimensional convex solution of $\Theta{\downarrow}(V\setminus W)$ by applying the method from Lemma \ref{lemmaExtendWeak}. Finally, by repeatedly applying the method from Proposition \ref{propExtendNTPPDC}, we obtain a $k+l+3$ dimensional convex solution of $\Theta$. Note that $2n+1 = |W\cup U \cup \{z\}| + |V\setminus(W\cup U \cup \{z\})| = 2k+2l+1 + |V\setminus(W\cup U \cup \{z\})| \geq 2k+2l+7$, from which we obtain $k+l+3 \leq n$.

If $|V\setminus(W\cup U \cup \{z\})| = 4$, we know by Corollary \ref{coro:X1=4} that $\Theta{\downarrow}(V\setminus(W\cup U \cup \{z\}))$ has a convex weak solution in $\mathbb{R}$. By applying Proposition \ref{propExtendNTPPDC} and Lemma \ref{lemmaExtendWeak}, as before, we can extend this weak solution to a convex solution of $\Theta$ in $\mathbf{R}^{k+l+2}$, with $k+l+2 \leq n$.

Suppose $|V\setminus(W\cup U \cup \{z\})| = 2$, and write $V\setminus(W\cup U \cup \{z\}) = \{x,y\}$. 
Note that $U=\{c_i,d_i: 1\leq i\leq k\}$ for some $k\geq 2$ and we have $\Theta\models\ec{c_i}{d_i}$ for $1\leq i\leq k$. By Lemma~\ref{lem:o+2ec-in2d}, we know $\Theta{\downarrow}\{x,y,c_1,d_1,c_2,d_2\}$ has a convex weak solution in $\mathbb{R}^2$, which can be extended to a convex solution in $\mathbb{R}^{k+l+1}$, using Proposition \ref{propExtendNTPPDC} and Lemma \ref{lemmaExtendWeak} as before.

Suppose $|V\setminus(W\cup U \cup \{z\})| = 0$. 
Then $U=\{c_i,d_i: 1\leq i\leq k\}$ for some $k\geq 3$. By Lemma~\ref{lem:o+2ec-in2d}, we know $\Theta{\downarrow}\{c_1,c_2,c_3,d_1,d_2,d_3\}$ has a convex weak solution in $\mathbb{R}^2$, which can be extended to a convex solution in $\mathbb{R}^{k+l+1}$, using Proposition \ref{propExtendNTPPDC} and Lemma \ref{lemmaExtendWeak} as before.

\vskip 2mm

Now assume that there is no region $z \in V\setminus W$ which satisfies the conditions of Lemma \ref{lemmaExtendWeak}. By \new{Remark}~\ref{lemmaExtendWeak-anti}, this implies that for each $u\in V\setminus W$, we have either $\Theta\models \tpp{u}{v}$ for some $v\in V\setminus W$ or there exist $v_1,v_2$ such that $\Theta\models\{\tpp{v_1}{u},\ec{v_2}{u}\}$. In particular,  there are $x,y,z \in V\setminus W$ such that $\Theta \models \{\tpp{y}{x}, \ec{x}{z}, \ec{y}{z}\}$ 
and $\Theta\not\models \tpp{x}{v}$ for any $v\in V\setminus W$.

If $|V \setminus (W \cup \{x,y,z\})| \geq 6$, we can proceed similarly as before. In particular, we construct a convex weak solution of $V \setminus (W \cup \{x,y,z\})$ in $\mathbb{R}^l$, with $|V \setminus (W \cup \{x,y,z\})| = 2l+2$. Then we extend it to a convex weak solution of $V \setminus W$ in $\mathbb{R}^{l+1}$ using the method from Proposition \ref{propExtendNTPPDC}, and we modify it to a convex solution in $\mathbb{R}^{l+2}$ of $V \setminus W$. Finally, we extend it to a convex solution in $\mathbb{R}^n$ of $\Theta$.

The only remaining case is where $|V \setminus (W \cup \{x,y,z\})| = 4$.  If $V \setminus (W \cup \{x,y,z\})$ is not isomorphic to one of the networks $N_4^1,  N_4^2, N_4^3$, we can construct a convex weak solution in $\mathbb{R}$ and extend it to a convex solution in $\mathbb{R}^n$ of $\Theta$ as before. Let us write $V \setminus (W \cup \{x,y,z\}) = \{a,b,c,d\}$ and assume that $\Theta{\downarrow}\{a,b,c,d\}$ is isomorphic to one of $N_4^1,  N_4^2, N_4^3$. By Lemma~\ref{lem:7regions} we know that $\Theta{\downarrow}(V \setminus W)$ has a convex weak solution in $\mathbb{R}^2$, from which it follows that $\Theta{\downarrow}(V \setminus W)$ has a convex solution in $\mathbb{R}^3$.
\end{proof}

\subsection{Lower bounds}\label{secLowerboundNumberRegions}
\subsubsection*{2 dimensions}
Consider the RCC8 network $\Theta_{2D}$ from Example \ref{exNot2Dconvex}.
Clearly $\Theta_{2D}$ is consistent. However, it does not have any convex realizations in two dimensions. Indeed, if $\mathcal{S}$ were a two-dimensional convex solution, we clearly would have $\dim(\mathcal{S}(a) \cap \mathcal{S}(b)) \leq 1$. If $\dim(\mathcal{S}(a)\cap \mathcal{S}(b))=0$ then $\mathcal{S}(x)$, $\mathcal{S}(y)$, $\mathcal{S}(u)$ and $\mathcal{S}(v)$ could only meet in one point, which means that $\ec{x}{u}$,$\ec{u}{y}$  and $\dc{x}{y}$ could not be jointly satisfied. If $\dim(\mathcal{S}(a)\cap \mathcal{S}(b))=1$ then $\mathcal{S}(x)\cap \mathcal{S}(u)$, $\mathcal{S}(x)\cap \mathcal{S}(v)$, $\mathcal{S}(y)\cap \mathcal{S}(u)$ and $\mathcal{S}(y) \cap \mathcal{S}(v)$ are nonempty and pairwise disjoint. Take four points $P_i$ $(i=1,2,3,4)$ from these sets and suppose $P_{j_1}<P_{j_2}<P_{j_3}<P_{j_4}$.  We note that $P_1$ and $P_2$ are both in $\mathcal{S}(x)$, and $P_3$ and $P_4$ are both in $\mathcal{S}(y)$. Because $x$ and $y$ are disjoint, we know $\{P_1,P_2\}=\{P_{j_1},P_{j_2}\}$ or $\{P_1,P_2\}=\{P_{j_3},P_{j_4}\}$. Similarly, note that $P_1$ and $P_3$ are both in $\mathcal{S}(u)$, and $P_2$ and $P_4$ are both in $\mathcal{S}(v)$. Because $u$ and $v$ are disjoint, we know $\{P_1,P_3\}=\{P_{j_1},P_{j_2}\}$ or $\{P_2,P_4\}=\{P_{j_3},P_{j_4}\}$. This is a contradiction.

\subsection*{3 dimensions}
Consider the network $\Theta_{3D}$ obtained by adding the following constraints to $\Theta_{2D}$:
\begin{align*}
\ec{c}{d} &&
\tpp{x}{c}&&
\tpp{y}{c}&&
\tpp{u}{d}&&
\tpp{v}{d}\\
\po{a}{c}&&
\po{a}{d}&&
\po{b}{c}&&
\po{b}{d}
\end{align*}
To see that $\Theta_{3D}$ is indeed not realizable in three dimensions, we show that $\dim(\mathcal{S}(a) \cap \mathcal{S}(b) \cap \mathcal{S}(c)\cap \mathcal{S}(d)) \leq 1$ for any three-dimensional convex solution $\mathcal{S}$, which leads to a contradiction as in the two-dimensional case.

Let $H_1$ be a hyperplane that separates $\mathcal{S}(a)$ and $\mathcal{S}(b)$, and let $H_2$ be a hyperplane that separates $\mathcal{S}(c)$ and $\mathcal{S}(d)$. This implies that $\mathcal{S}(a)\cap \mathcal{S}(b)\subseteq H_1$ and $\mathcal{S}(c)\cap \mathcal{S}(d)\subseteq H_2$.  All we need to show is that $H_1 \neq H_2$. This is clear, however, because if $H_1=H_2$ then $H_1$ would separate $\mathcal{S}(a)$ from $\mathcal{S}(c)$ or $\mathcal{S}(d)$, which means that $\mathcal{S}(a)$ could not partially overlap with both $\mathcal{S}(c)$ and $\mathcal{S}(d)$.  Thus $\dim(H_1\cap H_2) \leq 1$ which also means $\dim(\mathcal{S}(a) \cap \mathcal{S}(b) \cap \mathcal{S}(c)\cap \mathcal{S}(d)) \leq 1$.

\subsection*{$n$ dimensions}
For any $n\geq 4$ we consider the network $\Theta_{nD}$ obtained by adding the following constraints to $\Theta_{2D}$ for $i \in \{0,...,n-3\}$
\begin{align*}
\tpp{e_i}{a}&&
\ntpp{e_i}{f_i}&&
\tpp{a}{f_i}&&
\ec{g_i}{f_i}\\
\ec{g_i}{a}&&
\tpp{u}{g_i}&&
\tpp{v}{g_i}&&
\ec{e_i}{b}\\
\tpp{g_i}{b}&&
\end{align*}
and the following constraints for $i \in \{1,...,n-3\}$
\begin{align*}
\ec{e_i}{g_{i-1}}&&
\tpp{g_i}{g_{i-1}}
\end{align*}

We show that $\dim(\mathcal{S}(a) \cap \mathcal{S}(g_{n-3})) \leq 1$ for any $n$-dimensional convex solution $\mathcal{S}$, which again leads to a contradiction, and thus that $\Theta_{nD}$ is not realizable by $n$-dimensional convex regions.

Let $G_{i}$ be a hyperplane separating $\mathcal{S}(g_i)$ and $\mathcal{S}(f_i)$ for $i\in \{0,...,n-3\}$, and let $H_1$ be a hyperplane separating $\mathcal{S}(a)$ and $\mathcal{S}(b)$ as before.  We show by induction that $\dim(H_1 \cap G_0 \cap .... \cap G_k) \leq n-k-2$ for every $k\in \{0,...,n-3\}$, from which the stated immediately follows. 


First assume that $k=0$.  It suffices to show that $H_1 \neq G_0$ to show that $\dim(H_1\cap G_0) \leq n-2$.  If $H_1 = G_0$, we would have that $\mathcal{S}(a)\cap \mathcal{S}(b) \subseteq G_0$, and in particular that $G_0$ contains a boundary point of $\mathcal{S}(e_0)$; call this point $P$. However, since $\ntpp{\mathcal{S}(e_0)}{\mathcal{S}(f_0)}$, $P$ would also belong to $\mathcal{S}(f_0)$, and since $G_0$ only contains boundary points of $\mathcal{S}(f_0)$, we would have that $P$ is a boundary point of $\mathcal{S}(f_0)$ as well. This is a contradiction, since $\ntpp{e_0}{f_0}$ means that $\mathcal{S}(e_0)$ can only contain internal points of $\mathcal{S}(f_0)$.

For $k>0$, we show that $H_1 \cap G_0 \cap ... \cap G_{k-1} \not \subseteq G_{k}$ in a similar way.  Suppose $H_1 \cap G_0 \cap ... \cap G_{k-1} \subseteq G_{k}$ did hold.  We have that $H_1,G_0,...,G_{k-1}$ all separate $\mathcal{S}(a)$ from $\mathcal{S}(g_{k-1})$, hence $\mathcal{S}(a) \cap \mathcal{S}(g_{k-1}) \subseteq H_1 \cap G_0 \cap ... \cap G_{k-1} \subseteq G_{k}$.  Since $\mathcal{S}(e_k)\subseteq \mathcal{S}(a)$ and $\mathcal{S}(e_k) \cap \mathcal{S}(g_{k-1}) \neq \emptyset$, there must exist a point $P$ in $\mathcal{S}(e_k) \cap \mathcal{S}(g_{k-1})$ which is thus also in $G_{k}$. Clearly the point $P$ is a boundary point of $\mathcal{S}(e_k)$, and since $\ntpp{e_k}{f_k}$, we have that $P$ is an internal point of $\mathcal{S}(f_k)$. This is a contraction, since $G_{k}$ was assumed to be a hyperplane separating $\mathcal{S}(f_k)$ from $\mathcal{S}(g_k)$.

For any number of dimensions $n$ we can thus find an RCC network which is consistent but cannot be realized by convex $n$-dimensional regions. The counterexamples we have provided for two and three dimensions are optimal, in the sense that they involve $2n+2$ regions, i.e.\ any convex network with fewer regions would necessarily be realizable by convex regions in $n$ dimensions. The counterexample for $n\geq 4$ dimensions, on the other hand, uses $3n$ regions, and the question whether counterexamples with fewer regions exist remains open.

\section{Conclusions }\label{secConcluding}
We have studied how the consistency problem for RCC8 networks is affected by the requirement that all regions need to be convex. Previous results about convexity in RCC8 have been largely negative: for any fixed number of dimensions $k$, deciding whether a consistent RCC8 network $\Theta$ can be realized using convex regions in $\mathbb{R}^k$ is computationally hard \cite{Davis:1999}. In contrast, we have identified a number of important sufficient conditions under which consistent RCC8 networks have convex solutions. First, we have identified all restrictions on the set of base relations that guarantee that consistent atomic networks can be convexly realized in $\mathbb{R}$, $\mathbb{R}^2$, $\mathbb{R}^3$ and $\mathbb{R}^4$. Second, we have shown an upper bound which only depends on the number of regions, i.e.\ every consistent RCC8 network with at most $2n+1$ regions can be convexly realized in $\mathbb{R}^n$ ($n\geq 2$).

Our main motivation in this paper was to justify the use of standard RCC8 reasoners to reason about convex regions in high-dimensional spaces, e.g.\ to support qualitative reasoning about conceptual spaces \cite{gardenfors2001reasoning,dlrcc8}. \new{While existing work on conceptual spaces is mostly based on geometric representations (e.g.\ corresponding to the Voronoi diagram induced by a set of prototypes \cite{gardenfors2001reasoning}), in practice it can be difficult to obtain accurate region-based representations from data, especially for concepts which are relatively rare. In contrast, existing lexical resources such as WordNet\footnote{\url{http://wordnet.princeton.edu}} and ConcepNet\footnote{\url{http://conceptnet5.media.mit.edu}} encode several relations that can be interpreted as qualitative spatial relations between conceptual space representations. For example, the hypernym/hyponym relations encoded in WordNet can be seen as $\bpp$ and from ConceptNet it is sometimes possible to derive $\bec$ relations. For example, ConceptNet encodes\footnote{\url{http://conceptnet5.media.mit.edu/web/c/en/jog}} that ``jogging is a kind of exercise similar to running'' and ``jogging is slower than running'', from which we may derive $\ec{\textit{jogging}}{\textit{running}}$. Similarly, natural language processing methods, e.g.\ based on Hearst patterns \cite{Hearst:1992:AAH:992133.992154}, can be used to learn qualitative semantic relations from natural language sentences. To use such knowledge about the meaning of concepts in applications, we therefore need methods for qualitative reasoning about meaning. The results we have presented in this paper provide the first step towards such applications.}

Future work will build on these results by investigating generalizations of RCC5 and RCC8 which are tailored towards reasoning about concepts. For example, \cite{schockaert2013combining} has focused on adding a betweenness relation to RCC5, which is important for formalizing particular forms of commonsense reasoning \cite{AIJSchockaertPrade2013}. In \cite{Sheremet01062007}, a logic for concepts is introduced in which comparative distance can be expressed (i.e.\ ``$a$ is closer to $b$ than to $c$''), which is important for formalizing categorization, but without guarantees that concepts can be realized as convex regions. An interesting question would then be whether the results from this paper, viz.\ the observation that requiring convexity does not affect consistency if the number of dimensions is sufficiently large, carries over to such more expressive settings.

\section*{Acknowledgments}
We thank the anonymous reviewers for their invaluable comments and detailed suggestions. The work of Sanjiang Li was partially supported by ARC (DP120104159) and NSFC (61228305). The work of Steven Schockaert was partially supported by EPSRC (EP/K021788/1).

\appendix
\section{\new{Complexity of deciding convex realizability in $\mathbb{R}^k$}}

\new{In \cite{Davis:1999} it was shown that deciding whether a set of RCC8 constraints has a convex solution in $\mathbb{R}^2$ is as hard as deciding the consistency of a set of algebraic equations (of similar size) over the real numbers. This latter problem belongs to the complexity class $\exists{R}$, which has been introduced in \cite{Schaefer:2010aa}. The class $\exists{R}$ is known to be in PSPACE and to contain NP, but its exact relationship with the polynomial hierarchy is yet to be determined. Here we show that deciding whether a set of RCC8 constraints has a convex solution in $\mathbb{R}^k$ is $\exists{R}$-complete for any fixed $k\geq 2$. The membership proof follows easily from the membership proof for the 2-dimensional case (cf.\ the remark after Corollary 2 in \cite{Davis:1999}).} 

\new{We show by induction that deciding whether a (not necessarily atomic) consistent RCC8 network over $\{\bpp,\bec\}$ has a convex solution in $\mathbb{R}^{k}$ is $\exists{R}$-hard. The case $k=2$ has been shown in \cite{Davis:1999}. Now suppose that we have already shown this result for $k=n-1$. Let $\Theta$ be any consistent RCC8 network over $\{\bpp,\bec\}$ and let $V$ be the set of variables occurring in $\Theta$. We construct the RCC8 network $\Psi$ by adding to $\Theta$ the following constraints:
\begin{itemize}
\item  $\ec{a}{b}$ (for $a$ and $b$ fresh variables);
\item  $\pp{v_1}{v}$, $\pp{v_2}{v}$, $\pp{v_1}{a}$ and $\pp{v_2}{b}$ for every $v\in V$ (where $v_1$ and $v_2$ are fresh variables);
\item  $\ec{u_1}{v_1}$ and $\ec{u_2}{v_2}$ for every constraint of the form $\ec{u}{v}$ in $\Theta$.
\end{itemize}
We show that $\Psi$ has a convex solution in $\mathbb{R}^n$ iff $\Theta$ has a convex solution in $\mathbb{R}^{n-1}$. First suppose that $\Theta$ has a convex solution $\mathcal{S}$ in $\mathbb{R}^{n-1}$ and let $X = \ch\{\mathcal{S}(v) \,|\, v\in V\}$. Then we can define a convex solution $\mathcal{T}$ of $\Psi$ in $\mathbb{R}^n$ as follows:
\begin{align*}
\mathcal{T}(a) &= \{(x_1,...,x_{n-1},y) \,|\, (x_1,...,x_{n-1}) \in X, y\in [-1,0] \}\\
\mathcal{T}(b) &= \{(x_1,...,x_{n-1},y) \,|\, (x_1,...,x_{n-1}) \in X, y\in [0,1] \}
\end{align*}
and for each $u\in V$
\begin{align*}
\mathcal{T}(u) &= \{(x_1,...,x_{n-1},y) \,|\, (x_1,...,x_{n-1}) \in \mathcal{S}(u), y\in [-1,1] \}\\
\mathcal{T}(u_1) &= \{(x_1,...,x_{n-1},y) \,|\, (x_1,...,x_{n-1}) \in \mathcal{S}(u), y\in [-1,0] \}\\
\mathcal{T}(u_2) &= \{(x_1,...,x_{n-1},y) \,|\, (x_1,...,x_{n-1}) \in \mathcal{S}(u), y\in [0,1] \}
\end{align*}
It is straightforward to verify that $\mathcal{T}$ is indeed a convex solution of $\Psi$.}

\new{Conversely, assume that $\Psi$ has a convex solution $\mathcal{T}$ in $\mathbb{R}^n$. From $\ec{a}{b}$ it follows that there is an $n-1$ dimensional hyperplane $H$ separating $\mathcal{T}(a)$ and $\mathcal{T}(b)$. For $v\in V$, we define:
$$
\mathcal{S}(v) = H\cap \mathcal{T}(v)
$$
Because $\mathcal{T}(v_1)$ and $\mathcal{T}(v_2)$ are at opposite sides of $H$, we find that $H\cap \mathcal{T}(v)$ corresponds to a regular closed region in $\mathbb{R}^{n-1}$. It is also clear from the convexity of $\mathcal{T}(v)$ that $\mathcal{S}(v)$ is convex. Finally, we need to show that $\mathcal{S}$ is a solution of $\Theta$.} 

\new{For a constraint of the form $\ec{u}{v} \in \Theta$, we know that $\Psi$ contains the constraints $\ec{u_1}{v_1}$ and $\ec{u_2}{v_2}$, and we can thus choose points $p\in \mathcal{T}(u_1)\cap \mathcal{T}(v_1)$ and $q\in \mathcal{T}(u_2)\cap \mathcal{T}(v_2)$. The line segment between $p$ and $q$ intersects the hyperplane $H$ at a point $r$. Since $p \in \mathcal{T}(u)\cap \mathcal{T}(v)$ and $q \in \mathcal{T}(u)\cap \mathcal{T}(v)$, because of the convexity of $\mathcal{T}(u)$ and $\mathcal{T}(v)$, we find $r \in \mathcal{T}(u)\cap \mathcal{T}(v)$, and thus $\mathcal{S}(u) \cap \mathcal{S}(v) \neq \emptyset$. Furthermore it is clear that $\mathcal{S}(u)$ and $\mathcal{S}(v)$ cannot overlap, hence we find that $\mathcal{S}$ satisfies the constraint $\ec{u}{v}$.
}

\new{For a constraint of the form $\pp{u}{v} \in \Theta$, it is clear from $\mathcal{T}(u)\subset \mathcal{T}(v)$ that $\mathcal{S}(u)\subseteq \mathcal{S}(v)$. Moreover, we can assume without loss of generality that whenever $\pp{u}{v} \in \Theta$ there is a $z\in V$ such that $\ec{u}{z}$ and $\tpp{z}{v}$ (i.e.\ we can always introduce such variables $z$ without affecting the consistency of $\Theta$). It then easily follows from $\mathcal{S}(u)\subseteq \mathcal{S}(v)$ that $\mathcal{S}(u)\subset \mathcal{S}(v)$.}

\section{Proofs of the lemmas from Section 5.1}

\setcounter{lemma}{10}
\begin{lemma}\label{lemma5RealizableECPOTPP}
Let $\Theta$ be a consistent atomic RCC8 network over $\{\bec,\bpo,\btpp,\btppi\}$.  If $|V| = 5$ then $\Theta$ has a convex solution in $\mathbb{R}^2$ in which every \bo-clique has a common part.
\end{lemma}
\begin{proof}
Let us write $V=\{a,b,c,d,e\}$.
\begin{itemize}
\item First assume that there is no element $u$ among $\{a,b,c,d,e\}$ satisfying the conditions from Lemma 9. By \new{Remark} 2, this implies that
$\Theta$ contains a subnetwork which is isomorphic to one of the following two networks:
\begin{align*}
&
\xymatrix{
    a \ar@{.}[r]  \ar@{.}[dr]    & b \ar@{.}[dl]  \\
    c\ar@{->}[u] \ar@{.}[r]     & d\ar@{->}[u]  }
    &&
 \xymatrix{
    a \ar@{-}[r]  \ar@{.}[dr]    & b \ar@{.}[dl]  \\
    c\ar@{->}[u] \ar@{.}[r]     & d\ar@{->}[u]  }\\
    &\hspace*{5mm} (N_4^4) && \hspace*{5mm} (N_4^5)
\end{align*}
 The network $N_4^4$ can be convexly realized in $\mathbb{R}$, e.g.\ we can define a convex solution $\mathcal{S}$ as:
\begin{align*}
\mathcal{S}(a) &= [0,2] &
\mathcal{S}(b) &= [2,4] &
\mathcal{S}(c) &= [1,2] &
\mathcal{S}(d) &= [2,3]
\end{align*}
If $\Theta$ contains a subnetwork which is isomorphic to $N_4^4$, it  follows from Corollary 15 that $\Theta$ has a convex solution in $\mathbb{R}^2$. Now suppose that $\Theta$ has a subnetwork which is isomorphic to $N_4^5$ but not a subnetwork which is isomorphic to $N_4^4$, and moreover that $e$ is externally connected to or contained in at least one of $a,b,c,d$, and $e$ either contains or is contained in one of $a,b,c,d$ (otherwise $e$ satisfies the conditions of Lemma 9 which we assumed not to be the case). 
We then have that $\Theta$ is isomorphic to one of the following networks:
\begin{align*}
&
\xymatrix@C=1em{
  & e &  \\
 a\ar@{.}[ur]\ar@{-}[rr]\ar@{.}[drr] &  & b \ar@{-}[ul]\ar@{.}[dll]\\
 c\ar@{->}[u] \ar@{.}@/^3pc/[uur] \ar@{.}[rr] & & d\ar@{->}[u] \ar@{->}@/_3pc/[uul]}
 &&
\xymatrix@C=1em{
  & e &  \\
 a\ar@{.}[ur]\ar@{-}[rr]\ar@{.}[drr] &  & b \ar@{<-}[ul]\ar@{.}[dll]\\
 c\ar@{->}[u] \ar@{.}@/^3pc/[uur] \ar@{.}[rr] & & d\ar@{->}[u] \ar@{.}@/_3pc/[uul]}
&&
\xymatrix@C=1em{
  & e &  \\
 a\ar@{.}[ur]\ar@{-}[rr]\ar@{.}[drr] &  & b \ar@{<-}[ul]\ar@{.}[dll]\\
 c\ar@{->}[u] \ar@{.}@/^3pc/[uur] \ar@{.}[rr] & & d\ar@{->}[u] \ar@{-}@/_3pc/[uul]}\\
    &\hspace*{5mm} (N_5^1) && \hspace*{5mm} (N_5^2) && \hspace*{5mm} (N_5^3)\\
&
\xymatrix@C=1em{
  & e &  \\
 a\ar@{-}[ur]\ar@{-}[rr]\ar@{.}[drr] &  & b \ar@{-}[ul]\ar@{.}[dll]\\
 c\ar@{->}[u] \ar@{.}@/^3pc/[uur] \ar@{.}[rr] & & d\ar@{->}[u] \ar@{->}@/_3pc/[uul]}
 &&
\xymatrix@C=1em{
  & e &  \\
 a\ar@{-}[ur]\ar@{-}[rr]\ar@{.}[drr] &  & b \ar@{->}[ul]\ar@{.}[dll]\\
 c\ar@{->}[u] \ar@{.}@/^3pc/[uur] \ar@{.}[rr] & & d\ar@{->}[u] \ar@{->}@/_3pc/[uul]}
 &&
 \xymatrix@C=1em{
  & e &  \\
 a\ar@{-}[ur]\ar@{-}[rr]\ar@{.}[drr] &  & b \ar@{<-}[ul]\ar@{.}[dll]\\
 c\ar@{->}[u] \ar@{.}@/^3pc/[uur] \ar@{.}[rr] & & d\ar@{->}[u] \ar@{->}@/_3pc/[uul]}\\
    &\hspace*{5mm} (N_5^4) && \hspace*{5mm} (N_5^5) && \hspace*{5mm} (N_5^6)\\
 &
 \xymatrix@C=1em{
  & e &  \\
 a\ar@{-}[ur]\ar@{-}[rr]\ar@{.}[drr] &  & b \ar@{<-}[ul]\ar@{.}[dll]\\
 c\ar@{->}[u] \ar@{.}@/^3pc/[uur] \ar@{.}[rr] & & d\ar@{->}[u] \ar@{-}@/_3pc/[uul]}
 &&
 \xymatrix@C=1em{
  & e &  \\
 a\ar@{-}[ur]\ar@{-}[rr]\ar@{.}[drr] &  & b \ar@{<-}[ul]\ar@{.}[dll]\\
 c\ar@{->}[u] \ar@{.}@/^3pc/[uur] \ar@{.}[rr] & & d\ar@{->}[u] \ar@{.}@/_3pc/[uul]}
 &&
 \xymatrix@C=1em{
  & e &  \\
 a\ar@{<-}[ur]\ar@{-}[rr]\ar@{.}[drr] &  & b \ar@{<-}[ul]\ar@{.}[dll]\\
 c\ar@{->}[u] \ar@{.}@/^3pc/[uur] \ar@{.}[rr] & & d\ar@{->}[u] \ar@{.}@/_3pc/[uul]}\\
    &\hspace*{5mm} (N_5^7) && \hspace*{5mm} (N_5^8) && \hspace*{5mm} (N_5^9)
\end{align*}
Each of these networks has a convex solution in $\mathbb{R}^2$, as shown in Figure \ref{figlemma5RealizableECPOTPPA}.
\item Next assume that $V'=\{a,b,c,d\}$ and $u=e$ satisfy one of the conditions from Lemma 9. 
\begin{itemize}
\item If $\Theta{\downarrow}\{a,b,c,d\}$ is not isomorphic to the the networks $N_4^1$, $N_4^2$ and $N_4^3$ from Lemma 10, it follows from Lemma 10 that $\Theta{\downarrow}\{a,b,c,d\}$ has a convex weak solution in $\mathbb{R}$. Using Lemma 9, we then obtain that $\Theta$ has a convex solution in $\mathbb{R}^2$.
\item Suppose $\Theta{\downarrow}\{a,b,c,d\}$ is isomorphic to $N_4^1$. Note that since $e$ satisfies one of the conditions from Lemma 9 that $\Theta\not\models \tpp{e}{c}$. If $\Theta\models \{\tpp{c}{e},\po{d}{e}\}$ then the conditions of Lemma 9 are satisfied for $u=d$ and since $\Theta{\downarrow}\{a,b,c,e\}$ is not isomorphic to any of $N_4^1,N_4^2,N_4^3$ it follows that $\Theta$ has a convex solution.  If $\Theta\models \{\tpp{c}{e},\tpp{d}{e}\}$ then $\Theta$ is isomorphic to the following network:
\begin{align*}
 &\xymatrix@C=1em{
  & e &  \\
 a\ar@{->}[ur]\ar@{.}[rr]\ar@{-}[drr] &  & b \ar@{->}[ul]\ar@{->}[dll]\\
 c\ar@{<-}[u] \ar@{->}@/^3pc/[uur] \ar@{-}[rr] & & d\ar@{-}[u] \ar@{->}@/_3pc/[uul]}\\
 &\hspace*{5mm} (N_5^{10})
\end{align*}
where an arrow, a solid line, and a dotted line again represent a \btpp, \bpo, and \bec\ relation as before.
A convex solution of this network is shown in Figure \ref{figlemma5RealizableECPOTPPA}.

If $\Theta\models \ec{c}{e}$ then $\Theta$ is isomorphic to one of the following networks (considering that $\Theta \not\models \tpp{d}{e}$ because $e$ satisfies the conditions of Lemma 9):
\begin{align*}
&
\xymatrix@C=1em{
  & e &  \\
 a\ar@{.}[ur]\ar@{.}[rr]\ar@{-}[drr] &  & b \ar@{.}[ul]\ar@{->}[dll]\\
 c\ar@{<-}[u] \ar@{.}@/^3pc/[uur] \ar@{-}[rr] & & d\ar@{-}[u] \ar@{-}@/_3pc/[uul]}
 &&
\xymatrix@C=1em{
  & e &  \\
 a\ar@{.}[ur]\ar@{.}[rr]\ar@{-}[drr] &  & b \ar@{.}[ul]\ar@{->}[dll]\\
 c\ar@{<-}[u] \ar@{.}@/^3pc/[uur] \ar@{-}[rr] & & d\ar@{-}[u] \ar@{.}@/_3pc/[uul]}
 \end{align*}
In both cases $\Theta{\downarrow}\{a,b,c,e\}$ is not isomorphic to any of $N_4^1,N_4^2,N_4^3$ from which we find that $\Theta$ has a convex solution in $\mathbb{R}^2$.
Finally, if $\Theta\models \po{c}{e}$, then the conditions of Lemma 9 are also satisfied for $V'=\{a,b,d,e\}$ and $u=c$. 
Moreover, if in addition $\Theta\models \tpp{d}{e}$, then $\Theta{\downarrow}\{a,b,c,e\}$ is not isomorphic to one of the networks $N_4^1, N_4^2, N_4^3$ and $\Theta$ has a convex solution in $\mathbb{R}^2$; otherwise, if $\Theta\models d\{\bpo,\bec\} e$, then the conditions of Lemma~9 are also satisfied for $V'=\{a,b,c,e\}$ and $u=d$. In this case,  if $\Theta{\downarrow}\{a,b,d,e\}$ or $\Theta{\downarrow}\{a,b,c,e\}$ is not isomorphic to one of the networks $N_4^1, N_4^2, N_4^3$ we again find that $\Theta$ has a convex solution in $\mathbb{R}^2$ from Lemmas 10 and 9. Otherwise, it follows that $\Theta$ is isomorphic to the following network:
\begin{align*}
&
\xymatrix@C=1em{
  & e &  \\
 a\ar@{-}[ur]\ar@{.}[rr]\ar@{-}[drr] &  & b \ar@{-}[ul]\ar@{->}[dll]\\
 c\ar@{<-}[u] \ar@{-}@/^3pc/[uur] \ar@{-}[rr] & & d\ar@{-}[u] \ar@{.}@/_3pc/[uul]}\\
  &\hspace*{5mm} (N_5^{11})
 \end{align*}
A convex solution of this network is shown in Figure \ref{figlemma5RealizableECPOTPPA}.


\item Suppose $\Theta{\downarrow}\{a,b,c,d\}$ is isomorphic to $N_4^2$.  First suppose $e$ contains all other regions, i.e.\ $\Theta$ is isomorphic to the following network:
\begin{align*}
&
\xymatrix@C=1em{
  & e &  \\
 a\ar@{->}[ur]\ar@{.}[rr]\ar@{-}[drr] &  & b \ar@{->}[ul]\ar@{-}[dll]\\
 c\ar@{-}[u] \ar@{->}@/^3pc/[uur] \ar@{.}[rr] & & d\ar@{-}[u] \ar@{->}@/_3pc/[uul]}\\
   &\hspace*{5mm} (N_5^{12})
\end{align*}
A convex solution of this network is shown in Figure \ref{figlemma5RealizableECPOTPPA}.

Now suppose that $e$ contains some but not all of the regions $a,b,c,d$. Note that in this case all of $a,b,c,d$ are in relation $\btpp$ or $\bpo$ with $e$.  Assume e.g.\ that $\Theta\models \tpp{a}{e}$. If $\Theta\models \po{b}{e}$ then the conditions of Lemma 9 are satisfied for $u=b$ and $\Theta{\downarrow}\{a,b,e,c\}$ is not isomorphic to any of $N_4^1,N_4^2,N_4^3$ from which we obtain that $\Theta$ has a convex solution in $\mathbb{R}^2$. Therefore assume that $\Theta\models \tpp{b}{e}$. If $\Theta\models \{\tpp{c}{e},\po{d}{e}\}$  then the conditions of Lemma 9 are satisfied for respectively $u=d$ and $\Theta{\downarrow}\{a,b,c,e\}$ is not isomorphic to any of $N_4^1,N_4^2,N_4^3$ which means that $\Theta$ has a convex solution in $\mathbb{R}^2$. The case where $\Theta\models \{\tpp{d}{e},\po{c}{e}\}$ is isomorphic. The only remaining possibility is that $\Theta$ is isomorphic to the following network:
\begin{align*}
&
\xymatrix@C=1em{
  & e &  \\
 a\ar@{->}[ur]\ar@{.}[rr]\ar@{-}[drr] &  & b \ar@{->}[ul]\ar@{-}[dll]\\
 c\ar@{-}[u] \ar@{-}@/^3pc/[uur] \ar@{.}[rr] & & d\ar@{-}[u] \ar@{-}@/_3pc/[uul]}\\
    &\hspace*{5mm} (N_5^{13})
\end{align*}
A convex solution of this network is shown in Figure \ref{figlemma5RealizableECPOTPPA}.

Finally assume that $\Theta$ does not contain any occurrence of $\btpp$. Then any subset of four regions satisfies the conditions for $V'$ from Lemma 9. If at least any sub-network of size four is not isomorphic to the networks $N_4^1$, $N_4^2$ and $N_4^3$ from Lemma 10, then it is clear that $\Theta$ has a convex solution in $\mathbb{R}^2$. If all of these sub-networks are isomorphic to $N_4^1$, $N_4^2$ or $N_4^3$, it follows that $\Theta$ is isomorphic to the following network:
\begin{align*}
&
\xymatrix@C=1em{
  & e &  \\
 a\ar@{.}[ur]\ar@{.}[rr]\ar@{-}[drr] &  & b \ar@{.}[ul]\ar@{-}[dll]\\
 c\ar@{-}[u] \ar@{-}@/^3pc/[uur] \ar@{.}[rr] & & d\ar@{-}[u] \ar@{-}@/_3pc/[uul]}\\
    &\hspace*{5mm} (N_5^{14})
\end{align*}
A convex solution of this network is shown in Figure \ref{figlemma5RealizableECPOTPPA}.

\item Suppose $\Theta{\downarrow}\{a,b,c,d\}$ is isomorphic to $N_4^3$. If $e$ contains all other regions, $\Theta$ is isomorphic to:
\interdisplaylinepenalty=10000
 \begin{align*}
 &
\xymatrix@C=1em{
  & e &  \\
 a\ar@{->}[ur]\ar@{.}[rr]\ar@{-}[drr] &  & b \ar@{->}[ul]\ar@{.}[dll]\\
 c\ar@{.}[u] \ar@{->}@/^3pc/[uur] \ar@{-}[rr] & & d\ar@{-}[u] \ar@{->}@/_3pc/[uul]}\\
    &\hspace*{5mm} (N_5^{15})
\end{align*}
A convex solution of this network is shown in Figure \ref{figlemma5RealizableECPOTPPA}.

Next consider the case where $e$ contains some but not all of the regions $a,b,c,d$. If $\Theta\models \po{d}{e}$ then the conditions of Lemma 9 are satisfied for $u=d$ and since $\Theta{\downarrow}\{a,b,c,e\}$ is not isomorphic to any of $N_4^1,N_4^2,N_4^3$ it follows that $\Theta$ has a convex solution in $\mathbb{R}^2$. If $\Theta\models \tpp{d}{e}$ and e.g.\ $\Theta\models \po{a}{e}$,
the conditions of Lemma 9 are satisfied for $u=a$ and  $\Theta{\downarrow}\{b,c,d,e\}$ is not isomorphic to any of $N_4^1,N_4^2,N_4^3$.

If $\Theta$ does not contain any occurrences of $\btpp$ and each sub-network of size four is isomorphic to the networks $N_4^1$, $N_4^2$ and $N_4^3$ from Lemma~10, then $\Theta$ is isomorphic to the network $N_5^{14}$ which we already considered.
\end{itemize}

\end{itemize}
\end{proof}

\begin{figure}
\subfigure[$N_5^1$]{\includegraphics[width=120pt]{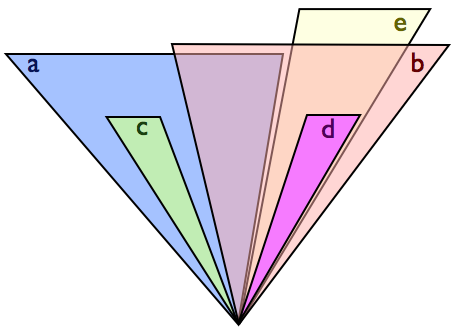}} \hfill
\subfigure[$N_5^2$]{\includegraphics[width=120pt]{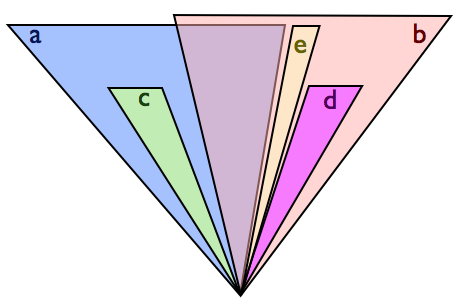}} \hfill
\subfigure[$N_5^3$]{\includegraphics[width=120pt]{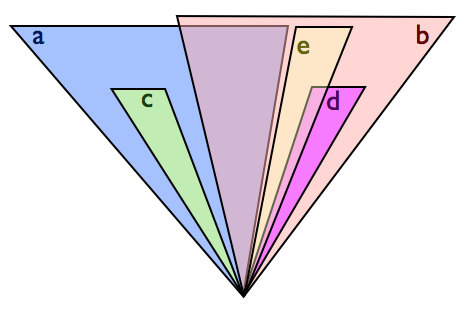}} \\
\subfigure[$N_5^4$]{\includegraphics[width=120pt]{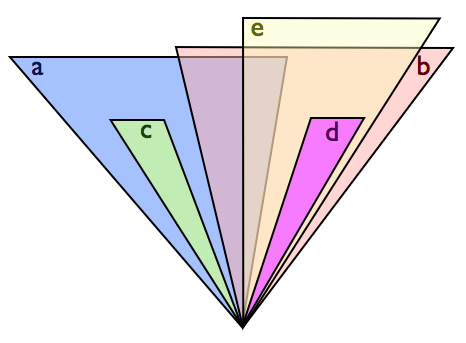}} \hfill
\subfigure[$N_5^5$]{\includegraphics[width=120pt]{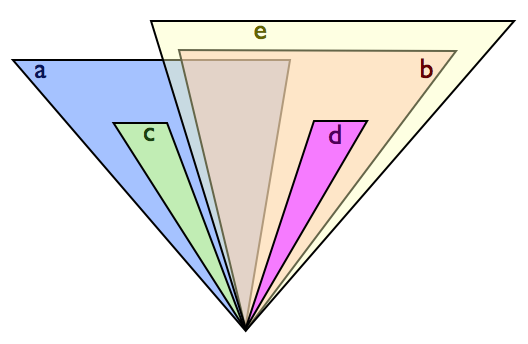}} \hfill
\subfigure[$N_5^6$]{\includegraphics[width=120pt]{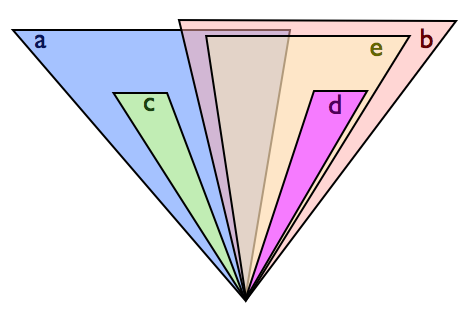}} \\
\subfigure[$N_5^7$]{\includegraphics[width=120pt]{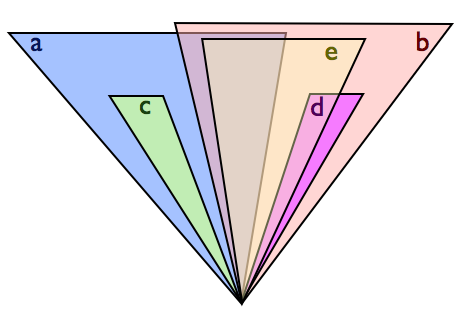}} \hfill
\subfigure[$N_5^8$]{\includegraphics[width=120pt]{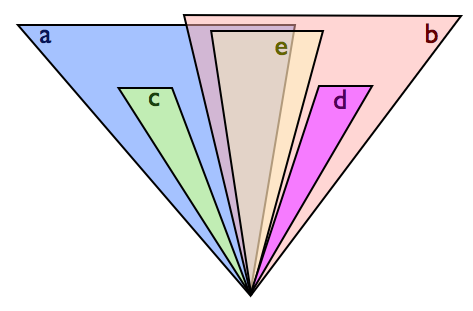}} \hfill
\subfigure[$N_5^9$]{\includegraphics[width=120pt]{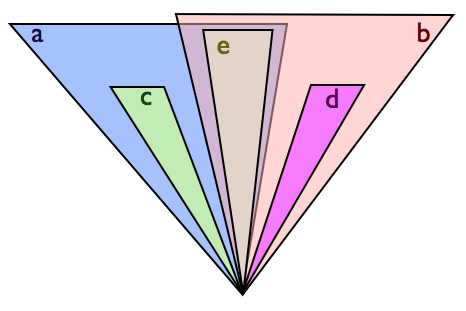}} \\
\subfigure[$N_5^{10}$]{\includegraphics[width=100pt]{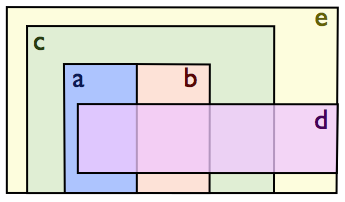}} \hfill
\subfigure[$N_5^{11}$]{\includegraphics[width=100pt]{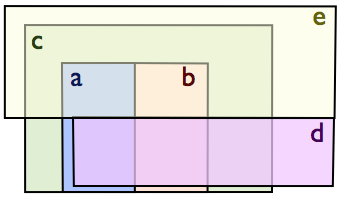}} \hfill
\subfigure[$N_5^{12}$]{\includegraphics[width=100pt]{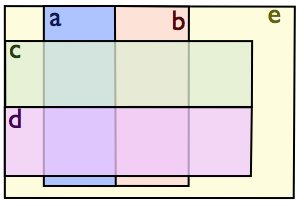}} \\
\subfigure[$N_5^{13}$]{\includegraphics[width=100pt]{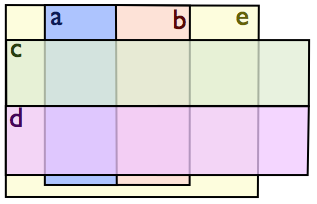}} \hfill
\subfigure[$N_5^{14}$]{\includegraphics[width=100pt]{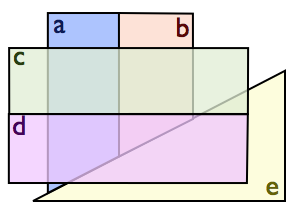}} \hfill
\subfigure[$N_5^{15}$]{\includegraphics[width=100pt]{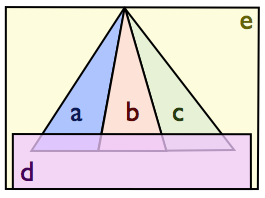}} 
\caption{Convex solutions of the networks $N_5^i$ from the proof of Lemma \ref{lemma5RealizableECPOTPP}. \label{figlemma5RealizableECPOTPPA}}
\end{figure}

\begin{lemma}\label{lemma5RealizableFull}
Let $\Theta$ be a consistent atomic RCC8 network.  If $|V| = 5$ then $\Theta$ has a convex solution in $\mathbb{R}^2$ in which every \bo-clique has a common part.
\end{lemma}
\begin{proof}
Let $V=\{a,b,c,d,e\}$.
If $\Theta$ does not contain any occurrences of $\bntpp$ or $\bdc$, the result follows directly from Lemma \ref{lemma5RealizableECPOTPP}.  Therefore assume that $\Theta \models d\{\bdc,\bntpp\}e$.  
\begin{itemize}
\item If $\Theta{\downarrow}\{a,b,c\}$ is convexly realizable in $\mathbb{R}$, it follows from Proposition 20 that $\Theta$ has a convex solution in $\mathbb{R}^2$.
\item Suppose $\Theta{\downarrow}\{a,b,c\}$ is isomorphic to the network $N_3^1$ from Lemma 8. 
\begin{itemize}
\item First assume that $\Theta\not\models \pp{x}{y}$ for any $x \in \{d,e\}$ and $y\in \{a,b,c\}$.
Let $\Theta'$ be obtained from $\Theta$ by weakening the constraints $\ec{a}{b}$, $\ec{b}{c}$ and $\ec{a}{c}$ to $\dr{a}{b}$, $\dr{b}{c}$ and $\dr{a}{c}$.  Let $\mathcal{S}'$ be an arbitrary convex solution in $\mathbb{R}$ of $\Theta'{\downarrow}\{a,b,c\}$ and let $\mathcal{T}'$ be the resulting convex solution in $\mathbb{R}^2$ of $\Theta'$ which is obtained by applying Proposition 20. Let $a^-,b^-,c^-,d^-,e^-,a^+,b^+,c^+,d^+,e^+$ be defined as in the proof of Proposition 20, then we have e.g.\ $\mathcal{T}'(a) = \{(x,y) \,|\, x\in \mathcal{S}'(a), y \in [a^-,a^+]\}$ and similar for $b$, $c$, $d$ and $e$. A property of the construction from the proof of Proposition 20 is that the values of $a^-,b^-,c^-,a^+,b^+,c^+$ do not depend on $\mathcal{S}'$; they only depend on the choice of $d^-,d^+,e^-,e^+$.  Assume  $a^- \leq b^- \leq c^-$ (the other cases are isomorphic).  We now define a convex solution $\mathcal{T}$ of $\Theta$ as follows: 
\begin{align*}
\mathcal{T}(a) &= \{(x,y) \,|\, x\in [-1,0], y\in [a^-,a^+] \}\\
\mathcal{T}(b) &= \{(x \cdot (\lambda-b^-),\lambda) \,|\, \lambda \in [b^-,b^+], x\in [1,2] \}\\
\mathcal{T}(c) &= \{(x \cdot (\lambda-b^-),\lambda) \,|\, \lambda \in [c^-,c^+], x\in [0,1] \}\\
\mathcal{T}(d) &= \{(x,y) \,|\, x\in [-2, \max(2(b^+-b^-),c^+-b^-) + 1], y\in [d^-,d^+] \}\\
\mathcal{T}(e) &= \{(x,y) \,|\, x\in [-3, \max(2(b^+-b^-),c^+-b^-) + 2], y\in [e^-,e^+] \}
\end{align*}
It is easy to verify that $\mathcal{T}$ is indeed a solution of $\Theta$. As an illustration of this procedure, consider the following network:
 \begin{align*}
 &
\xymatrix@C=1em{
  & a &  \\
 b\ar@{.}[ur]\ar@{.}[rr]\ar@{-}[drr] &  & c \ar@{.}[ul]\ar@{-}[dll]\\
 d\ar@{.}[u]  & & e\ar@{.}[u] \ar@{<-}@/_3pc/[uul]}\\
    &\hspace*{5mm} (N_5^{16})
\end{align*}
where the absence of an arrow denotes $\bdc$ and a single arrow, solid line and dotted line denote $\btpp$, $\bpo$ and $\bec$ as before. The convex solution in $\mathbb{R}^2$ which is obtained by applying the above procedure is shown in Figure \ref{figlemma5RealizableFullA}.

\item Now assume that $\Theta \models \pp{e}{a}$.  The main idea is to use a procedure similar to the method from Proposition 20 to extend a convex solution in $\mathbb{R}$ of $\Theta{\downarrow}\{b,c\}$ to a convex solution in $\mathbb{R}^2$ of $\Theta$. In fact, if $\ntpp{d}{a}$ and $\ntpp{a}{e}$ both hold, or $\dc{d}{e}$ and $\ntpp{a}{d}$, we can simply apply the method from Proposition 20. However, we may also have $\tpp{d}{a}$, $\tpp{e}{a}$, $\tpp{a}{e}$ or $\po{a}{e}$, but these cases can be handled in a similar way. As an example, in Figure \ref{figlemma5RealizableFullA} a convex solution is shown of the following network:
 \begin{align*}
 &
\xymatrix@C=1em{
  & a &  \\
 b\ar@{.}[ur]\ar@{.}[rr]\ar@{-}[drr] &  & c \ar@{.}[ul]\\
 d\ar@{.}[u] \ar@{->}@/^3pc/[uur] \ar@{=>}[rr] & & e\ar@{.}[u] \ar@{<-}@/_3pc/[uul]}\\
    &\hspace*{5mm} (N_5^{17})
\end{align*}
where a double arrow denotes $\bntpp$.
\end{itemize}
\item Suppose $\Theta{\downarrow}\{a,b,c\}$ is isomorphic to the network $N_3^2$ from Lemma 8. 
\begin{itemize}
\item First assume that $\Theta\not\models \pp{x}{y}$ for any $x \in \{d,e\}$ and $y\in \{a,b,c\}$. Let $\Theta'$ be defined as before and let again $a^-,b^-,c^-,d^-,e^-,a^+,b^+,c^+,d^+,e^+$ be the bounds obtained from the solution $\mathcal{T}'$ of $\Theta'$.  A property of the bounds $a^-,a^+,c^-,c^+$ is that $[a^-,a^+]\cap [c^-,c^+]\neq \emptyset$. Let $l \in [a^-,a^+]\cap [c^-,c^+]$. We define a convex solution $\mathcal{T}'$ in $\mathbb{R}^2$ as follows:
\begin{align*}
\mathcal{T}(a) &= \ch( \{(x,y) \,|\, x\in [0,2], y\in [a^-,a^+] \} \cup \{(3,l)\}) \\
\mathcal{T}(b) &= \{(x,y) \,|\, x\in [1,3], y\in [b^-,b^+] \}\\
\mathcal{T}(c) &= \{(x,y) \,|\, x\in [3,4], y\in [c^-,c^+] \}\\
\mathcal{T}(d) &= \{(x,y) \,|\, x\in [-1,5], y\in [d^-,d^+] \}\\
\mathcal{T}(e) &= \{(x,y) \,|\, x\in [-2,6], y\in [e^-,e^+] \}
\end{align*}
\item The cases where $\Theta\models\pp{d}{c}$, and the cases where $\Theta\models\pp{d}{a}$, $\Theta\not\models\pp{d}{b}$ and $\Theta\not\models\pp{e}{b}$, are similar to the network $N_5^{17}$ which we already considered. We consider the following networks:

\interdisplaylinepenalty=10000
\begin{align*}
 &
\xymatrix@C=1em{
  & a &  \\
 b\ar@{-}[ur]\ar@{.}[rr]\ar@{<-}[drr] &  & c \ar@{.}[ul]\\
 d\ar@{=>}[u] \ar@{=>}@/^3pc/[uur] \ar@{=>}[rr] & & e\ar@{.}[u] \ar@{->}@/_3pc/[uul]}
 &&
\xymatrix@C=1em{
  & a &  \\
 b\ar@{-}[ur]\ar@{.}[rr]\ar@{->}[drr] &  & c \ar@{.}[ul]\\
 d\ar@{=>}[u] \ar@{->}@/^3pc/[uur] \ar@{=>}[rr] & & e\ar@{-}[u] \ar@{-}@/_3pc/[uul]} 
 &&
\xymatrix@C=1em{
  & a &  \\
 b\ar@{-}[ur]\ar@{.}[rr]\ar@{<-}[drr] &  & c \ar@{.}[ul]\\
 d\ar@{=>}[u] \ar@{->}@/^3pc/[uur]  & & e \ar@{=>}@/_3pc/[uul]} 
 \\
    &\hspace*{5mm} (N_5^{18})
    &&\hspace*{5mm} (N_5^{19})
        &&\hspace*{5mm} (N_5^{20})
\end{align*}
\interdisplaylinepenalty=2500

A convex solution of these networks is shown in Figure \ref{figlemma5RealizableFullA}. The remaining cases are analogous.
\end{itemize}
\end{itemize}
\end{proof}

\begin{figure}
\subfigure[$N_5^{16}$]{\includegraphics[width=100pt]{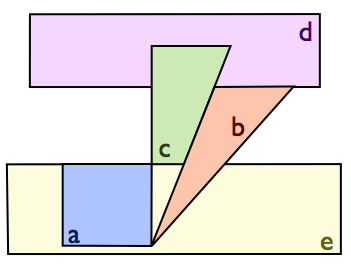}} \hfill
\subfigure[$N_5^{17}$]{\includegraphics[width=100pt]{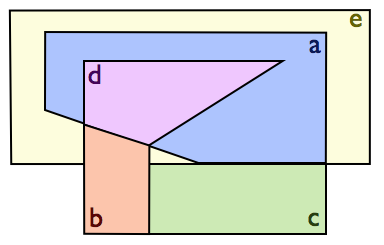}} \hfill
\subfigure[$N_5^{18}$]{\includegraphics[width=100pt]{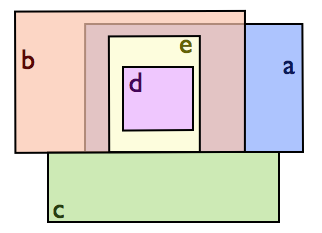}} \\
\subfigure[$N_5^{19}$]{\includegraphics[width=100pt]{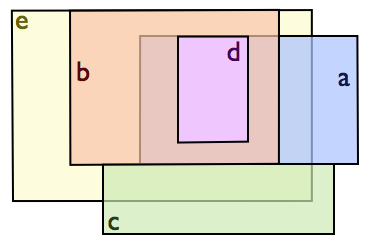}} \hfill
\subfigure[$N_5^{20}$]{\includegraphics[width=100pt]{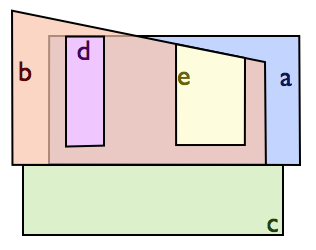}} \hfill
\makebox[100pt]{}
\caption{Convex solutions of the networks $N_5^i$ from the proof of Lemma \ref{lemma5RealizableFull}. \label{figlemma5RealizableFullA}}
\end{figure}

\begin{lemma}\label{lem:o+2ec-in2d}
Let $\Theta$ be a consistent atomic RCC8 network over $\{\bec,\bpo,\btpp,\btppi\}$ and $V=\{x,y,c_1,c_2,d_1,d_2\}$. If $\Theta\models \{\ec{c_1}{d_1},\ec{c_2}{d_2}\}$ then $\Theta$ has a convex weak solution in $\mathbb{R}^2$ in which every \bo-clique has a common part.
\end{lemma}
\begin{proof}
Suppose $\Theta\not\models \ec{x}{y}$. 
If there is an $i$ such that $\Theta{\downarrow}\{x,y,c_i,d_i\}$ is not isomorphic to the networks $N_4^1$ and $N_4^3$ from Lemma 10 (noting that $\Theta{\downarrow}\{x,y,c_i,d_i\}$ cannot be isomorphic to $N_4^2$), it holds that $\Theta{\downarrow}\{x,y,c_i,d_i\}$ has a convex weak solution in $\mathbb{R}$, which can be extended to a convex weak solution in $\mathbb{R}^2$ using Proposition 20. Therefore we assume that for every choice of $i$, it holds that $\Theta{\downarrow}\{x,y,c_i,d_i\}$ is isomorphic to $N_4^1$ or $N_4^3$. This implies in particular that $\Theta\models \po{x}{y}$, as otherwise $\Theta{\downarrow}\{x,y,c_i,d_i\}$ will be not isomorphic to $N_4^1,N_4^2$ and $N_4^3$. Furthermore, if $\Theta \models \ec{c_i}{d_j}$ we can assume that $\Theta{\downarrow}(\{x,y,c_1,c_2,d_1,d_2\} \setminus \{c_i,d_j\})$ is not isomorphic to $N_4^1,N_4^2$ and $N_4^3$ because otherwise we could again obtain a convex weak solution in $\mathbb{R}^2$ by using Proposition 20.
Under these assumptions, we can show that  $\Theta$ is isomorphic to one of the following networks:
\interdisplaylinepenalty=10000
\begin{align*}
&
 \xymatrix@C=1.5em{
  & c_1 \ar@{.}[r]\ar@{-}[dd]\ar@{-}[ddr]  & d_1\ar@{-}[dd]\ar@{-}[ddl] &\\
   x\ar@{-}[rrr] \ar@{<-}[ur] \ar@{<-}[dr] \ar@{<-}[urr]\ar@{<-}[drr]  & &		& y\ar@{-}[ull]\ar@{-}[ul]\ar@{-}[dll]\ar@{-}[dl] \\
  & c_2\ar@{.}[r]   & d_2 &}
  &&
 \xymatrix@C=1.5em{
  & c_1 \ar@{.}[r]\ar@{.}[dd]\ar@{-}[ddr]  & d_1\ar@{.}[dd]\ar@{-}[ddl] &\\
   x\ar@{-}[rrr] \ar@{<-}[ur] \ar@{<-}[dr] \ar@{<-}[urr]\ar@{<-}[drr]  & &		& y\ar@{-}[ull]\ar@{-}[ul]\ar@{-}[dll]\ar@{-}[dl] \\
  & c_2\ar@{.}[r]   & d_2 &} 
  &&
 \xymatrix@C=1.5em{
  & c_1 \ar@{.}[r]\ar@{.}[dd]\ar@{.}[ddr]  & d_1\ar@{.}[dd]\ar@{.}[ddl] &\\
   x\ar@{-}[rrr] \ar@{<-}[ur] \ar@{<-}[dr] \ar@{<-}[urr]\ar@{<-}[drr]  & &		& y\ar@{-}[ull]\ar@{-}[ul]\ar@{-}[dll]\ar@{-}[dl] \\
  & c_2\ar@{.}[r]   & d_2 &}\\
 &\hspace*{15mm} (N_6^1) && \hspace*{15mm} (N_6^2) && \hspace*{15mm} (N_6^3) \allowdisplaybreaks\\ 
  &
 \xymatrix@C=1.5em{
  & c_1 \ar@{.}[r]\ar@{-}[dd]\ar@{-}[ddr]  & d_1\ar@{-}[dd]\ar@{-}[ddl] &\\
   x\ar@{-}[rrr] \ar@{<-}[ur] \ar@{-}[dr] \ar@{<-}[urr]\ar@{-}[drr]  & &		& y\ar@{-}[ull]\ar@{-}[ul]\ar@{<-}[dll]\ar@{<-}[dl] \\
  & c_2\ar@{.}[r]   & d_2 &}
  %
  %
  %
  %
  &&
 \xymatrix@C=1.5em{
  & c_1 \ar@{.}[r]\ar@{-}[dd]\ar@{-}[ddr]  & d_1\ar@{-}[dd]\ar@{-}[ddl] &\\
   x\ar@{-}[rrr] \ar@{<-}[ur] \ar@{-}[dr] \ar@{<-}[urr]\ar@{-}[drr]  & &		& y\ar@{-}[ull]\ar@{-}[ul]\ar@{.}[dll]\ar@{.}[dl] \\
  & c_2\ar@{.}[r]   & d_2 &} 
  &&
  \xymatrix@C=1.5em{
  & c_1 \ar@{.}[r]\ar@{-}[dd]\ar@{-}[ddr]  & d_1\ar@{-}[dd]\ar@{-}[ddl] &\\
   x\ar@{-}[rrr] \ar@{.}[ur] \ar@{.}[dr] \ar@{.}[urr]\ar@{.}[drr]  & &		& y\ar@{-}[ull]\ar@{-}[ul]\ar@{-}[dll]\ar@{-}[dl] \\
  & c_2\ar@{.}[r]   & d_2 &} \\
 &\hspace*{15mm} (N_6^4) && \hspace*{15mm} (N_6^5) && \hspace*{15mm} (N_6^6)\allowdisplaybreaks\\ 
 &
  \xymatrix@C=1.5em{
  & c_1 \ar@{.}[r]\ar@{.}[dd]\ar@{-}[ddr]  & d_1\ar@{.}[dd]\ar@{-}[ddl] &\\
   x\ar@{-}[rrr] \ar@{.}[ur] \ar@{.}[dr] \ar@{.}[urr]\ar@{.}[drr]  & &		& y\ar@{-}[ull]\ar@{-}[ul]\ar@{-}[dll]\ar@{-}[dl] \\
  & c_2\ar@{.}[r]   & d_2 &} 
  &&
    \xymatrix@C=1.5em{
  & c_1 \ar@{.}[r]\ar@{.}[dd]\ar@{.}[ddr]  & d_1\ar@{.}[dd]\ar@{.}[ddl] &\\
   x\ar@{-}[rrr] \ar@{.}[ur] \ar@{.}[dr] \ar@{.}[urr]\ar@{.}[drr]  & &		& y\ar@{-}[ull]\ar@{-}[ul]\ar@{-}[dll]\ar@{-}[dl] \\
  & c_2\ar@{.}[r]   & d_2 &} 
  &&
  \xymatrix@C=1.5em{
  & c_1 \ar@{.}[r]\ar@{-}[dd]\ar@{-}[ddr]  & d_1\ar@{-}[dd]\ar@{-}[ddl] &\\
   x\ar@{-}[rrr] \ar@{.}[ur] \ar@{-}[dr] \ar@{.}[urr]\ar@{-}[drr]  & &		& y\ar@{-}[ull]\ar@{-}[ul]\ar@{.}[dll]\ar@{.}[dl] \\
  & c_2\ar@{.}[r]   & d_2 &}   %
\\
 &\hspace*{15mm} (N_6^7) && \hspace*{15mm} (N_6^8) && \hspace*{15mm} (N_6^9)\\ \allowdisplaybreaks
  &
  \xymatrix@C=1.5em{
  & c_1 \ar@{.}[r]\ar@{.}[dd]\ar@{-}[ddr]  & d_1\ar@{-}[dd]\ar@{-}[ddl] &\\
   x\ar@{-}[rrr] \ar@{.}[ur] \ar@{-}[dr] \ar@{.}[urr]\ar@{-}[drr]  & &		& y\ar@{-}[ull]\ar@{-}[ul]\ar@{.}[dll]\ar@{.}[dl] \\
  & c_2\ar@{.}[r]   & d_2 &}   
   &&
  \xymatrix@C=1.5em{
  & c_1 \ar@{.}[r]\ar@{.}[dd]\ar@{.}[ddr]  & d_1\ar@{-}[dd]\ar@{-}[ddl] &\\
   x\ar@{-}[rrr] \ar@{.}[ur] \ar@{-}[dr] \ar@{.}[urr]\ar@{-}[drr]  & &		& y\ar@{-}[ull]\ar@{-}[ul]\ar@{.}[dll]\ar@{.}[dl] \\
  & c_2\ar@{.}[r]   & d_2 &}  
  && 
\\
 &\hspace*{15mm} (N_6^{10}) && \hspace*{15mm} (N_6^{11}) \allowdisplaybreaks
\end{align*}
These networks have a convex solution in $\mathbb{R}^2$, as shown in Figure \ref{figRealizable2nPlus1}.

Suppose $\Theta\models \ec{x}{y}$. For convenience, we write $c_3=x$ and $d_3=y$. If $\Theta{\downarrow}\{c_i,c_j,d_i,d_j\}$ $(i\not=j$) is not isomorphic to $N_4^2$ (note that it cannot be isomorphic to $N_4^1$ and $N_4^3$), then $\Theta$ has a convex weak solution in $\mathbb{R}^2$ as before. Suppose on the other hand that for each $i$ and $j$, $\Theta{\downarrow}\{c_i,c_j,d_i,d_j\}$ is isomorphic to $N_4^2$. Then $\Theta$ is isomorphic to the following network:
\begin{align*}
&  \xymatrix@C=1.5em{
  & c_1 \ar@{.}[r]\ar@{-}[dd]\ar@{-}[ddr]  & d_1\ar@{-}[dd]\ar@{-}[ddl] &\\
   c_2\ar@{.}[rrr] \ar@{-}[ur] \ar@{-}[dr] \ar@{-}[urr]\ar@{-}[drr]  & &		& d_2\ar@{-}[ull]\ar@{-}[ul]\ar@{-}[dll]\ar@{-}[dl] \\
  & c_3\ar@{.}[r]   & d_3 &} \\
 &\hspace*{15mm} (N_6^{12})
\end{align*}
A convex solution of $N_6^8$ is shown in Figure \ref{figRealizable2nPlus1}.
\end{proof}

\begin{figure}
\subfigure[$N_6^{1}$]{\includegraphics[width=100pt]{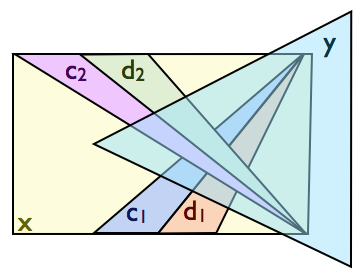}} \hfill
\subfigure[$N_6^{2}$]{\includegraphics[width=100pt]{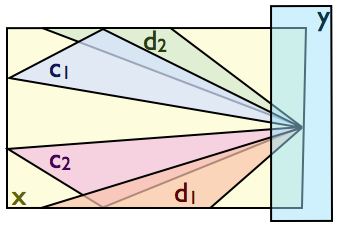}} \hfill
\subfigure[$N_6^{3}$]{\includegraphics[width=100pt]{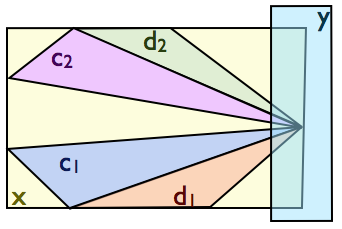}} \\
\subfigure[$N_6^{4}$]{\includegraphics[width=100pt]{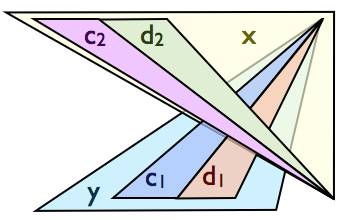}} \hfill
\subfigure[$N_6^{5}$]{\includegraphics[width=100pt]{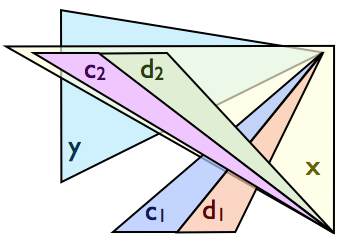}} \hfill
\subfigure[$N_6^{6}$]{\includegraphics[width=100pt]{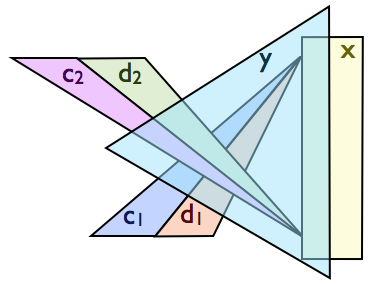}} \\
\subfigure[$N_6^{7}$]{\includegraphics[width=100pt]{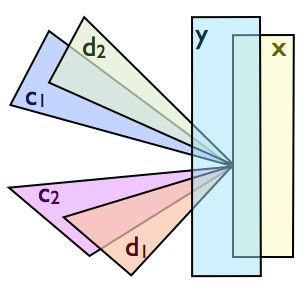}} \hfill
\subfigure[$N_6^{8}$]{\includegraphics[width=100pt]{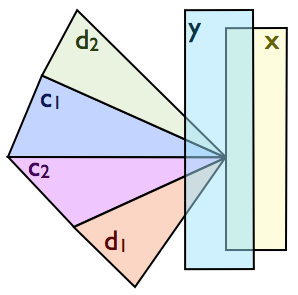}} \hfill
\subfigure[$N_6^{9}$]{\includegraphics[width=100pt]{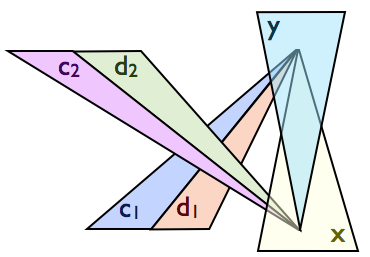}} \\
\subfigure[$N_6^{10}$]{\includegraphics[width=100pt]{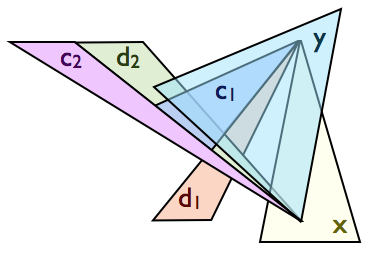}} \hfill
\subfigure[$N_6^{11}$]{\includegraphics[width=100pt]{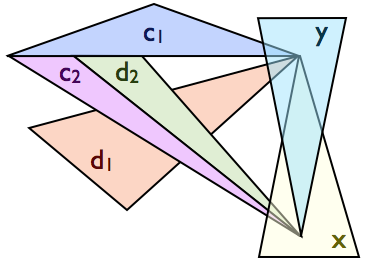}} \hfill
\subfigure[$N_6^{12}$]{\includegraphics[width=100pt]{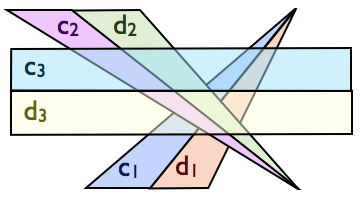}} 
\caption{Convex solutions of the networks $N_6^i$ from Lemma \ref{lem:o+2ec-in2d}. \label{figRealizable2nPlus1}}
\end{figure}

\begin{lemma}\label{lem:7regions}
Suppose $\Theta$ is a basic RCC8 network over $\{\bec,\bpo,\btpp,\btppi\}$ and $V=\{x,y,z,a,b,c,d\}$. In addition, assume that $\Theta$ satisfies 
\begin{itemize}
\item $\Theta{\downarrow}\{a,b,c,d\}$ is $N_4^1,N_4^2$ or $N_4^3$;
\item there is no $u$ in $V$ which satisfies the conditions of Lemma 9;
\item $\Theta\models \{\tpp{y}{x},\ec{z}{x},\ec{z}{y}\}$ and  $\Theta\not\models \tpp{x}{v}$ for $v\in\{a,b,c,d\}$.
\end{itemize}
Then $\Theta$ has a convex weak solution in $\mathbb{R}^2$ in which every \bo-clique has a common part.
\end{lemma}
\begin{proof}
Since there is no region $u \in V$ which satisfies the conditions of Lemma 9, by \new{Remark} 2, there exist $v_1,v_2$ such that $\Theta\models\{\tpp{v_1}{u},\ec{v_2}{u}\}$ for every maximal region $u$ in $\Theta$. In particular, we know $x$ is such a maximal region in $\Theta$.

We note that if one of the following two conditions is satisfied, then using the method from Proposition 20 we can get a convex weak solution of $\Theta$ in $\mathbb{R}^2$   by extending a convex weak solution of $\Theta{\downarrow}(V \setminus \{u_1,u_2,u_3,u_4\})$ in $\mathbb{R}$:
\begin{align} \label{eq:tpp-chain}
(\exists u_1,u_2,u_3,u_4 \in  V) &\quad \Theta\models \{\tpp{u_1}{u_2},\tpp{u_2}{u_3},\ec{u_3}{u_4}\}\\
\label{eq:tpp-pair}
(\exists u_1,u_2,u_3,u_4 \in  V) &\quad \Theta\models \{\tpp{u_1}{u_2},\tpp{u_3}{u_4},\ec{u_2}{u_4}\}.
\end{align}
In the following, we assume $\Theta$ satisfies neither \eqref{eq:tpp-chain} nor \eqref{eq:tpp-pair}. As a consequence, we know that 
\begin{align*}
\Theta \not\models \tpp{v}{y},\  \Theta \not\models \tpp{v}{z},\ \mbox{and}\ \Theta \not\models \{\tpp{y}{v}, \tpp{v}{x}\} \ \mbox{for $v\in\{a,b,c,d\}$.}
\end{align*}

Suppose $v,w$ are two regions in $\{a,b,c,d\}$ and $\Theta\models \ec{v}{w}$. We assert that $v,w$ cannot be both maximal in $\Theta$. We prove this by contradiction. Since $v,w$ are maximal in $\Theta$, we know in particular $\Theta\not\models \tpp{v}{x}$, $\Theta\not\models \tpp{w}{x}$. Moreover, by \new{Remark}~2, we have $\Theta\models \tpp{y}{v}$ or $\Theta\models \tpp{z}{v}$, and $\Theta\models \tpp{y}{w}$ or $\Theta\models\tpp{z}{w}$. Note that if $\Theta\models \{\tpp{z}{v},\tpp{y}{w}\}$ or $\Theta \models \{\tpp{z}{w},\tpp{y}{v}\}$, then \eqref{eq:tpp-pair} is satisfied and this contradicts our assumption. Also note that neither $\Theta\models \{\tpp{z}{v},\tpp{z}{w}\}$ nor $\Theta\models \{\tpp{y}{v},\tpp{y}{w}\}$ is possible as $\Theta\models \ec{v}{w}$. Therefore, $v,w$ cannot be both maximal.

Suppose $\Theta{\downarrow}\{a,b,c,d\}$ is $N_4^2$. Note that $\Theta\models\{\ec{a}{b},\ec{c}{d}\}$. If neither $a$ nor $b$ is maximal, 
then we have $\Theta\models \{ \tpp{a}{x}, \tpp{b}{x}\}$.
Furthermore, by $\Theta\models\ec{z}{x}$, we know $\Theta\models \{\ec{z}{a},\ec{z}{b},\ec{z}{x},\ec{z}{y}\}$. Because $\Theta{\downarrow}\{x,y,a,b\}$ is isomorphic to none of $N_4^1,N_4^2$ and $N_4^3$, it has a convex weak solution in $\mathbb{R}$. Since $z$ is \bec\ to all these four regions, we can extend it to a convex weak solution for  $\Theta{\downarrow}\{x,y,z,a,b\}$ in $\mathbb{R}$. Now, using Proposition 20, we can further extend this to a convex weak solution of $\Theta$ in $\mathbb{R}^2$.

We next consider the subcase when exact one of $a,b$ is maximal. Without loss of generality, we assume that $a$ is maximal and $b$ is not.  It is clear that $\Theta{\downarrow}\{x,y,a,b\}$ is not  isomorphic to any of $N_4^1,N_4^2$ and $N_4^3$ and, hence, has a convex weak solution $\mathcal{S}$ in $\mathbb{R}$. Write $\mathcal{S}(a)=[a^-,a^+]$ and similarly for $b,x,y$. We have either $a^-<x^-<a^+<x^+$ or $x^-<a^-<x^+<a^+$. It is easy to see that there is an interval which is contained in (or partially overlaps, or is disjoint from) $[a^-,a^+]$, which is \bec\ to the other three intervals. Therefore, we can construct a convex weak solution in $\mathbb{R}$ for $\Theta{\downarrow}\{x,y,z,a,b\}$, which can be extended to a convex weak solution in $\mathbb{R}^2$ for $\Theta$ by using Proposition 20.

Suppose $\Theta{\downarrow}\{a,b,c,d\}$ is $N_4^3$ and $\Theta\models \{\ec{a}{b},\ec{a}{c},\ec{b}{c}\}$.  At most one of $a,b,c$ can be maximal in $\Theta$. Without loss of generality, we assume $c$ is not maximal. There are two subcases. Suppose $d$ is not maximal either. Then we have $\Theta\models \{\tpp{c}{x},\tpp{d}{x}\}$ and hence $\Theta\models \{\ec{z}{c},\ec{z}{d},\ec{z}{x},\ec{z}{y}\}$. Just as before, we can construct a convex weak solution in $\mathbb{R}$ for $\Theta{\downarrow}\{x,y,z,c,d\}$ and extend it to a convex weak solution of $\Theta$ in $\mathbb{R}^2$. 

We next consider the subcase when $c$ is not maximal and $d$ is maximal. Because $c$ is not maximal, we have $\Theta\models \tpp{c}{x}$. Furthermore, because $d$ is maximal, we have either $\Theta\models \tpp{z}{d},\ec{y}{d}\}$ or $\Theta\models \{\tpp{y}{d},\ec{z}{d}\}$. In both cases, we can construct a convex weak solution in $\mathbb{R}$ for $\Theta{\downarrow}\{x,y,c,d\}$ and then extend it to a convex weak solution of $\Theta{\downarrow}\{x,y,z,c,d\}$ in $\mathbb{R}$ and to a convex weak solution of $\Theta$ in $\mathbb{R}^2$ as before.

Suppose $\Theta{\downarrow}\{a,b,c,d\}$ is $N_4^1$ and $\Theta\models\{\ec{a}{b},\tpp{a}{c},\tpp{b}{c}\}$. Because $\Theta$ does not satisfy \eqref{eq:tpp-chain}, we know $\Theta\not\models \tpp{c}{x}$ and hence $c$ is maximal in $\Theta$. We have $\Theta\models \po{c}{x}$ and either 
$\Theta\models \ec{y}{c}$ or $\Theta\models \ec{z}{c}$. Since $z$ does not satisfy the conditions of Lemma 9 and $\Theta \not\models \{\tpp{z}{a},\tpp{z}{b}\}$, we know $\Theta\models \tpp{z}{c}$ or $\Theta\models\tpp{z}{d}$.
If $d$ is not maximal, then $\Theta\models \tpp{d}{x}$, and hence 
\begin{align*}
\Theta\models \{ \tpp{d}{x},\ec{z}{d}, \tpp{z}{c}, \ec{y}{c}, y\{\bec,\bpo\} d \}.
\end{align*}
If $d$ is maximal, then we have either  $\Theta\models \{\tpp{z}{d},\ec{y}{d}\}$ or $\Theta\models \{\tpp{y}{d},\ec{z}{d}\}$. In the first case, we have either 
\begin{align*}
\Theta\models\{\tpp{z}{d},\ec{y}{d},\po{x}{d}, \ec{y}{c}, z\{\bec,\bpo,\btpp\} c\}
\end{align*}
 or 
\begin{align*}
\Theta\models\{\tpp{z}{d},\ec{y}{d},\po{x}{d}, \ec{z}{c}, y\{\bec,\bpo,\btpp\} c\}.
\end{align*}
In the second case, we have 
\begin{align*}
\Theta\models\{\tpp{y}{d},\ec{z}{d},\po{x}{d}, \tpp{z}{c},\ec{y}{c}\}. 
\end{align*}
In all these cases, we can construct a convex weak solution in $\mathbb{R}$ for $\Theta{\downarrow}\{x,y,c,d\}$ and then extend it to $\Theta{\downarrow}\{x,y,z,c,d\}$ in $\mathbb{R}$ and to $\Theta$ in $\mathbb{R}^2$ as before.
\end{proof}

\bibliography{qsr,commonsense,momentcurve}

\begin{thebibliography}{10}
\expandafter\ifx\csname url\endcsname\relax
  \def\url#1{\texttt{#1}}\fi
\expandafter\ifx\csname urlprefix\endcsname\relax\def\urlprefix{URL }\fi
\expandafter\ifx\csname href\endcsname\relax
  \def\href#1#2{#2} \def\path#1{#1}\fi

\bibitem{Randell:1992}
D.~Randell, Z.~Cui, A.~Cohn, A spatial logic based on regions and connection,
  in: Proceedings of the 3rd International Conference on Knowledge
  Representation and Reasoning, 1992, pp. 165--176.

\bibitem{sridhar2011video}
M.~Sridhar, A.~G. Cohn, D.~C. Hogg, From video to {RCC8}: exploiting a distance
  based semantics to stabilise the interpretation of mereotopological
  relations, in: Conference on Spatial Information Theory, Springer, 2011, pp.
  110--125.

\bibitem{grutter2008improving}
R.~Gr{\"u}tter, T.~Scharrenbach, B.~Bauer-Messmer, Improving an {RCC}-derived
  geospatial approximation by {OWL} axioms, in: Proceedings of the
  International Semantic Web Conference, 2008, pp. 293--306.

\bibitem{conf/semweb/StockerS08}
M.~Stocker, E.~Sirin, Pelletspatial: A hybrid {RCC-8} and {RDF/OWL} reasoning
  and query engine., in: R.~Hoekstra, P.~F. Patel-Schneider (Eds.), OWLED, Vol.
  529 of CEUR Workshop Proceedings, 2008.

\bibitem{battle2012enabling}
R.~Battle, D.~Kolas, Enabling the geospatial semantic web with {Parliament} and
  {GeoSPARQL}, Semantic Web 3~(4) (2012) 355--370.

\bibitem{koubarakis2011challenges}
M.~Koubarakis, K.~Kyzirakos, M.~Karpathiotakis, C.~Nikolaou, M.~Sioutis,
  S.~Vassos, D.~Michail, T.~Herekakis, C.~Kontoes, I.~Papoutsis, Challenges for
  qualitative spatial reasoning in linked geospatial data, Worskshop on
  Benchmark and Applications of Spatial Reasoning (2011) 33--38.

\bibitem{smart2007framework}
P.~D. Smart, A.~I. Abdelmoty, B.~A. El-Geresy, C.~B. Jones, A framework for
  combining rules and geo-ontologies, in: Web reasoning and rule systems,
  Springer, 2007, pp. 133--147.

\bibitem{Renz:2002}
J.~Renz, A canonical model of the {Region Connection Calculus}, Journal of
  Applied Non-Classical Logics 12~(3--4) (2002) 469--494.

\bibitem{gardenfors2001reasoning}
P.~{G\"ardenfors}, M.~Williams, Reasoning about categories in conceptual
  spaces, in: Proceedings of the International Joint Conference on Artificial
  Intelligence, 2001, pp. 385--392.

\bibitem{rosch1973natural}
E.~H. Rosch, Natural categories, Cognitive Psychology 4~(3) (1973) 328--350.

\bibitem{nalbantov2006nearest}
G.~I. Nalbantov, P.~J. Groenen, J.~C. Bioch, Nearest convex hull
  classification, Tech. rep., Erasmus School of Economics (ESE) (2006).

\bibitem{derrac2014enriching}
J.~Derrac, S.~Schockaert, Enriching taxonomies of place types using flickr, in:
  Proceedings of the 8th International Symposium on Foundations of Information
  and Knowledge Systems, 2014, pp. 174--192.

\bibitem{Erk:2009:RWR:1596374.1596387}
K.~Erk, Representing words as regions in vector space, in: Proceedings of the
  Thirteenth Conference on Computational Natural Language Learning, 2009, pp.
  57--65.

\bibitem{Freksa:1991}
C.~Freksa, Conceptual neighborhood and its role in temporal and spatial
  reasoning, in: M.~Singh, L.~Trav\'e-Massuy\`es (Eds.), Decision Support
  Systems and Qualitative Reasoning, North-Holland, Amsterdam, 1991, pp.
  181--187.

\bibitem{Schockaert20111815}
S.~Schockaert, H.~Prade, Solving conflicts in information merging by a flexible
  interpretation of atomic propositions, Artificial Intelligence 175~(11)
  (2011) 1815 -- 1855.

\bibitem{Cohn:1996}
A.~Cohn, N.~Gotts, The `egg-yolk' representation of regions with indeterminate
  boundaries, in: Geographic Objects with Indeterminate Boundaries
  (P.A.~Burrough and A.U.~Frank, eds.), Taylor and Francis Ltd., 1996, pp.
  171--187.

\bibitem{Schockaert2009258}
S.~Schockaert, M.~D. Cock, E.~E. Kerre, Spatial reasoning in a fuzzy region
  connection calculus, Artificial Intelligence 173~(2) (2009) 258 -- 298.

\bibitem{schockaert2013combining}
S.~Schockaert, S.~Li, Combining {RCC5} relations with betweenness information,
  in: Proceedings of the Twenty-Third International Joint Conference on
  Artificial Intelligence, 2013, pp. 1083--1089.

\bibitem{elith2009species}
J.~Elith, J.~R. Leathwick, Species distribution models: ecological explanation
  and prediction across space and time, Annual Review of Ecology, Evolution,
  and Systematics 40 (2009) 677--697.

\bibitem{walley1991statistical}
P.~Walley, Statistical reasoning with imprecise probabilities, Chapman and Hall
  London, 1991.

\bibitem{Allen83}
J.~Allen, Maintaining knowledge about temporal intervals, Communications of the
  ACM 26~(11) (1983) 832--843.

\bibitem{Vilain:1989:CPA:93913.93996}
M.~Vilain, H.~Kautz, P.~van Beek, Constraint propagation algorithms for
  temporal reasoning: A revised report, in: D.~S. Weld, J.~d. Kleer (Eds.),
  Readings in Qualitative Reasoning About Physical Systems, Morgan Kaufmann
  Publishers Inc., 1990, pp. 373--381.

\bibitem{Davis:1999}
E.~Davis, N.~Gotts, A.~Cohn, Constraint networks of topological relations and
  convexity, Constraints 4 (1999) 241--280.

\bibitem{Basu:1996:CAC:235809.235813}
S.~Basu, R.~Pollack, M.-F. Roy, On the combinatorial and algebraic complexity
  of quantifier elimination, Journal of the ACM 43~(6) (1996) 1002--1045.

\bibitem{turney2010frequency}
P.~D. Turney, P.~Pantel, et~al., From frequency to meaning: Vector space models
  of semantics, Journal of artificial intelligence research 37~(1) (2010)
  141--188.

\bibitem{westphal2009qualitative}
M.~Westphal, S.~W{\"o}lfl, Qualitative {CSP}, finite {CSP}, and {SAT}:
  comparing methods for qualitative constraint-based reasoning, in: Proceedings
  of the 21st international jont conference on Artifical intelligence, 2009,
  pp. 628--633.

\bibitem{ecai2012paper}
S.~Schockaert, S.~Li, Convex solutions of {RCC8} networks, in: Proceedings of
  the 20th European Conference on Artificial Intelligence, 2012, pp. 726--731.

\bibitem{Cui:1993}
Z.~Cui, A.~Cohn, D.~Randell, Qualitative and topological relationships in
  spatial databases, in: Advances in Spatial Databases, Vol. 692 of Lecture
  Notes in Computer Science, 1993, pp. 296--315.

\bibitem{Stell2000111}
J.~Stell, Boolean connection algebras: A new approach to the {Region-Connection
  Calculus}, Artificial Intelligence 122~(1--2) (2000) 111 -- 136.

\bibitem{Li:2006}
S.~Li, On topological consistency and realization, Constraints 11~(1) (2006)
  31--51.

\bibitem{Renz:1999a}
J.~Renz, B.~Nebel, On the complexity of qualitative spatial reasoning: A
  maximal tractable fragment of the {Region Connection Calculus}, Artificial
  Intelligence 108~(1--2) (1999) 69--123.

\bibitem{dewdney1977convex}
A.~Dewdney, J.~Vranch, A convex partition of {$R^3$} with applications to
  {Crum}'s problem and {Knuth}'s post-office problem, Utilitas Math 12 (1977)
  193--199.

\bibitem{s-eubnf-91}
R.~Seidel, Exact upper bounds for the number of faces in $d$-dimensional
  {V}oronoi diagrams, in: P.~Gritzman, B.~Sturmfels (Eds.), Applied Geometry
  and Discrete Mathematics: The Victor Klee Festschrift, Vol.~4 of DIMACS
  Series in Discrete Mathematics and Theoretical Computer Science, AMS Press,
  1991, pp. 517--530.

\bibitem{erickson2003arbitrarily}
J.~Ericksonff, S.~Kim, Arbitrarily large neighborly families of congruent
  symmetric convex 3-polytopes, Discrete Geometry: In Honor of W. Kuperberg's
  60th Birthday (2003) 267--278.

\bibitem{edelsbrunner1987algorithms}
H.~Edelsbrunner, Algorithms in combinatorial geometry, Vol.~10, Springer, 1987.

\bibitem{matouvsek2003using}
J.~Matou{\v{s}}ek, Using the Borsuk-Ulam Theorem: Lectures on Topological
  Methods in Combinatorics and Geometry, Springer, 2003.

\bibitem{anderson1993}
B.~Anderson, Polynomial root dragging, The American Mathematical Monthly 100
  (1993) 864--866.

\bibitem{Frayer:2010-08-01T00:00:00:0002-9890:641}
C.~Frayer, J.~A. Swenson, Polynomial root motion, American Mathematical Monthly
  117~(7) (2010) 641--646.

\bibitem{dlrcc8}
O.~L. Oz{\c{c}}ep, R.~M{\"o}ller, Spatial semantics for concepts, in:
  Proceedings of the 26th International Workshop on Description Logics, 2013,
  pp. 816--828.

\bibitem{Hearst:1992:AAH:992133.992154}
M.~A. Hearst, Automatic acquisition of hyponyms from large text corpora, in:
  Proceedings of the 14th Conference on Computational Linguistics, 1992, pp.
  539--545.

\bibitem{AIJSchockaertPrade2013}
S.~Schockaert, H.~Prade, Interpolative and extrapolative reasoning in
  propositional theories using qualitative knowledge about conceptual spaces,
  Artificial Intelligence 202 (2013) 86 -- 131.

\bibitem{Sheremet01062007}
M.~Sheremet, D.~Tishkovsky, F.~Wolter, M.~Zakharyaschev, A logic for concepts
  and similarity 17~(3) (2007) 415--452.

\bibitem{Schaefer:2010aa}
M.~Schaefer, Complexity of some geometric and topological problems, Vol. 5849
  of Lecture Notes in Computer Science, 2010, pp. 334--344.

\end{thebibliography}
\bibliographystyle{elsarticle-num}
\end{document}